\documentclass{svproc}
\usepackage{bm}
\usepackage{array}
\usepackage{amsmath}
\usepackage{amsfonts,latexsym}
\usepackage{graphicx,subfig,wrapfig}
\usepackage{times}
\usepackage{psfrag,epsfig}
\usepackage{verbatim}
\usepackage{tabularx}
\usepackage[implicit=false]{hyperref}
\usepackage{color}
\usepackage{soul}
\usepackage{indentfirst}
\usepackage{inputenc}
\usepackage[capitalise]{cleveref}
\usepackage{courier}
\usepackage{authblk}
\usepackage{mathtools}
\usepackage{mathtools}
\usepackage{cite}
\usepackage{amssymb}
\usepackage{graphicx}
\usepackage{caption}
\usepackage{dblfloatfix}
\usepackage{xfrac}
\usepackage{algorithm,algpseudocode}
\usepackage{algorithmicx}
\usepackage{lipsum}
\usepackage{xifthen}
\usepackage{needspace}

\usepackage{enumitem}
\usepackage{filecontents}
\usepackage{bibunits}
\usepackage{todonotes}
\usepackage{chngcntr}
\usepackage{seqsplit}

\usepackage{amsthm}

\def\diag{\text{diag}}
\def\0{\boldsymbol{0}}

\def\I{\boldsymbol{I}}

\def\I{\mathbf{I}}

\def\anc{\text{anc}}
\def\pa{\text{par}}
\def\M{\overline{\mathcal{M}}}
\def\gv{{\vec{g}}}

\def\R{\mathbb{R}}

\def\JJ{\mathcal{J}}

\def\se3{\mathfrak{se}(3)}

\def\Ad{\mathrm{Ad}}

\def\ad{\mathrm{ad}}

\def\tD{\mathbb{D}}

\def\lq{\overline{q}}

\def\leta{\overline{\eta}}

\def\k1{{k+1}}

\long\def\answer#1{}

\long\def\comment#1{}

\newcommand{\HRule}[1][\medskipamount]{\par
	\vspace*{\dimexpr-\parskip-\baselineskip+#1}
	\noindent\rule{\linewidth}{0.2mm}\par
	\vspace*{\dimexpr-\parskip-.5\baselineskip+#1}}

\def\qika{{q_i^{k,\alpha}}}

\def\qikb{{q_i^{k,\beta}}}

\def\qjkn{{q_j^{k,\nu}}}

\def\qdotika{{\dot{q}_i^{k,\alpha}}}

\newcommand{\qk}[1]{{#1^{k}}}
\newcommand{\qka}[1]{{#1^{k,\alpha}}}
\newcommand{\qkb}[1]{{#1^{k,\beta}}}
\newcommand{\qkg}[1]{{#1^{k,\gamma}}}

\newcommand{\qkn}[1]{{#1^{k,\nu}}}
\newcommand{\qkp}[1]{{#1^{k,\rho}}}
\newcommand{\qkvp}[1]{{#1^{k,\varrho}}}
\newcommand{\qkz}[1]{{#1^{k,0}}}
\newcommand{\qks}[1]{{#1^{k,s}}}

\newcommand{\qkag}[1]{{#1^{k,\alpha\gamma}}}
\newcommand{\qkan}[1]{{#1^{k,\alpha\nu}}}
\newcommand{\qkap}[1]{{#1^{k,\alpha\rho}}}

\newcommand{\qkbg}[1]{{#1^{k,\beta\gamma}}}
\newcommand{\qkbn}[1]{{#1^{k,\beta\nu}}}
\newcommand{\qkbp}[1]{{#1^{k,\beta\rho}}}

\newcommand{\qkgn}[1]{{#1^{k,\gamma\nu}}}

\newcommand{\qkgp}[1]{{#1^{k,\gamma\rho}}}
\newcommand{\qkgvp}[1]{{#1^{k,\gamma\varrho}}}

\newcommand{\qkvpn}[1]{{#1^{k,\varrho\nu}}}
\newcommand{\qkvpp}[1]{{#1^{k,\varrho\rho}}}
\newcommand{\qa}[1]{{#1^{\alpha}}}

\newcommand{\qab}[1]{{#1^{\alpha\beta}}}
\newcommand{\qag}[1]{{#1^{\alpha\gamma}}}
\newcommand{\qan}[1]{{#1^{\alpha\nu}}}
\newcommand{\qap}[1]{{#1^{\alpha\rho}}}

\newcommand{\qpg}[1]{{#1^{\rho\gamma}}}

\def\pa{{\mathrm{par}}}
\def\des{{\mathrm{des}}}
\def\chd{{\mathrm{chd}}}

\def\la{\overline{a}}
\def\lb{\overline{b}}
\def\lv{\overline{v}}
\def\lS{\overline{S}}

\def\lLam{\overline{\Lambda}}
\def\le{\overline{\eta}}
\def\ld{{\overline{\delta}}}
\def\ldh{{\overline{\delta}}\hspace{0.1em}}
\def\lmu{{\overline{\mu}}}
\def\lG{{\overline{\varGamma}}}
\def\lM{{\overline{M}}}
\def\lF{{\overline{F}}}

\def\lD{{\overline{D}}}
\def\lO{{\overline{\Omega}}}

\def\lq{{\overline{q}}}

\def\cL{\mathcal{L}}
\def\cF{\mathcal{F}}

\def\lsig{{\overline{\sigma}}}

\newcommand{\innprod}[2]{\langle {#1},{#2}\rangle}

\newcommand{\algrule}[1][.2pt]{\par\vskip.5\baselineskip\hrule height #1\par\vskip.5\baselineskip}

\newcommand\blfootnote[1]{%
	\begingroup
	\renewcommand\thefootnote{}\footnote{#1}%
	\addtocounter{footnote}{-1}%
	\endgroup
}

\newcounter{myalg}
\newenvironment{myalg}[1][]%
{
	\needspace{2\baselineskip}
	\noindent \rule{\linewidth}{0.5pt} \endgraf
	\refstepcounter{myalg}
	\noindent\textbf{Algorithm~\themyalg}%
	\ifthenelse{\isempty{#1}}{}{\hspace{0.2em}\ #1}
	\HRule
}{
	\noindent \rule{\linewidth}{0.5pt}
}%

{
	\needspace{2\baselineskip}
	\noindent \rule{\linewidth}{0.5pt} \endgraf
	\refstepcounter{myalg}
	\noindent\textbf{Algorithm~\themyalg}%
	\ifthenelse{\isempty{#1}}{}{\hspace{0.2em}\ #1}
	\HRule
}{
	\noindent \rule{\linewidth}{0.5pt}
}%

\newtheorem{prop}{Proposition}
\crefname{prop}{Proposition}{Propositions}
\newtheorem{assumption}{Assumption}
\crefname{assumption}{Assumption}{Assumptions}
\crefname{myalg}{Algorithm}{Algorithms}

\makeatletter
\def\user@resume{resume}
\def\user@intermezzo{intermezzo}
\newcounter{previousequation}
\newcounter{lastsubequation}
\newcounter{savedparentequation}
\renewenvironment{subequations}[1][]{%
	\def\user@decides{#1}%
	\setcounter{previousequation}{\value{equation}}%
	\ifx\user@decides\user@resume 
	\setcounter{equation}{\value{savedparentequation}}%
	\else  
	\ifx\user@decides\user@intermezzo
	\refstepcounter{equation}%
	\else
	\setcounter{lastsubequation}{0}%
	\refstepcounter{equation}%
	\fi\fi
	\protected@edef\theHparentequation{%
		\@ifundefined {theHequation}\theequation \theHequation}%
	\protected@edef\theparentequation{\theequation}%
	\setcounter{parentequation}{\value{equation}}%
	\ifx\user@decides\user@resume 
	\setcounter{equation}{\value{lastsubequation}}%
	\else
	\setcounter{equation}{0}%
	\fi
	\def\theequation  {\theparentequation  \alph{equation}}%
	\def\theHequation {\theHparentequation \alph{equation}}%
	\ignorespaces
}{%
	\ifx\user@decides\user@resume
	\setcounter{lastsubequation}{\value{equation}}%
	\setcounter{equation}{\value{previousequation}}%
	\else
	\ifx\user@decides\user@intermezzo
	\setcounter{equation}{\value{parentequation}}%
	\else
	\setcounter{lastsubequation}{\value{equation}}%
	\setcounter{savedparentequation}{\value{parentequation}}%
	\setcounter{equation}{\value{parentequation}}%
	\fi\fi
	\ignorespacesafterend
}
\makeatother

\usepackage{url}

\defaultbibliography{mybib}
\defaultbibliographystyle{unsrt}

\begin{document}
\mainmatter              
\title{Efficient Computation of Higher-Order Variational Integrators in Robotic Simulation and Trajectory Optimization}

\author{Taosha Fan \quad  Jarvis Schultz \quad Todd Murphey}

\institute{Department of Mechanical Engineering, Northwestern University,\\ 2145 Sheridan Road, Evanston, IL 60208, USA\\
\email{taosha.fan@u.northwestern.edu, jschultz@northwestern.edu, t-murphey@northwestern.edu}}

\maketitle 
\vspace{-0.7em}
\begin{abstract}
This paper addresses the problem of efficiently computing higher-order variational integrators in simulation and trajectory optimization of mechanical systems as those often found in robotic applications. We develop $O(n)$ algorithms to evaluate the discrete Euler-Lagrange (DEL) equations and compute the Newton direction for solving the DEL equations, which results in linear-time variational integrators of arbitrarily high order. To our knowledge, no linear-time higher-order variational or even implicit integrators have been developed before. Moreover, an $O(n^2)$ algorithm to linearize the DEL equations is presented, which is useful for trajectory optimization. These proposed algorithms eliminate the bottleneck of implementing higher-order variational integrators in simulation and trajectory optimization of complex robotic systems. The efficacy of this paper is validated through comparison with existing methods, and implementation on various robotic systems---including trajectory optimization of the Spring Flamingo robot, the LittleDog robot and the Atlas robot. The results illustrate that the same integrator can be used for simulation and trajectory optimization in robotics, preserving mechanical properties while achieving good scalability and accuracy.
\end{abstract}
\section{Introduction}
\blfootnote{This material is based upon work supported by the National Science Foundation under award DCSD-1662233. Any opinions, findings, and conclusions or recommendations expressed in this material are those of the authors and do not necessarily reflect the views of the National Science Foundation. 
}
Variational integrators conserve symplectic form, constraints and energetic quantities \cite{marsden2001discrete,johnson2009scalable,kobilarov2009lie,johnson2015structured,fan2015structured,junge2005discrete}. As a result, variational integrators generally outperform the other types of integrators with respect to numerical accuracy and stability, thus permitting large time steps in simulation and trajectory optimization, which is useful for complex robotic systems \cite{marsden2001discrete,johnson2009scalable,kobilarov2009lie,johnson2015structured,fan2015structured,junge2005discrete}. Moreover, variational integrators can also be regularized for collisions and friction by leveraging the linear complementarity problem (LCP) formulation \cite{lacoursiere2007ghosts,manchestercontact}.

The computation of variational integrators is comprised of the discrete Euler-Lagra-nge equation (DEL) evaluation, the descent direction computation for solving the DEL equations and the DEL equation linearization. The computation of these three phases of variational integrators can be accomplished with automatic differentiation and our prior methods \cite{johnson2009scalable,johnson2015structured}, both of which are $O(n^2)$ to evaluate the DEL equations and $O(n^3)$ to compute the Newton direction and linearize the DEL equations for an $n$-degree-of-freedom mechanical system. Recently, a linear-time second-order variational integrator was developed in \cite{lee2016linear}, which uses the quasi-Newton method and works for small time steps and comparatively simple mechanical systems.  

Higher-order variational integrators are needed for greater accuracy in predicting the dynamic motion of robots \cite{posa2016optimization,hereid2018dynamic}. However, the computation of higher-order variational integrators has rarely been addressed. The quasi-Newton method in \cite{lee2016linear} only applies to second-order variational integrators, and while automatic differentiation and our prior methods \cite{johnson2009scalable,johnson2015structured} are implementable for higher-order variational integrators, the complexity increases superlinearly as the integrator order increases. 

In this paper, we address the computation efficiency of higher-order variational integrators and develop: i) an $O(n)$ method for the evaluation of the DEL equations, ii) an $O(n)$ method for the computation of the Newton direction, and iii) an $O(n^2)$ method for the linearization of the DEL equations. The proposed characteristics i) -- iii) eliminate the bottleneck of implementing higher-order variational integrators in simulation and trajectory optimization of complex robotic systems, and to the best of our knowledge, no similar work has been presented before. In particular, we believe that the resulting variational integrator from i) and ii) is the first exactly linear-time implicit integrator of third or higher order for mechanical systems.

The rest of this paper is organized as follows. \cref{section::preliminary} reviews higher-order variational integrators, the Lie group formulation of rigid body motion and the tree representation of mechanical systems. \cref{section::lin_vi,section::quad_lin} respectively detail the linear-time higher-order variational integrator and the quadratic-time linearization, which are the main contributions of this paper. \cref{section::discussion} compares our work with existing methods, and \cref{section::examples} presents examples of trajectory optimization for the Spring Flamingo robot, the LittleDog robot and the Atlas robot. The conclusions are made in \cref{section::conclusion}.

\vspace{-0.5em}
\section{Preliminaries and Notation}\label{section::preliminary}
In this section, we review higher-order variational integrators, the Lie group formulation of rigid body motion, and the tree representation of mechanical systems. In addition, notation used throughout this paper is introduced accordingly. 
\vspace{-0.5em}
\subsection{Higher-Order Variational Integrators}
In this paper, higher-order variational integrators are derived with the methods in \cite{marsden2001discrete,ober2015construction,ober2017galerkin}.

A trajectory $(q(t),\,\dot{q}(t))$ where $0\leq t\leq T$ of a forced mechanical system should satisfy the \textit{Lagrange-d'Alembert principle}:
\begin{equation}\label{eq::lda}
\delta\mathfrak{S}=\delta\int_0^T \cL(q,\dot{q})dt+\int_0^T \!\!\cF(t)\cdot\delta q dt = 0
\end{equation}
in which $\cL(q,\dot{q})$ is the system's Lagrangian and $\cF(t)$ is the generalized force. Provided that the time interval $[0,\,T]$ is evenly divided into $N$ sub-intervals with $\Delta t = T/N$, and each $q(t)$ over $[k\Delta t,\,(k+1)\Delta t]$ is interpolated with $s+1$ control points $\qka{q}=q(\qka{t})$ in which $\alpha=0,\,1,\,\cdots,\, s$ and $k\Delta t = \qkz{t}<t^{k,1}<\cdots<\qks{t}=(k+1)\Delta t$, then there are coefficients $\qab{b}$ ($0\leq\alpha,\,\beta\leq s$) such that
\begin{equation}\label{eq::qdot}
\vspace{-0.25em}
\dot{q}(\qka{t})\approx\qka{\dot{q}} =\frac{1}{\Delta t}\sum_{\beta=0}^s \qab{b}\qkb{q}.
\vspace{-0.25em}
\end{equation}
{In this paper, we assume that the quadrature points of the quadrature rule are also $\qka{t}$ though our algorithms in \cref{section::lin_vi,section::quad_lin} can be generalized for any quadrature rules.} Then the Lagrange-d'Alembert principle \cref{eq::lda} is approximated as
\begin{equation}\label{eq::dlda}
\delta\mathfrak{S} \approx \sum_{k=0}^{N-1}\sum_{\alpha=0}^s \qa{w}\left[\delta \cL(\qka{q},\qka{\dot{q}})+\cF(\qka{t})\cdot\delta \qka{q}\right]\cdot\Delta t=0,
\end{equation} 
in which $\qa{w}$ are weights of the quadrature rule used for integration. In variational integrators, the \textit{discrete Lagrangian} and the \textit{discrete generalized force} are defined to be
\begin{equation}\label{eq::ddl}
\cL_d(\qkz{q},\,q^{k,1},\,\cdots,\,\qks{q})=\sum\limits_{\alpha=0}^s \qa{w}\cL(\qka{q},\qka{\dot{q}})\Delta t
\end{equation}
and $\qka{\cF_d}(\qka{t})= \qa{w} \cF(\qka{t})\Delta t$, respectively. Note that by definition we have $\qks{t}=t^{k+1,0}$ and $\qks{q}=q^{k+1,0}$, and as a result of \cref{eq::dlda}, we obtain 
\begin{subequations}\label{eq::DEL_general}
	\begin{equation}\label{eq::DEL_general1}
	p^{k} + \tD_1 \cL_d(\lq^k) +\cF_d^{k,0}=0,
	\end{equation}
	\begin{equation}\label{eq::DEL_general2}
	\tD_{\alpha+1} \cL_d(\lq^k) +\cF_d^{k,\alpha}=0\quad \forall \alpha=1,\,\cdots,\,s-1,
	\end{equation}
	\begin{equation}\label{eq::DEL_general3}
	p^{k+1}=\tD_{s+1}\cL_d(\lq^k) +\cF_d^{k,s}
	\end{equation}
\end{subequations}
in which $p^k$ is the \textit{discrete momentum}, $\overline{q}^k$ stands for the tuple $(\qkz{q},\,q^{k,1},\,\cdots,\,\qka{q})$, and $\tD_{\alpha+1}\cL_d$ is the derivative with respect to $\qka{q}$. Note that \cref{eq::DEL_general} is known as the \textit{discrete Euler-Lagrangian (DEL) equations}, which implicitly define an update rule $(\qkz{q},\,p^k)\rightarrow (q^{k+1,0},\,p^{k+1})$ by solving $sn$ nonlinear equations from \cref{eq::DEL_general1,eq::DEL_general2}. In a similar way, for mechanical systems with constraints $h(q,\dot{q})=0$, we have
\begin{subequations}\label{eq::DEL_general_c}
	\begin{equation}\label{eq::DEL_general_c1}
	p^{k} + \tD_1 \cL_d(\lq^k) +\cF_d^{k,0} + \qkz{A}(\qkz{q})\cdot\lambda^{k,0}=0,
	\end{equation}
	\begin{equation}
	\tD_{\alpha+1} \cL_d(\lq^k) +\cF_d^{k,\alpha} + \qka{A}(\qka{q})\cdot\qka{\lambda}=0\quad\forall \alpha=1,\,\cdots,\,s-1,
	\end{equation}
	\begin{equation}
	p^{k+1}=\tD_{s+1}\cL_d(\lq^k) +\cF_d^{k,s},
	\end{equation}
	\begin{equation}
	\qka{h}(q^{k+1,\alpha},\dot{q}^{k+1,\alpha})=0\quad \forall \alpha=1,\,\cdots,\,s
	\end{equation}
\end{subequations}
in which $\qka{A}(\qka{q})$ is the discrete constraint force matrix and $\qka{\lambda}$ is the discrete constraint force.\par

The resulting higher-order variational integrator is referred as the Galerkin integrator \cite{marsden2001discrete,ober2015construction,ober2017galerkin}, the accuracy of which depends on the number of control points as well as the numerical quadrature of the discrete Lagrangian. If there are $s+1$ control points and the Lobatto quadrature is employed, then the resulting variational integrator has an accuracy of order $2s$ \cite{ober2015construction,ober2017galerkin}. The Galerkin integrator includes the \textit{trapezoidal variational integrator} and the \textit{Simpson variational integrator} as shown in \cref{example::svi,example::tvi}, the DEL equations of which are given by \cref{eq::DEL_general,eq::DEL_general_c}. \par

\begin{example}\label{example::tvi}
	The trapezoidal variational integrator is a second-order integrator with two control points $\qk{\lq}=(\qkz{q},\,q^{k,1})$ such that $\qkz{q}=q\big(k\Delta t\big)$ and $q^{k,1}=q\big((k+1)\Delta t\big)$, $\qkz{\dot{q}}=\dot{q}^{k,1}=\frac{q^{k,1}-\qkz{q}}{\Delta t}$, and $\cL_d (\qk{\lq})=\frac{\Delta t}{2}\left[\cL(\qkz{q},\qkz{\dot{q}})+\cL(q^{k,1},\dot{q}^{k,1})\right]$.
\end{example}
\vspace{-1em}
\begin{example}\label{example::svi}
	The Simpson variational integrator is a fourth-order integrator with three control points $\qk{\lq}=(\qkz{q},\,q^{k,1},\,q^{k,2})$ in which $\qkz{q}=q\big(k\Delta t\big)$, $\qkz{q}=q\big((k+\frac{1}{2})\Delta t\big)$ and $q^{k,2}=q\big((k+1)\Delta t\big)$, $\qkz{\dot{q}}=\frac{4q^{k,1}-3\qkz{q}-q^{k,2}}{\Delta t}$, $\dot{q}^{k,1}=\frac{q^{k,2}-\qkz{q}}{\Delta t}$ and $\dot{q}^{k,2}=\frac{\qkz{q}+3q^{k,2}-4q^{k,1}}{\Delta t}$, and $\cL_d (\qk{\lq})=\frac{\Delta t}{6}\big[\cL(\qkz{q},\qkz{\dot{q}})+4\cL(q^{k,1},\dot{q}^{k,1})+\cL(q^{k,2},\dot{q}^{k,2})\big]$.
\end{example}
\vspace{-1.5em}
\subsection{The Lie Group Formulation of Rigid Body Motion}\label{subsection::lie}
The configuration of a rigid body $g=(R,p)\in SE(3)$ can be represented as a $4\times 4$ matrix $g=\begin{bmatrix}
R & p\\0 & 1
\end{bmatrix}$ in which $R\in SO(3)$ is a rotation matrix and $p\in\R^3$ is a position vector. The body velocity of the rigid body $v=(\omega, v_{O})\in T_e SE(3)$ is an element of the Lie algebra and can be represented either as a $6\times 1$ vector $v=(g^{-1}\dot{g})^\vee=\begin{bmatrix}
\omega^T & v_O^T
\end{bmatrix}$ or a $4\times 4$ matrix $\hat{v}=g^{-1}\dot{g}=\begin{bmatrix}
\hat{\omega} & v_O\\
0 & 0
\end{bmatrix}$ in which $\omega=(\omega_x,\omega_y,\omega_z)\in T_e SO(3)$ is the angular velocity, $v_O$ is the linear velocity, $\hat{\omega}=\begin{bmatrix}
0 & -\omega_z & \omega_y\\ \omega_z & 0 & -\omega_x\\ -\omega_y & \omega_x & 0
\end{bmatrix}\in \R^{3\times 3}$, and the hat ``$\wedge$'' and unhat``$\vee$'' are linear operators that relate the vector and matrix representations. The same representation and operators also apply to the spatial velocity $\lv=(\overline{\omega},\lv_O)\in T_e SE(3)$, whose $6\times1$ vector and $4\times 4$ matrix representations are respectively $\lv = (\dot{g}g^{-1})^\vee$ and $\hat{\lv}=\dot{g}g^{-1}$.\par
In the rest of this paper, if not specified, vector representation is used for $T_e SE(3)$, such as $v$, $\lv$, etc., and the adjoint operators $\Ad_g$ and $\ad_v: T_eSE(3)\rightarrow T_e SE(3)$ can be accordingly represented as $6\times 6$ matrices $\Ad_g =\begin{bmatrix}
R & 0\\
\hat{p}R & R
\end{bmatrix}$ and $\ad_v = \begin{bmatrix}
\hat{\omega} & 0\\
\hat{v}_O & \hat{\omega}
\end{bmatrix}$ such that $\lv = \Ad_g v$ and $\ad_{v_1} v_2 = (\hat{v}_1\hat{v}_2-\hat{v}_2\hat{v}_1)^\vee$. For consistence, the dual Lie algebra $T_e^* SE(3)$ uses the $6\times1$ vector representation as well. As a result, the body wrench $F=(\tau,f_O)\in T_e^*SE(3)$ is represented as a $6\times1$ vector $F=\begin{bmatrix}
\tau^T & f_O^T
\end{bmatrix}^T$ in which $\tau\in T_e ^*SO(3)$ is the torque and $f_O$ is the linear force so that $\innprod{F}{v} = F^Tv$. Moreover, we define the linear operator $\ad_F^D:T_eSE(3)\rightarrow T_e^* SE(3)$ which is represented as a $6\times 6$ matrix $\ad_F^D=\begin{bmatrix}
\hat{\tau} & \hat{f}_O\\
\hat{f}_O & 0
\end{bmatrix}$ so that $F^T\ad_{v_1}v_2 = v_2^T\ad_F^D v_1=-v_1^T\ad_F^Dv_2$ for $v_1,v_2\!\in\! T_eSE(3)$. The same representation and operators also apply to the spatial wrench $\lF= \Ad_g^{-T} F=(\overline{\tau},\overline{f}_O)$ which is paired with the spatial velocity $\lv=\Ad_g v$.\par
\vspace{-0.5em}
\subsection{The Tree Representation of Mechanical Systems}\label{subsection::tree}
\vspace{-0.25em}
In general, a mechanical system with $n$ inter-connected rigid bodies indexed as $1,\,2,\,\cdots,\,n$ can be represented through a tree structure so that each rigid body has a single parent and zero or more children \cite{johnson2009scalable,featherstone2014rigid}, and such a representation is termed as \textit{tree representation}. In this paper, the spatial frame is denoted as $\{0\}$, which is the root of the tree representation, and we denote the body frame of rigid body $i$ as $\{i\}$, and the parent, ancestors, children and descendants of rigid body $i$ as $\pa(i)$, $\anc(i)$, $\chd(i)$ and $\des(i)$, respectively. Since all joints can be modeled using a combination of revolute joints and prismatic joints, we assume that each rigid body $i$ is connected to its parent by a one-degree-of-freedom joint $i$ which is either a revolute or a prismatic joint and parameterized by a real scalar $q_i\in\R$. As a result, the tree representation is parameterized with $n$ generalized coordinates $q=\begin{bmatrix}
q_1 & q_2 &\cdots & q_n
\end{bmatrix}^T\in\R^n$. For each joint $i$, the joint twist with respect to frame $\{0\}$ and $\{i\}$ are respectively denoted as $6\times1$ vectors $\lS_i=\begin{bmatrix}
\overline{s}_i^T & \overline{n}_i^T
\end{bmatrix}^T$ and  $S_i=\begin{bmatrix}
{s}_i^T & {n}_i^T
\end{bmatrix}^T$ in which $\overline{s}_i,\,s_i$ are $3\times1$ vectors corresponding to rotation and $\overline{n}_i,\,n_i$ are $3\times 1$ vectors corresponding to translation. Note that $S_i$, $s_i$ and $n_i$ are constant by definition. Moreover, $\lS_i$ and $S_i$ are related as $\lS_i=\Ad_{g_i}S_i$ where $g_i\in SE(3)$ is the configuration of rigid body $i$, and $\dot{\lS_i}=\ad_{\lv_i}\lS_i$, where $\lv_i\in T_e SE(3)$ is the spatial velocity of rigid body $i$.\par

It is assumed without loss of generality in this paper that the origin of frame $\{i\}$ is the mass center of rigid body $i$, and $j\in\des(i)$ only if $i<j$, or equivalently $j\in\anc(i)$ only if $i>j$.\par

The rigid body dynamics can be computed through the tree representation. The configuration $g_i=(R_i,p_i)\in SE(3)$ of rigid body $i$ is $g_i = g_{\pa(i)}g_{\pa(i),i}(q_i)$ in which $g_{\pa(i),i}(q_i) = g_{\pa(i),i}(0)\exp(\hat{S}_i q_i )$ is the rigid body transformation from frame $\{i\}$ to its parent frame $\{\pa(i)\}$, and the spatial velocity $\overline{v}_i$ of rigid body $i$ is $\overline{v}_i = \overline{v}_{\pa(i)} + \lS_i\cdot\dot{q}_i$. In addition, the spatial inertia matrix $\lM_i$ of rigid body $\{i\}$ with respect to frame $\{0\}$ is  
$\lM_i=\Ad_{g_i}^{-T}M_i\Ad_{g_i}^{-1}$ in which $M_i = \diag\{\mathcal{I}_i, m_i\I\}\in\R^{6\times 6}$ is the constant body inertia matrix of rigid body $i$, $\mathcal{I}_i\in \R^{3\times 3}$ is the body rotational inertia matrix, $m_i\in\R$ is the mass and $\I\in\R^{3\times 3}$ is the identity matrix.

In rigid body dynamics, an important notion is the \textit{articulated body} \cite{featherstone2014rigid}. In terms of the tree representation, articulated rigid body $i$ consists of rigid body $i$ and all its descendants $j\in\des(i)$, and the interactions with articulated body $i$ can only be made through rigid body $i$, which is known as the handle of the articulated body $i$.\par 

In the last thirty years, a number of algorithms for efficiently computing the rigid body dynamics have been developed based on tree representations and articulated bodies \cite{featherstone2014rigid,yamane2002efficient,fijany1995parallel}, making explicit integrators have $O(n)$ complexity for an $n$-degree-of-freed-om mechanical system. Even though the same algorithms might be used for the evaluation of implicit integrators, \textit{none of them can be used for the computation of the Newton direction for solving implicit integrators}. If the residue is $r^k$, the Newton direction of an implicit integrator is computed as $\delta q^k=-\JJ(q^k)^{-1}r^k$; however, the Jacobian matrix $\JJ(q^k)$ is usually asymmetric and indefinite, and has a size greater than $n\times n$ for higher-order implicit integrators, which means that the computation of implicit integrators is distinct from explicit integrators whose computation is simply a combination of the algorithms in \cite{featherstone2014rigid,yamane2002efficient,fijany1995parallel} with an appropriate integration scheme. Furthermore, the computation of implicit integrators is much more complicated than the computation of forward and inverse dynamics and out of the scope of those algorithms in \cite{featherstone2014rigid,yamane2002efficient,fijany1995parallel}.

\section{The Linear-Time Higher-Order Variational Integrator}\label{section::lin_vi}
In this and next section, we present the propositions and algorithms efficiently computing higher-order variational integrators, {whose derivations are omitted due to space limitations. Though not required for implementation, we refer the reader to the supplementary appendix of this paper \cite{fan2018wafr_app} for detailed proofs.\footnote{In addition to the proofs, the supplementary appendix \cite{fan2018wafr_app} also contains the complete $O(n)$ algorithms to compute the Newton direction for higher-order variational integrators.}} 

{
In the rest of this paper, if not specified, we assume that the mechanical system has $n$ degrees of freedom and the higher-order variational integrator has $s+1$ control points $\qka{q}=q(\qka{t})$ in which $0\leq \alpha\leq s$. Note that the notation $(\cdot)^{k,\alpha}$ is used to denote quantities $(\cdot)$ associated with $\qka{q}$ and $\qka{t}$, such as $\qka{q_i}$, $\qka{g_i}$, $\qka{\lv_i}$, etc.
}
\subsection{The DEL Equation Evaluation}\label{subsection::eval}
 To evaluate the DEL equations, the \textit{discrete articulated body momentum} and \textit{discrete articulated body impulse} are defined from the perspective of articulated bodies as follows.
\begin{definition}\label{def::abmom}
	The discrete articulated body momentum $\lmu^{k,\alpha}_i\in \R^6$ for articulated body $i$ is defined to be $\qka{\lmu_i} = \qka{\lM_i}\qka{\lv_i}+\sum_{j\in\chd(i)} \qka{\lmu_j}$ in which $\qka{\lM_i}$ and $\qka{\lv_i}$ are respectively the spatial inertia matrix and spatial velocity of rigid body $i$.
\end{definition}

{
\begin{definition}\label{def::abFd}
	Suppose $\lF_i(t)\in\R^6$ is the sum of all the wrenches directly acting on rigid body $i$, which does not include those applied or transmitted through the joints that are connected to rigid body $i$. The discrete articulated body impulse $\qka{\lG_i}\in\R^6$ for articulated body $i$ is defined to be $\qka{\lG_i}=\lF_{i}^{k,\alpha} + \sum_{j\in\chd(i)}\lG_{j}^{k,\alpha}$ in which $\qka{\lF_i}=\qa{\omega} \lF_i(\qka{t})\Delta t\in\R^6$ is the discrete impulse acting on rigid body $i$. Note that $\lF_i(t)$, $\qka{\lF_i}$ and $\qka{\lG_i}$ are expressed in frame $\{0\}$.
\end{definition}

\begin{remark}
As for wrenches exerted on rigid body $i$, in addition to $\lF_i(t)$ which includes gravity as well as the external wrenches that directly act on rigid body $i$, there are also wrenches applied through joints, e.g., from actuators, and wrenches transmitted through joints, e.g., from the parent and children of rigid body $i$ in the tree representation.
\end{remark}
}
It can be seen in \cref{prop::eval} that $\qka{\lmu_i}$ and $\qka{\lG_i}$ make it possible to evaluate the DEL equations without explicitly calculating $\tD_{\alpha+1} \cL_d(\qk{\lq})$ and $\qka{\cF_d}$ in \cref{eq::DEL_general,eq::DEL_general_c}.
\begin{prop}\label{prop::eval}
If $Q_i(t)\in\R$ is the sum of all joint forces applied to joint $i$ and $p^k=\begin{bmatrix}
p_1^k & p_2^k &\cdots & p_n^k
\end{bmatrix}^T\in \R^n$ is the discrete momentum, the DEL equations \cref{eq::DEL_general} can be evaluated as
\begin{subequations}\label{eq::eval}
	\begin{align}
	&\label{eq::eval1}r_i^{k,0}=p_i^{k} + {\lS_i^{k,0}}^T\cdot\lO_i^{k,0}+\sum_{\beta= 0}^s a^{0\beta}{\lS_i^{k,\beta}}^T\cdot\lmu_i^{k,\beta}+Q_i^{k,0},\\
	&\label{eq::eval2}r_i^{k,\alpha}={\lS_i^{k,\alpha}}^T\cdot\lO_i^{k,\alpha}+\sum_{\beta= 0}^s a^{\alpha\beta}{\lS_i^{k,\beta}}^T \cdot\lmu_i^{k,\beta}+\qka{Q_i}
	\quad\forall\alpha=1,\cdots,s-1,\\
	&p_i^{k+1}={\qks{\lS_i}}^T\cdot\qks{\lO_i}+\sum_{\beta= 0}^s a^{s\beta}{\lS_i^{k,\beta}}^T \cdot\lmu_i^{k,\beta}+\qks{Q_i}
	\end{align}
\end{subequations}
in which $\qka{r_i}$ is the residue of the DEL equations \cref{eq::DEL_general1,eq::DEL_general2}, $a^{\alpha\beta}=w^\beta b^{\beta\alpha}$, $\qka{\lO_i}= w^{\alpha}\Delta t\cdot {\ad}_{\qka{\lv_{i}}}^T\cdot\qka{\lmu_i}+\qka{\lG_{i}}$, and $\qka{Q_i}=\qa{\omega} Q_i(\qka{t})\Delta t$ is the discrete joint force applied to joint $i$. 
\end{prop}
\begin{proof}
See \cite[\cref{subsection::prop1}]{fan2018wafr_app}
\end{proof}
\vspace{-0.5em}

In \cref{eq::eval1,eq::eval2}, if all $\qka{r_i}$ are equal to zero, a solution to the variational integrator as well as the DEL equations is obtained.

{All the quantities used in \cref{prop::eval} can be recursively computed in the tree representation, therefore, we have \cref{algorithm::drha1} that evaluates the DEL equations, which essentially consists of $s+1$ forward passes from root to leaf nodes and $s+1$ backward passes in the reverse order, thus totally takes $O(sn)$ time.  In contrast, automatic differentiation and our prior methods \cite{johnson2009scalable,johnson2015structured} take $O(sn^2)$ time to evaluate the DEL equations.}

\begin{algorithm}[t]
	\caption{Recursive Evaluation of the DEL Equations}
	\label{algorithm::drha1}
	\begin{algorithmic}[1]
		\State initialize $\qka{g_0}=\I$ and $\qka{\lv_0}=0$
		\For{$i=1\rightarrow n$}
		\vspace{0.2em}
		\For{$\alpha=0\rightarrow s$}
		\vspace{0.2em}
		\State $g^{k,\alpha}_{i}=g_{\pa(i)}^{k,\alpha}g^{k,\alpha}_{\pa(i),i}(q_i^{k,\alpha})$
		\vspace{0.2em}
		\State $\lS_i^{k,\alpha} = \Ad_{g_{i}^{k,\alpha}} S_i$,\quad $\lM_i^{k,\alpha} = \Ad_{g_{i}^{k,\alpha}}^{-T} M_i { \Ad^{-1}_{g_{i}^{k,\alpha}}}$
		\State $\dot{q}_i^{k,\alpha}=\frac{1}{\Delta t}\sum\limits_{\beta=0}^s\qab{b}{\qikb}$,\quad $\lv^{k,\alpha}_{i}= \lv^{k,\alpha}_{\pa(i)}+ {{\lS_i^{k,\alpha}}}\cdot \dot{q}_i^{k,\alpha}$
		\EndFor
		\EndFor
		\For{$i=n\rightarrow 1$}
		\vspace{0.2em}
		\For{$\alpha=0\rightarrow s$}
		\vspace{0.2em}
		\State 	$\lmu^{k,\alpha}_i = {\lM_i^{k,\alpha}}\lv_i^{k,\alpha}+\sum\limits_{j\in\chd(i)}\lmu^{k,\alpha}_j$,\quad $\lG_{i}^{k,\alpha}=\lF_{i}^{k,\alpha}+\sum\limits_{j\in\chd(i)}\lG_{j}^{k,\alpha}$\\
		\vspace{-0.5em}
		\State $\qka{\lO_i}= w^{\alpha}\Delta t\cdot {\ad}_{\qka{\lv_{i}}}^T\cdot\qka{\lmu_i}+\qka{\lG_{i}}$
		\EndFor
		\vspace{0.15em}
		\State $r_i^{k,0}=p_i^{k} + {\lS_i^{k,0}}^T\lO_i^{k,0}+\sum\limits_{\beta= 0}^s a^{0\beta}{\lS_i^{k,\beta}}^T\cdot\lmu_i^{k,\beta}+Q_i^{k,0}$
		\vspace{0.15em}
		\For{$\alpha=1\rightarrow s-1$}
		\vspace{0.2em}
		\State $r_i^{k,\alpha}={\lS_i^{k,\alpha}}^T\lO_i^{k,\alpha}+\sum\limits_{\beta= 0}^s a^{\alpha\beta}{\lS_i^{k,\beta}}^T\cdot\lmu_i^{k,\beta}+\qka{Q_i}$
		\EndFor
		\vspace{0.45em}
		\State $p_i^{k+1}={\lS_i^{k,s}}^T\lO_i^{k,s}+\sum\limits_{\beta= 0}^s a^{s\beta}{\lS_i^{k,\beta}}^T\cdot\lmu_i^{k,\beta}+Q_i^{k,s}$
		\vspace{0.1em}
		\EndFor
	\end{algorithmic}
\end{algorithm}
\vspace{-0.5em}
\subsection{Exact Newton Direction Computation}\label{subsection::newton}
From \cref{eq::DEL_general}, the Newton direction $\delta \qk{\lq}=\begin{bmatrix}
\delta {q^{k,1}}^T,\,\cdots,\,{\delta\qks{q}}^T
\end{bmatrix}^T\in\R^{sn}
$ is computed as $\delta \qk{\lq} = -{\JJ^k}^{-1}(\qk{\lq}) \cdot\qk{r}$
in which $\JJ^k(\qk{\lq})\in\R^{sn\times sn}$ is the Jacobian of \cref{eq::DEL_general1,eq::DEL_general2} with respect to control points $q^{k,1},\,\cdots,\,\qks{q}$, and $\qk{r}\in\R^{sn}$ is the residue of evaluating the DEL equations \cref{eq::DEL_general1,eq::DEL_general2} by \cref{prop::eval}.\par

In this section, we make the the following assumption on $\qka{\lF_i}$ and $\qka{Q_i}$, which is general and applies to a large number of mechanical systems in robotics.

\begin{assumption}\label{assumption::abi}
	Let $u(t)$ be the control inputs of the mechanical system, we assume that the discrete impulse $\qka{\lF_i}$ and discrete joint force $\qka{Q_i}$ can be respectively formulated as
	$\qka{\lF_i}=\qka{\lF_i}(\qka{g_i},\qka{\lv_i},\qka{u}) $
	and
	$\qka{Q_i}=\qka{Q_i}(\qka{q_i},\qka{\dot{q}_i},\qka{u})$
	in which $\qka{u}=u(\qka{t})$.
\end{assumption}

If \cref{assumption::abi} holds and $\qk{\JJ}^{-1}(\qk{\lq})$ exists, it can be shown that \cite[\cref{algorithm::dabi}]{fan2018wafr_app} computes the Newton direction for variational integrators in $O(s^3n)$ time.

\begin{prop}\label{prop::dabi}
For higher-order variational integrators of unconstrained mechanical systems, if \cref{assumption::abi} holds and $\qk{\JJ}^{-1}(\qk{\lq})$ exists, the Newton direction $\delta \qk{\lq}= -{\qk{\JJ}}^{-1}(\qk{\lq})\cdot\qk{r}$ can be computed with \cite[\cref{algorithm::dabi}]{fan2018wafr_app} in $O(s^3n)$ time.
\end{prop}
\begin{proof}
See \cite[\cref{subsection::pp2}]{fan2018wafr_app}.
\end{proof} 

In \cite[\cref{algorithm::dabi}]{fan2018wafr_app}, the forward and backward passes of the tree structure take $O(s^2n)$ time, and the $n$ computations of the $s\times s$ matrix inverse takes $O(s^3n)$ time, thus the overall complexity of \cite[\cref{algorithm::dabi}]{fan2018wafr_app} is $O(s^3n+s^2n)$. In contrast, automatic differentiation and our prior methods in \cite{johnson2009scalable,johnson2015structured} take $O(s^2n^3)$ time to compute $\qk{\JJ}(\qk{\lq})$ and another $O(s^3n^3)$ time to compute the $sn\times sn$ matrix inverse $\qk{\JJ}^{-1}(\qk{\lq})$, and the overall complexity is $O(s^3n^3+s^2n^3)$. Though the quasi-Newton method \cite{lee2016linear} is $O(n)$ time for second-order variational integrator in which $s=1$, it requires small time steps and can not be used for third- or higher-order variational integrators.

{Therefore, both \cref{algorithm::drha1} and \cite[\cref{algorithm::dabi}]{fan2018wafr_app} have $O(n)$ complexity for a given $s$, which results in a linear-time variational integrator. Furthermore, \cref{algorithm::drha1} and \cite[\cref{algorithm::dabi}]{fan2018wafr_app} have no restrictions on the number of control points, which indicates that the resulting linear-time variational integrator can be arbitrarily high order. \textit{To our knowledge, this is the first exactly linear-time third- or higher-order implicit integrator for mechanical systems.}}

\subsection{Extension to Constrained Mechanical Systems}
Thus far all our discussions of linear-time variational integrators have been restricted to unconstrained mechanical systems. However, \cref{algorithm::drha1} and \cite[\cref{algorithm::dabi}]{fan2018wafr_app} can be extended to constrained mechanical systems as well.\par 

In terms of the the DEL equation evaluation, the extension to constrained mechanical systems is immediate. From \cref{eq::DEL_general_c}, we only need to add the constraint term $\qka{A}(\qka{q})\cdot \qka{\lambda}$ to the results of using \cref{algorithm::drha1}.

If the variational integrator is second-order and the mechanical system has $m$ constraints, it is possible to compute the Newton direction $\delta q^{k+1}$ and $\delta\lambda^k$ in $O(mn)+O(m^3)$ time using \cite[\cref{algorithm::dabi}]{fan2018wafr_app}. In accordance with \cref{eq::DEL_general_c}, $\delta q^{k+1}$ and $\delta\lambda^k$ should satisfy
$\qk{\JJ}(\qk{q})\cdot\delta q^{k+1} + \qk{A}(\qk{q})\cdot\delta\qk{\lambda}= -r_q^k$ and $\tD h^k(q^{k+1},\dot{q}^{k+1})\cdot\delta q^{k+1}= -r_c^k $
in which $\qk{r_q}$ and $\qk{r_c}$ are equation residues. Then $\delta q^{k+1}$ and $\delta\lambda^k$ can be computed as follows: i) compute $\delta q_r^{k+1}=-{\qk{\JJ}}^{-1}\cdot r_q^k$ with \cite[\cref{algorithm::dabi}]{fan2018wafr_app} which takes $O(n)$ time; ii) compute ${\qk{\JJ}}^{-1}\cdot\qk{A}$ by using \cite[\cref{algorithm::dabi}]{fan2018wafr_app} $m$ times which takes $O(mn)$ time; iii) compute $\delta \lambda^{k}=\big(\tD h^k\cdot\qk{\JJ}\cdot \qk{A}\big)^{-1}(r_c^k+\tD h^k\cdot\delta q_{r}^{k+1})$ which takes $O(m^3)$ time; iv) compute $\delta q^{k+1}=\delta q_r^{k+1}-{\qk{\JJ}}^{-1}\cdot\qk{A}\cdot\delta\lambda^k$.

In regard to third- or higher-order variational integrators, if the constraints are of $h_i^k(\qka{g_i},\qka{\lv_i})=0$ or $h_i^k(\qka{q_i},\qka{\dot{q}_i})=0$ or both for each $i=1,\,2,\,\cdots,\,n$, \cite[\cref{algorithm::dabi}]{fan2018wafr_app} can be used to compute the Newton direction $\delta\qk{\lq}$ and $\delta\qk{\lambda}$ in a similar procedure to the second-order variational integrator.

In next section, we will discuss the linearization of higher-order variational integrators in $O(n^2)$ time.

\section{The Linearization of Higher-Order Variational Integrators}\label{section::quad_lin}

The linearization of discrete time systems is useful for trajectory optimization, stability analysis, controller design, etc., which are import tools in robotics.

From \cref{eq::DEL_general,eq::DEL_general_c}, the linearization of variational integrators is comprised of the computation of $\tD^2 \cL_d(\qk{\lq})$, $\tD \qka{\cF_d}(\qka{t})$ and $\tD \qka{A}(\qka{q})$. In most cases, $\tD\qka{\cF_d}(\qka{t})$ and $\tD \qka{A}(\qka{q})$ can be efficiently computed in $O(n^2)$ time, therefore, the linearization efficiency is mostly affected by $\tD^2 \cL_d(\qk{\lq})$.

It is by definition that the Lagrangian of a mechanical system is $\cL(q,\dot{q})= K(q,\dot{q})-V(q)$
in which $K(q,\dot{q})$ is the kinetic energy and $V(q)$ is the potential energy, and from \cref{eq::ddl}, the computation of $\tD^2 \cL_d(\qk{\lq})$ is actually to compute $\frac{\partial K}{\partial \dot{q}^2}$, $\frac{\partial^2 K }{\partial \dot{q}\partial q}$, $\frac{\partial^2 K }{\partial q\partial \dot{q}}$, $\frac{\partial^2 K }{\partial q^2}$ and $\frac{\partial V}{\partial q^2}$, for which we have \cref{prop::dkdq} and \cref{prop::dvkdq} as follows.

\begin{prop}\label{prop::dkdq}
For the kinetic energy $K(q,\dot{q})$ of a mechanical system, $\frac{\partial^2 K }{\partial \dot{q}^2}$, $\frac{\partial^2 K }{\partial \dot{q}\partial q}$, $\frac{\partial^2 K }{\partial q\partial \dot{q}}$, $\frac{\partial^2 K}{\partial q^2}$ can be recursively computed with \cref{algorithm::dkdv} in $O(n^2)$ time.
\end{prop}
\begin{proof}
See \cite[\cref{subsection::pp3}]{fan2018wafr_app}.
\end{proof}

In the matter of potential energy $V(q)$, we only consider the gravitational potential energy $V_\gv(q)$, and the other types of potential energy can be computed in a similar way.

\begin{prop}\label{prop::dvkdq}
	If $\gv\in \R^3$ is gravity, then for the gravitational potential energy $V_{\vec{g}}(q)$,  $\frac{\partial^2 V_{\vec{g}} }{\partial q^2}$ can be recursively computed with \cref{algorithm::dvdq} in $O(n^2)$ time.
\end{prop}
\begin{proof}
	See \cite[\cref{subsection::pp4}]{fan2018wafr_app}.
\end{proof}

 In regard to \cref{prop::dvkdq} and \cref{algorithm::dvdq}, we remind the reader of the notation introduced in \cref{subsection::lie,subsection::tree} that $m_i\in\R$ is the mass of rigid body $i$, $p_i\in \R^3$ is the mass center of rigid body $i$ as well as the origin of frame $\{i\}$, and $\lS_i=\begin{bmatrix}
\overline{s}_i^T & \overline{n}_i^T
\end{bmatrix}^T\in \R^6$ is the spatial Jacobian of joint $i$ with respect to frame $\{0\}$.
\begin{figure}[!htb]
	\vspace{-2em}
	\begin{minipage}{\columnwidth}
	\begin{algorithm}[H]
		\caption{Recursive Computation of $\frac{\partial^2 K}{\partial \dot{q}^2}$, $\frac{\partial^2 K }{\partial \dot{q}\partial q}$, $\frac{\partial^2 K }{\partial q\partial \dot{q}}$, $\frac{\partial^2 K }{\partial q^2}$}
		\label{algorithm::dkdv}
		\begin{algorithmic}[1]
			\State initialize $g_0=\I$ and $\lv_0=0$
			\For{$i=1\rightarrow n$}
			\State $g_{i}=g_{\pa(i)}g_{\pa(i),i}(q_i)$
			\vspace{0.2em}
			\State $\lM_i = \Ad_{g_{i}}^{-T} M_i { \Ad^{-1}_{g_{i}}}$,\;\; $\lS_i = \Ad_{g_{i}} S_i$
			\State $\lv_{i}= \lv_{\pa(i)}+ {{\lS_i}}\cdot \dot{q}_i$,\;\; $\dot{\lS}_i = \ad_{\lv_{i}} \lS_i$
			\EndFor
			\vspace{0.1em}
			\State initialize $\frac{\partial^2 K}{\partial \dot{q}^2}=\0$,\;\; $\frac{\partial^2 K }{\partial \dot{q}\partial q}=\0$,\;\; $\frac{\partial^2 K }{\partial q\partial \dot{q}}=\0$,\;\; $\frac{\partial^2 K }{\partial q^2}=\0$
			\vspace{0.1em}
			\For{$i=n\rightarrow 1$}
			\State 	$\lmu_i = {\lM_i}\lv_i+\sum\limits_{j\in\chd(i)}\lmu_j$,\quad $\M_i=\lM_i+\sum\limits_{j\in\chd(i)}\M_{j}$
			\State $\M_i^A=\M_i {\lS}_i$,\;\; $\M_i^B=\M_i \dot{\lS}_i-\ad_{\lmu_i}^D\lS_i$
			\vspace{0.2em}
			\For{$j\in \mathrm{anc}(i)\cup\{i\}$}
			\State $\frac{\partial^2 K}{\partial \dot{q}_i\partial \dot{q}_j}=\frac{\partial^2 K}{\partial \dot{q}_j\partial \dot{q}_i}=\lS_j^T \M_i^A$
			\vspace{0.1em}
			\State $\frac{\partial^2 K}{\partial \dot{q}_i \partial q_j}=\frac{\partial^2 K}{\partial q_j \partial \dot{q}_i}=\dot{\lS}_j^T \M_i^A$,\quad $\frac{\partial^2 K}{\partial q_i \partial \dot{q}_j}=\frac{\partial^2 K}{\partial \dot{q}_j \partial q_i}=\lS_j^T \M_i^B$
			\vspace{0.1em}
			\State $\frac{\partial^2 K}{\partial q_i\partial q_j}=\frac{\partial^2 K}{\partial q_j\partial q_i}=\dot{\lS}_j^T \M_i^B$
			\EndFor
			\EndFor
		\end{algorithmic}
	\end{algorithm}
	\vspace{-3.5em}
	\end{minipage}
	\begin{minipage}{\columnwidth}
	\begin{algorithm}[H]
		\caption{Recursive Computation of $\frac{\partial^2 V_{\vec{g}} }{\partial q^2}$}
		\label{algorithm::dvdq}
		\begin{algorithmic}[1]
			\State initialize $g_0=\I$
			\For{$i=1\rightarrow n$}
			\State $g_{i}=g_{\pa(i)}g_{\pa(i),i}(q_i)$,\;\; $\lS_i = \Ad_{g_{i}} S_i$
			\vspace{0.2em}
			\EndFor
			\vspace{0.2em}
			\State initialize $\frac{\partial^2 V_{\vec{g}} }{\partial q^2}=\0$
			\vspace{0.2em}
			\For{$i=n\rightarrow 1$}
			\vspace{0.1em}
			\State 	$\lsig_{m_i} = m_i+\sum\limits_{j\in\chd(i)}\lsig_{m_j}$,\quad $\lsig_{p_i} = m_ip_i+\sum\limits_{j\in\chd(i)}\lsig_{p_j}$
			\State $\lsig_i^A = \hat{\gv}\left(\lsig_{m_i}\cdot\overline{n}_i-\hat{\lsig}_{p_i}\cdot \overline{s}_i\right)$
			\vspace{0.2em}
			\For{$j\in \mathrm{anc}(i)\cup\{i\}$}
			\vspace{0.2em}
			\State $\frac{\partial^2 V_\gv}{\partial q_i\partial q_j}=\frac{\partial^2 V_\gv}{\partial q_j\partial q_i}=\overline{s}_j^T\cdot\lsig_i^A$
			\vspace{0.2em}
			\EndFor
			\EndFor
		\end{algorithmic}
	\end{algorithm}
\vspace{-3em}
\end{minipage}
\end{figure}

If $\frac{\partial K}{\partial \dot{q}^2}$, $\frac{\partial^2 K }{\partial \dot{q}\partial q}$, $\frac{\partial^2 K }{\partial q\partial \dot{q}}$, $\frac{\partial^2 K }{\partial q^2}$ and $\frac{\partial V}{\partial q^2}$ are computed in $O(n^2)$ time, then according to \cref{eq::qdot,eq::ddl}, the remaining computation of $\tD^2 \cL_d(\qka{\lq})$ is simply the application of the chain rule. Therefore, if the variational integrator has $s+1$ control points, the complexity of the linearization is $O(s^2n^2)$. In contrast, automatic differentiation and our prior methods \cite{johnson2009scalable,johnson2015structured} take $O(s^2n^3)$ time to linearize the variational integrators. 
\vspace{-0.5em}

\section{Comparison with Existing Methods}\label{section::discussion}
The variational integrators using \cref{algorithm::dkdv,algorithm::drha1,algorithm::dvdq} and \cite[\cref{algorithm::dabi}]{fan2018wafr_app} are compared with the linear-time quasi-Newton method \cite{lee2016linear}, automatic differentiation and the Hermite-Simpson direct collocation method, which verifies the accuracy, efficiency and scalability of our work. All the tests are run in C++ on a 3.1GHz Intel Core Xeon Thinkpad P51 laptop.
\vspace{-1em}
\subsection{Comparison with the Linear-Time Quasi-Newton Method}
\begin{figure}[!htbp]
	\vspace{-3.5em}
	\centering
	\begin{tabular}{ccc}
		\subfloat[][]{\includegraphics[trim =0mm 0mm 0mm 0mm,width=0.32\textwidth]{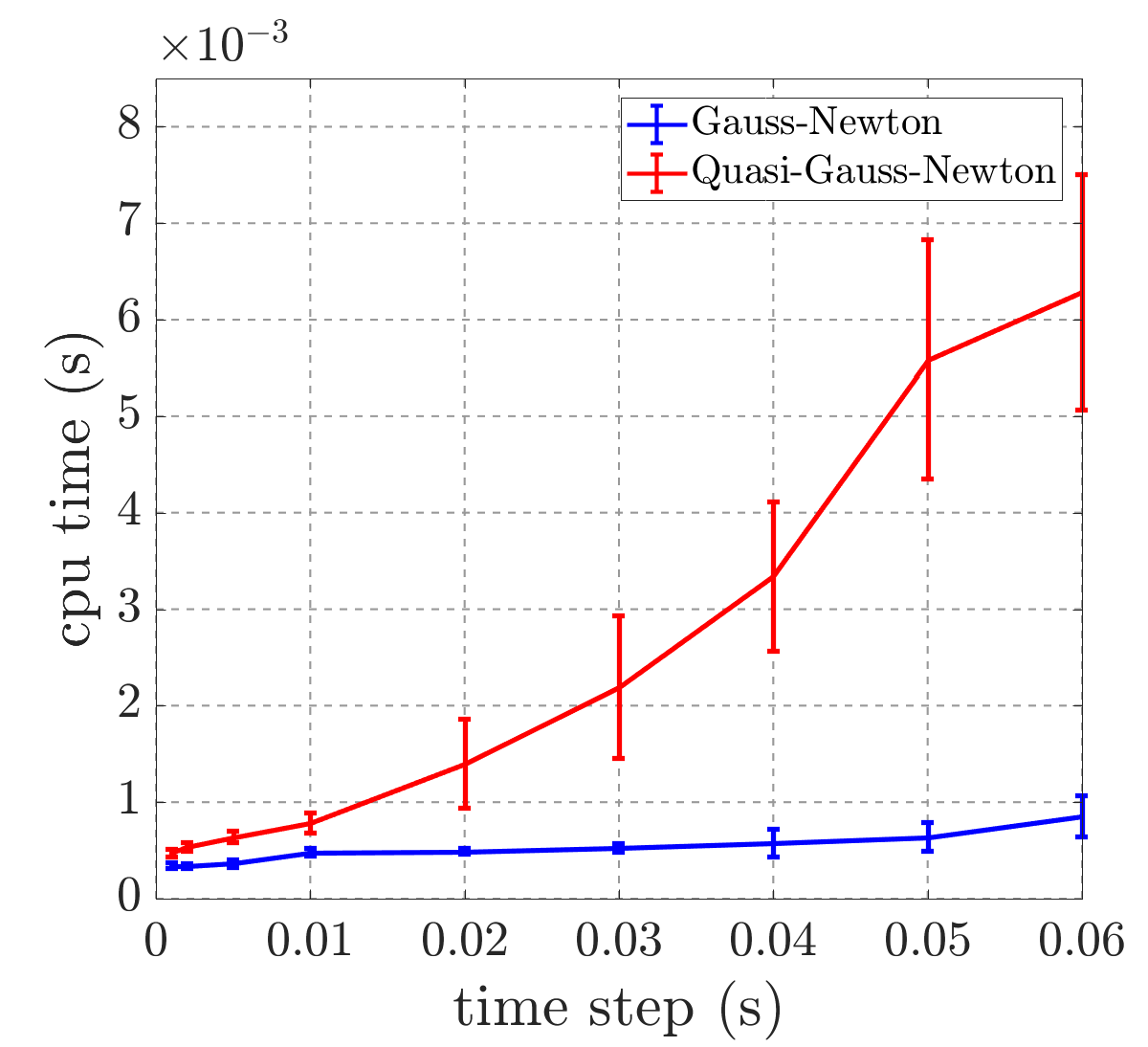}} &
		\subfloat[][]{\includegraphics[trim =0mm 0mm 0mm 0mm,width=0.32\textwidth]{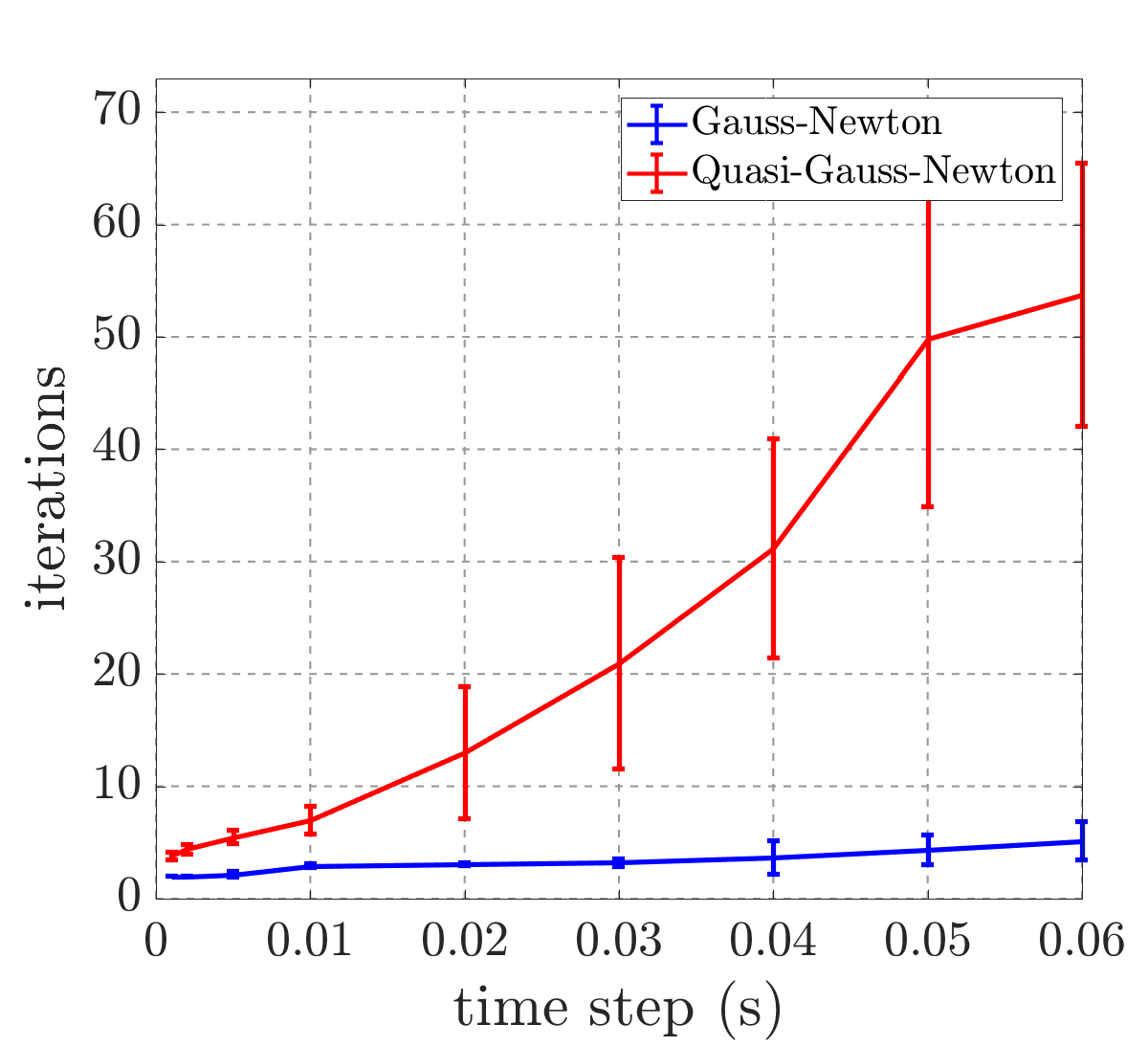}} &
		\subfloat[][]{\includegraphics[trim =0mm 0mm 0mm 0mm,width=0.32\textwidth]{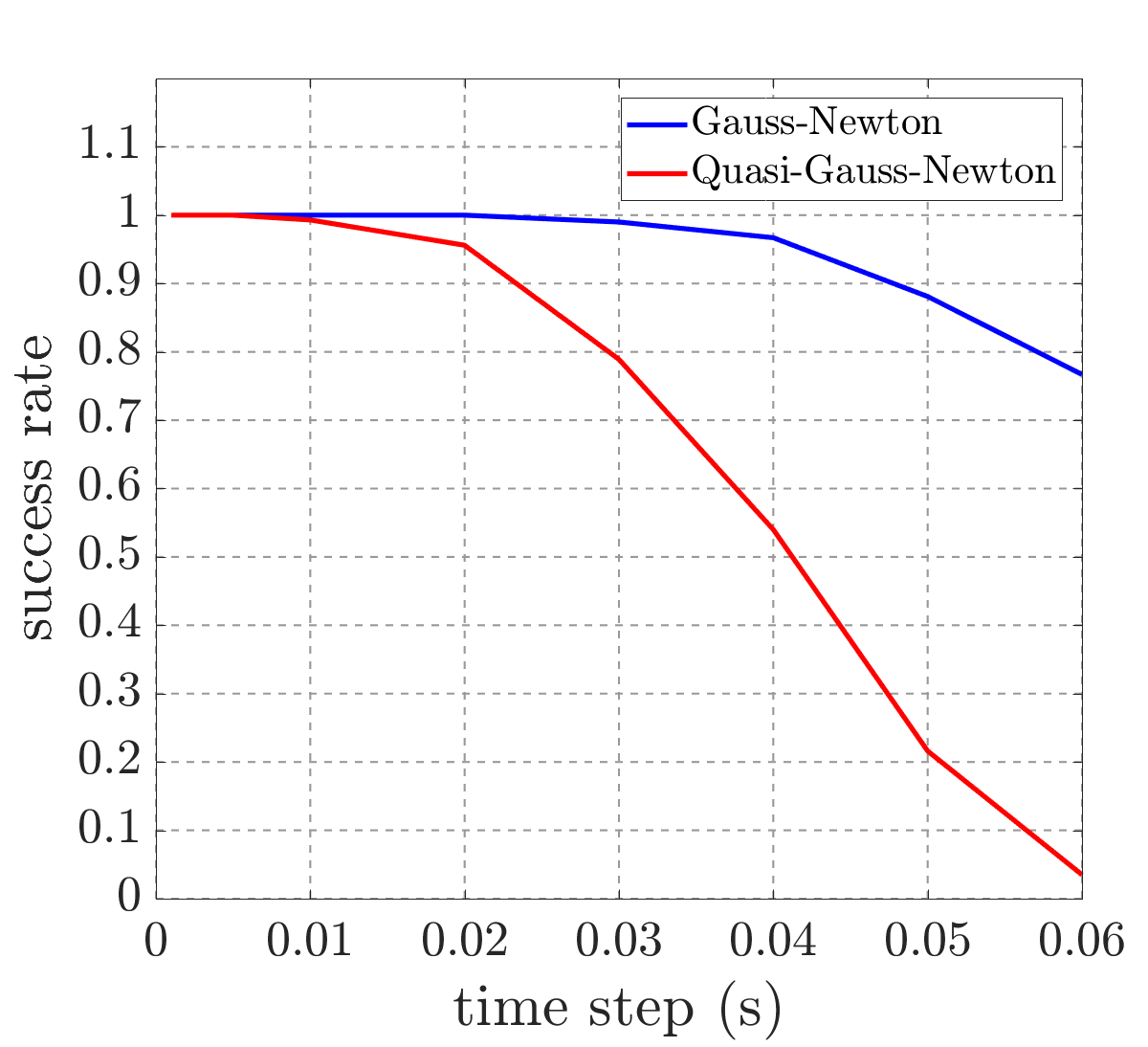}}
	\end{tabular}
\vspace{-0.5em}
	\caption{The comparison of the $O(n)$ Newton method with the $O(n)$ quasi-Newton method \cite{lee2016linear} for the trapezoidal variational integrator of a $32$-link pendulum with different time steps. The results of computational time are in (a), number of iterations in (b) and success rates in (c). Each result is calculated over 1000 initial conditions.}
	\label{fig::pendulum_time} 
	\vspace{-1.5em}
\end{figure}

In this subsection, we compare the $O(n)$ Newton method using \cref{algorithm::drha1} and \cite[\cref{algorithm::dabi}]{fan2018wafr_app} with the $O(n)$ quasi-Newton method in \cite{lee2016linear} on the trapezoidal variational integrator (\cref{example::tvi}) of a $32$-link pendulum with different time steps.

In the comparison, $1000$ initial joint angles $q^0$ and joint velocities $\dot{q}^0$ are uniformly sampled from $[-\frac{\pi}{2},\,\frac{\pi}{2}]$ for each of the selected time steps, which are $0.001$s, $0.002$s, $0.005$s, $0.01$s, $0.02$s, $0.03$s, $0.04$s, $0.05$s and $0.06$s, and the Newton and quasi-Newton methods are used to solve the DEL equations for one time step. The results are in \cref{fig::pendulum_time}, in which the computational time and the number of iterations are calculated only over initial conditions that the DEL equations are successfully solved. It can be seen that the Newton method using \cref{algorithm::drha1} and \cite[\cref{algorithm::dabi}]{fan2018wafr_app} outperforms the quasi-Newton method in \cite{lee2016linear} in all aspects, especially for relatively large time steps.
\vspace{-0.5em}
\subsection{Comparison with Automatic Differentiation}
In this subsection, we compare \cref{algorithm::dkdv,algorithm::dvdq,algorithm::drha1} and \cite[\cref{algorithm::dabi}]{fan2018wafr_app} with automatic differentiation for evaluating the DEL equations, computing the Newton direction and linearizing the DEL equations. The variational integrator used is the Simpson variational integrator (\cref{example::svi}). \par

In the comparison, we use pendulums with different numbers of links as benchmark systems. For each pendulum, $100$ initial joint angles $q^0$ and joint velocities $\dot{q}^0$ are uniformly sampled from $[-\frac{\pi}{2},\,\frac{\pi}{2}]$. The results are in \cref{fig::pendulum_autodiff} and it can be seen that our recursive algorithms are much more efficient, which is consistent with the fact that \cref{algorithm::dkdv,algorithm::dvdq,algorithm::drha1} and \cite[\cref{algorithm::dabi}]{fan2018wafr_app} are $O(n)$ for evaluating the DEL equations, $O(n)$ for computing the Newton direction, and $O(n^2)$ for linearizing the DEL equations, whereas automatic differentiation are $O(n^2)$, $O(n^3)$ and $O(n^3)$, respectively. 

\begin{figure}[!htbp]	
	\vspace{-2em}
	\centering
	\begin{tabular}{ccc}
		\subfloat[][]{\includegraphics[trim =0mm 0mm 0mm 0mm,width=0.32\textwidth]{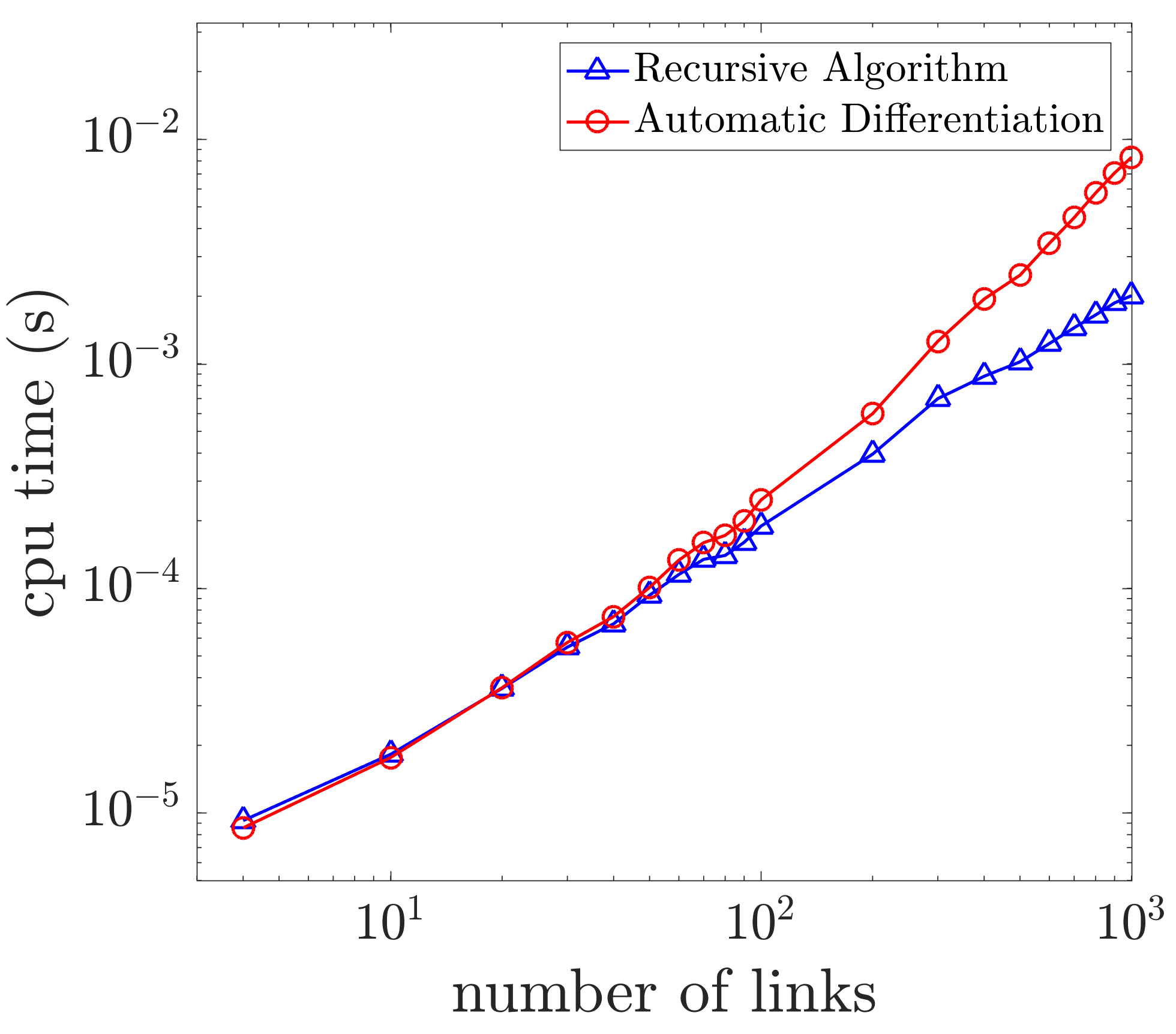}} &
		\subfloat[][]{\includegraphics[trim =0mm 0mm 0mm 0mm,width=0.32\textwidth]{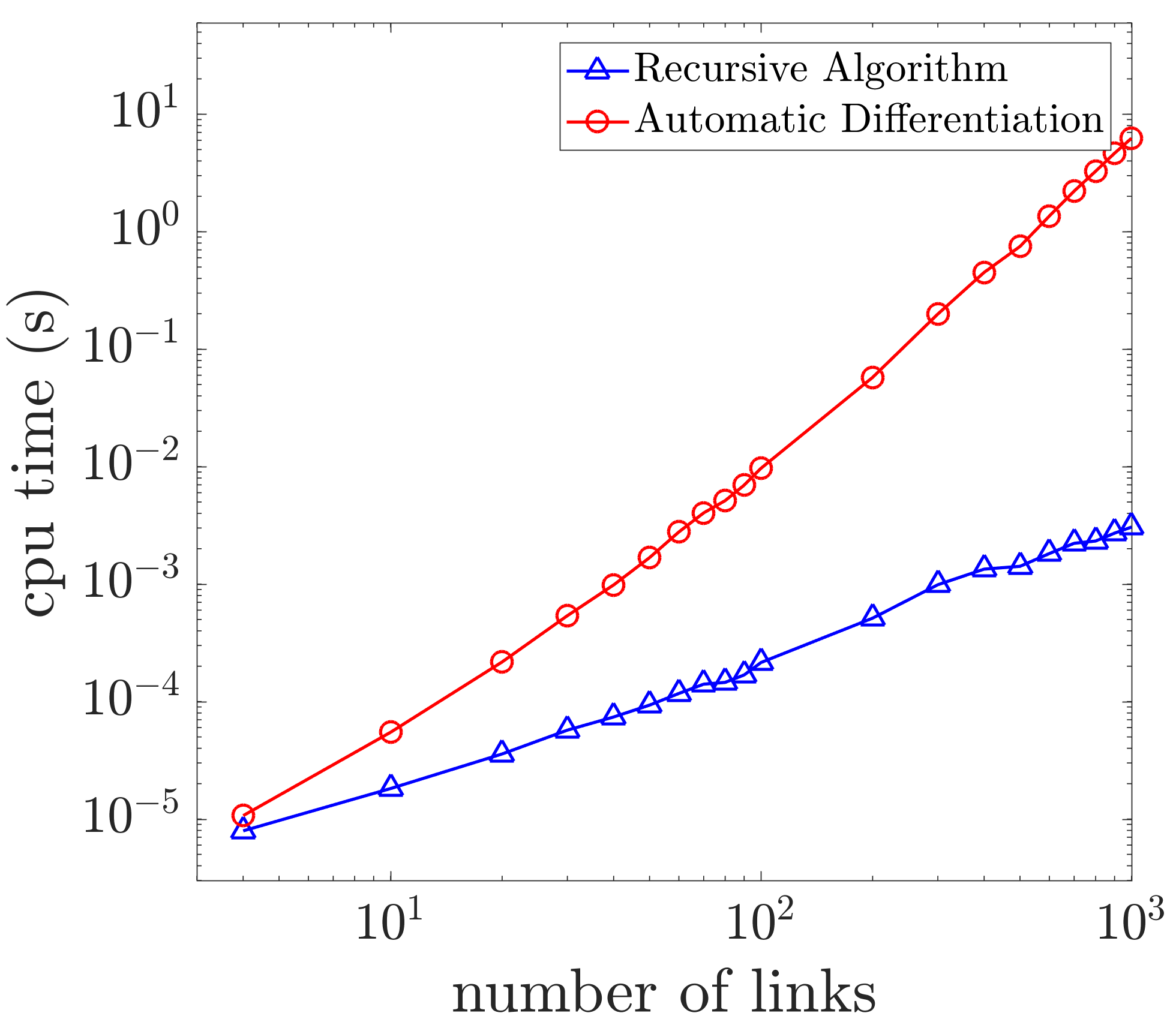}}&
		\subfloat[][]{\includegraphics[trim =0mm 0mm 0mm 0mm,width=0.32\textwidth]{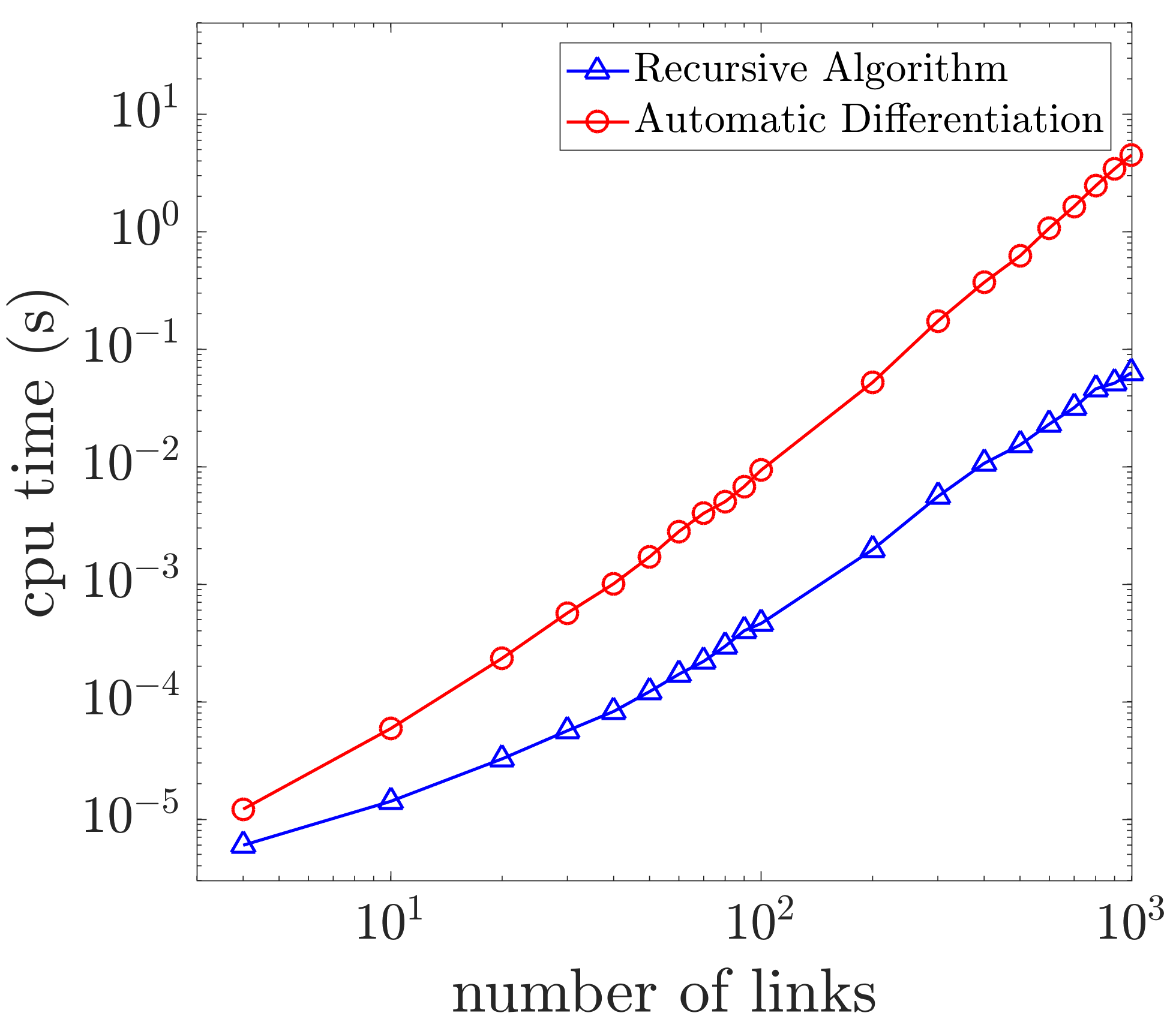}} 	
	\end{tabular}
\vspace{-0.5em}
	\caption{The comparison of our recursive algorithms with automatic differentiation for pendulums with different numbers of links. The variational integrator used is the Simpson variational integrator. The results of evaluating the DEL equations are in (a), computing the Newton direction in (b) and linearizing the DEL equations in (c). Each result is calculated over 100 initial conditions.}
	\label{fig::pendulum_autodiff} 
	\vspace{-1em}
\end{figure}

\subsection{Comparison with the Hermite-Simpson Direct Collocation Method}
{In this subsection, we compare the fourth-order Simpson variational integrator (\cref{example::svi}) with the Hermite-Simpson direct collocation method, which is a third-order implicit integrator commonly used in robotics for trajectory optimization \cite{posa2016optimization,hereid2018dynamic}.\footnote{The Hermite-Simpson direct collocation methods used in \cite{posa2016optimization,hereid2018dynamic} are actually implicit integrators that integrate the trajectory as a second-order system in the $(q,\dot{q})$ space, whereas the variational integrators integrate the trajectory in the $(q,p)$ space.} Note that both integrators use three control points for integration.
	
The strict comparison of the two integrators for trajectory optimization is usually difficult since it depends on a number of factors, such as the target problem, the optimizers used, the optimality and feasibility tolerances, etc. Therefore, we compare the Simpson variational integrator and the Hermite-Simpson direct collocation method by listing the order of accuracy, the number of variables and the number of constraints for trajectory optimization. In general, the computational loads of optimization depends on the problem size that is directly related with the number of variables and the the number of constraints. The higher-order accuracy suggests the possibility of large time steps in trajectory optimization, which reduces not only the problem size but the computational loads of optimization as well. The results are in \cref{table::compare}.\footnote{The explicit and implicit formulations of the Hermite-Simpson direct collocation methods differ in whether the joint acceleration $\ddot{q}$ is explicitly computed or implicitly involved as extra variables. Even though the explicit formulation of the Hermite-Simpson direct collocation has less variables and constraints than the implicit formulation, it is usually more complicated for the evaluation and linearization, therefore, the implicit formulation is usually more efficient and more commonly used in trajectory optimization \cite{hereid2018dynamic}.} It can be concluded that the Simpson variational integrator is more accurate and has less variables and constraints in trajectory optimization, especially for constrained mechanical systems.
\begin{table}[!htbp]
	\begin{center}
		\begin{tabular}{ c | c | c| c }
			\hline
			integrator \quad &accuracy & \quad $\#$ of variables & \quad $\#$ of constraints \quad \\
			\hline
			variational integrator        & $4$th-order& $(4N+3)n+(2N+1)m$  &  $3Nn+(2N+1)m$\\
			direct collocation (explicit) & $3$rd-order& $(6N+3)n+(2N+1)m$  &  $4Nn+(6N+3)m$ \\
			direct collocation (implicit) & $3$rd-order& $(8N+4)n+(2N+1)m$  &  $(6N+1)n+(6N+3)m$\\
			\hline
		\end{tabular}
	\end{center}
	\caption{The comparison of the Simpson variational integrator with the Hermite-Simpson direct collocation method for trajectory optimization. The trajectory optimization problem has $N$ stages and the mechanical system has $n$ degrees of freedom, $m$ holonomic constraints and is fully actuated with $n$ control inputs. Note that both integrators use three control points for integration.}
	\vspace{-1.25em}
	\label{table::compare}
\end{table}
\begin{figure}[!htbp]	
	\vspace{-1.5em}
	\begin{tabular}{ccc}
		\subfloat[][]{\includegraphics[trim =0mm 0mm 0mm 0mm,width=0.32\textwidth]{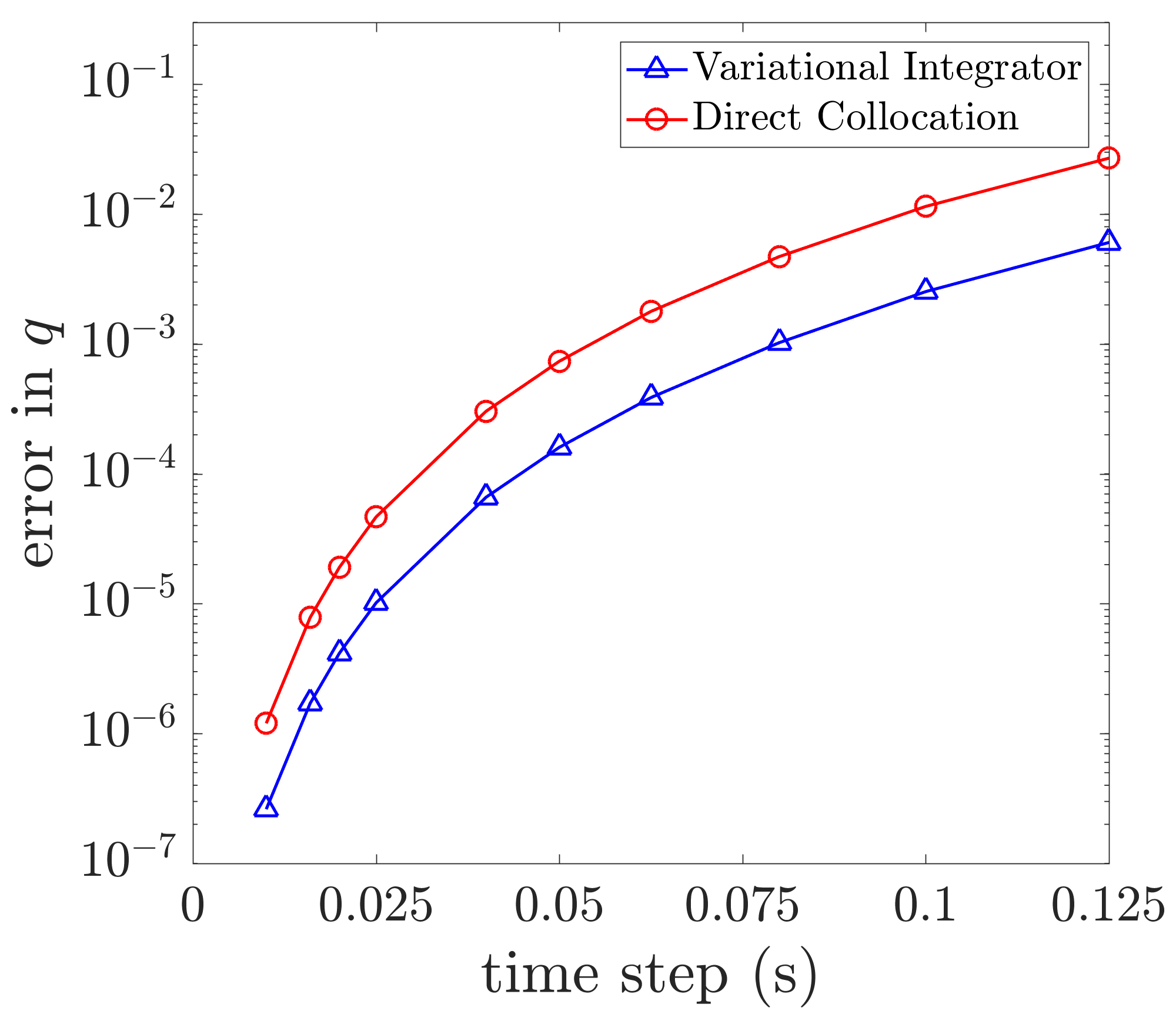}} &
		\subfloat[][]{\includegraphics[trim =0mm 0mm 0mm 0mm,width=0.32\textwidth]{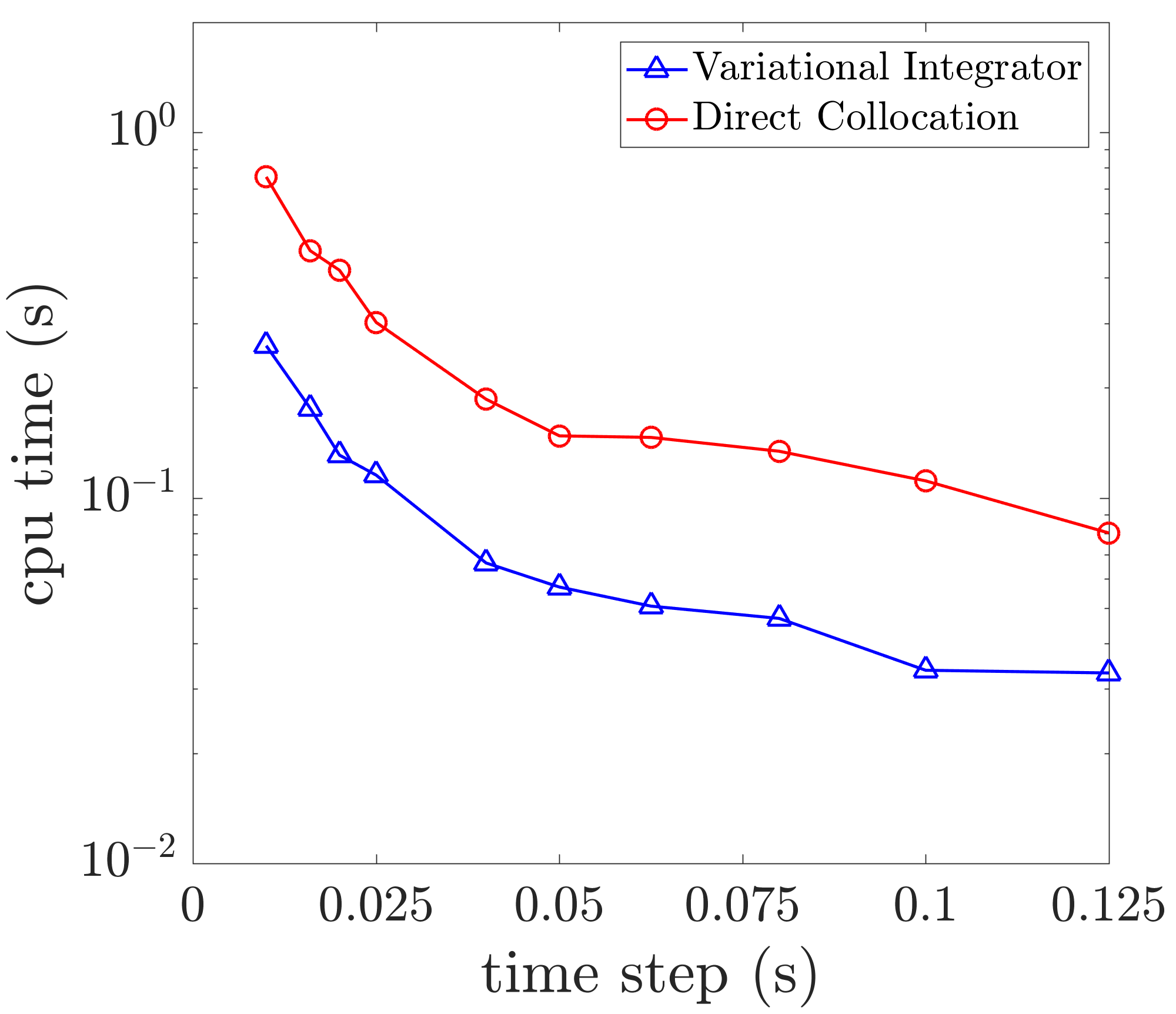}}& 
		\subfloat[][]{\includegraphics[trim =0mm 0mm 0mm 0mm,width=0.32\textwidth]{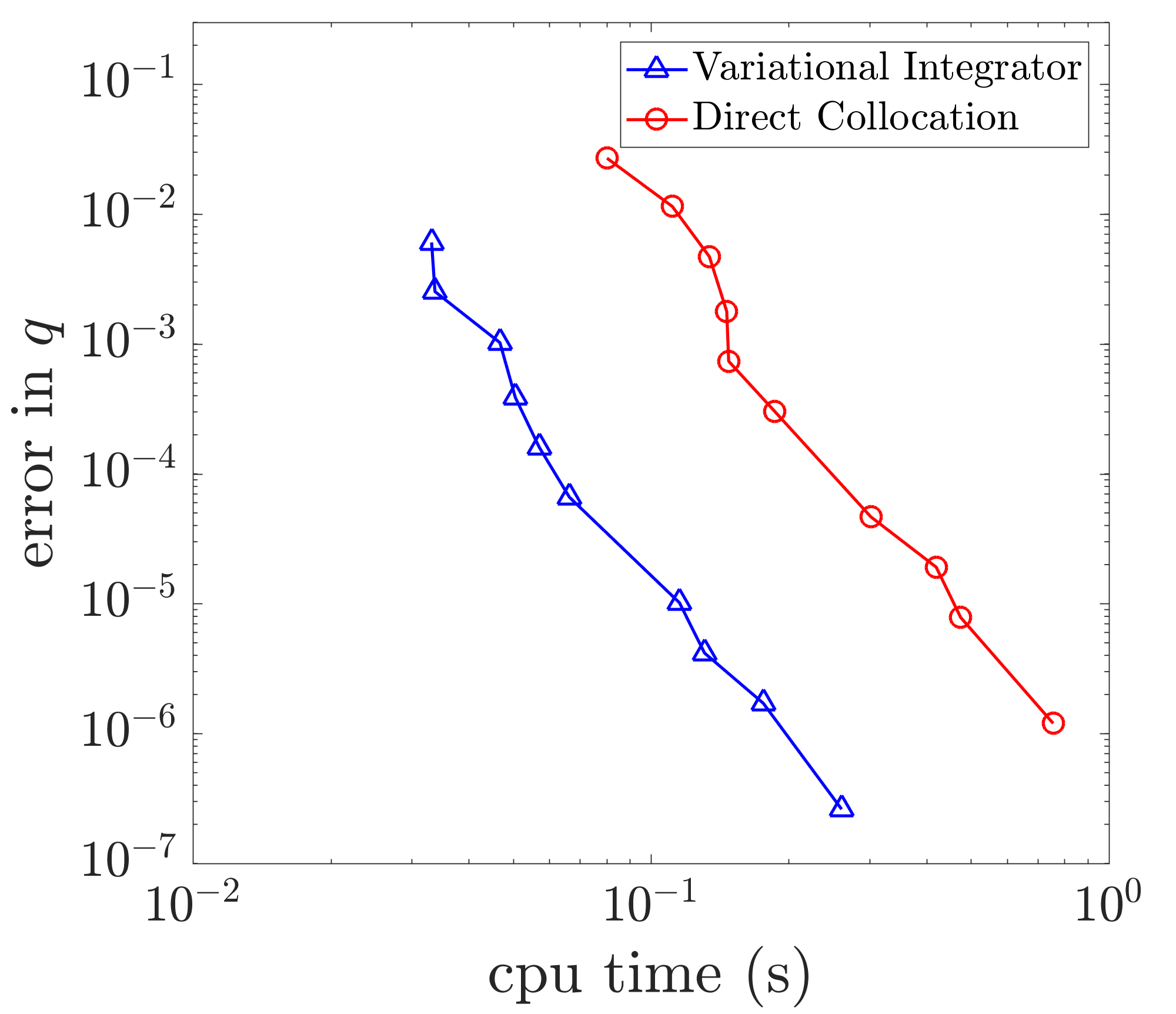}}
	\end{tabular}
\vspace{-0.5em}
	\caption{The comparison of the Simpson variational integrator with the Hermite-Simpson direction collocation method on a $12$-link pendulum with different time steps. The results of the integrator error are in (a), the computational time in (b) and the integration error v.s. computational time in (c). Each result is calculated over 100 initial conditions.}
	\label{fig::pendulum_time2} 
	\vspace{-1.75em}
\end{figure}  

The accuracy comparison in \cref{table::compare} of the Simpson variational integrator with the Hermite-Simpson direct collocation method is further numerically validated on a $12$-link pendulum. In the comparison, different time steps are used to simulate $100$ trajectories with the final time $T=10$ s, and the initial joint angles $q^0$ are uniformly sampled from $[-\frac{\pi}{12},\,\frac{\pi}{12}]$ and the initial joint velocities $\dot{q}^0$ are zero. Moreover, the Simpson variational integrator uses \cref{algorithm::drha1} and \cite[\cref{algorithm::dabi}]{fan2018wafr_app} which has $O(n)$ complexity for the integrator evaluation and the Newton direction computation, whereas the Hermite-Simpson direct collocation method uses \cite{featherstone2014rigid,carpentier2018analytical} which is $O(n)$ for the integrator evaluation and $O(n^3)$ for the Newton direction computation. For each initial condition, the benchmark solution $q_d(t)$ is created from the Hermite-Simpson direct collocation method with a time step of $5\times 10^{-4}$ s and the simulation error in $q(t)$ is evaluated as  $\frac{1}{T}\int_{0}^T\|q(t)-q_d(t)\|dt$. The running time of the simulation is also recorded. The results are in \cref{fig::pendulum_time2}, which indicates that the Simpson variational integrator is more accurate and more efficient in simulation, and more importantly, a better alternative to the Hermite-Simpson direction collocation method for trajectory optimization.} 

{In regard to the integrator evaluation and linearization, for unconstrained mechanical systems, experiments (not shown) suggest that the Simpson variational integrator using \cref{algorithm::drha1,algorithm::dkdv,algorithm::dvdq} is usually faster than the Hermite-Simpson direct collocation method using \cite{featherstone2014rigid,carpentier2018analytical} even though theoretically both integrators have the same order of complexity. However, for constrained mechanical systems, if there are $m$ holonomic constraints, the Simpson variational integrator is $O(mn)$ for the evaluation and $O(mn^2)$ for the linearization while the Hermite-Simpson direct collocation method in \cite{posa2016optimization,hereid2018dynamic} is respectively $O(mn^2)$ and $O(mn^3)$, the difference of which results from that the Hermite-Simpson direct collocation method is more complicated to model the constrained dynamics.}

\vspace{-0.5em}
\section{Implementation for Trajectory Optimization}\label{section::examples}
In this section, we implement the fourth-order Simpson variational integrator (\cref{example::svi}) with \cref{algorithm::dkdv,algorithm::drha1,algorithm::dvdq} on the Spring Flamingo robot \cite{pratt1998intuitive}, the LittleDog robot \cite{shkolnik2011bounding} and the Atlas robot \cite{nelson2012petman} for trajectory optimization, the results of which are included in our supplementary videos. It should be noted that the variational integrators used in \cite{johnson2009scalable,johnson2015structured,kobilarov2009lie,manchestercontact,junge2005discrete} for trajectory optimization are second order. In \cref{subsection::dog,subsection::jump}, a LCP formulation similar to \cite{manchestercontact} is used to model the discontinuous frictional contacts with which no contact mode needs to be prespecified. These examples indicate that higher-order variational integrators are good alternatives to the direct collocation methods \cite{posa2016optimization,hereid2018dynamic}. The trajectory optimization problems are solved with SNOPT \cite{gill2005snopt}. 
\subsection{Spring Flamingo}\label{subsection::jump}
\begin{figure}[!htpb]
	\centering
	\vspace{-2.5em}
	\begin{tabular}{cccc}
		\subfloat[][$t=0$ s]{\includegraphics[trim =0mm 0mm 0mm 0mm,width=0.23\textwidth]{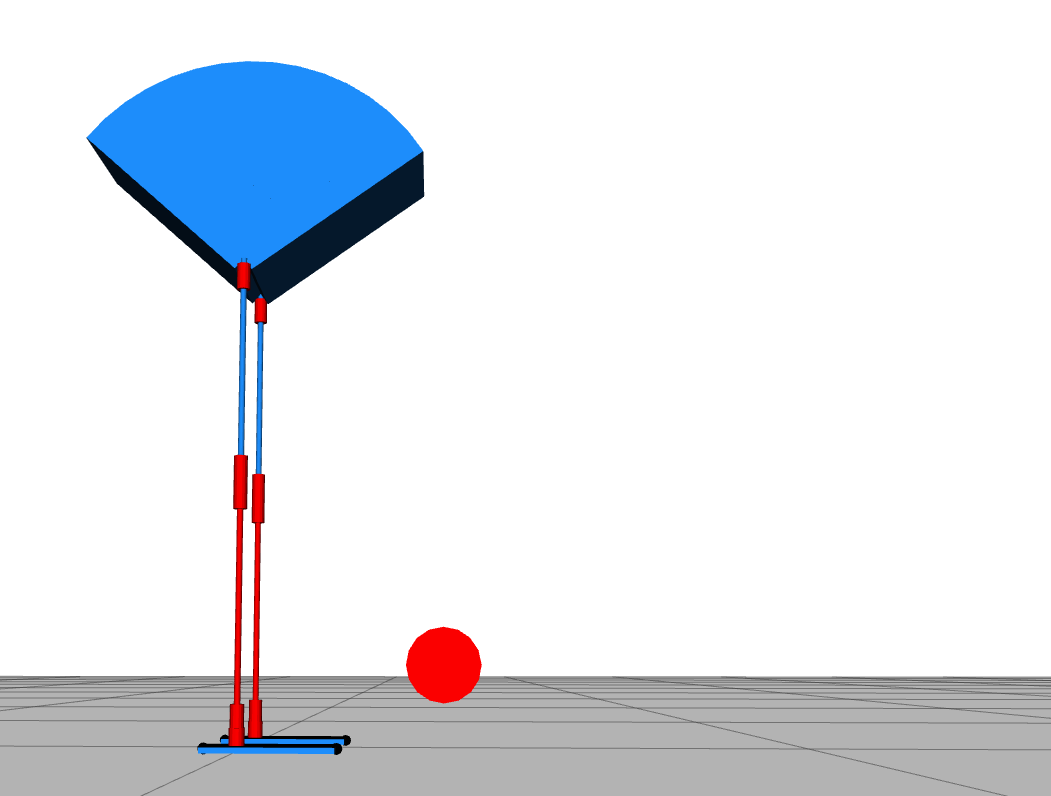}} &
		\subfloat[][$t=0.13$ s]{\includegraphics[trim =0mm 0mm 0mm 0mm,width=0.23\textwidth]{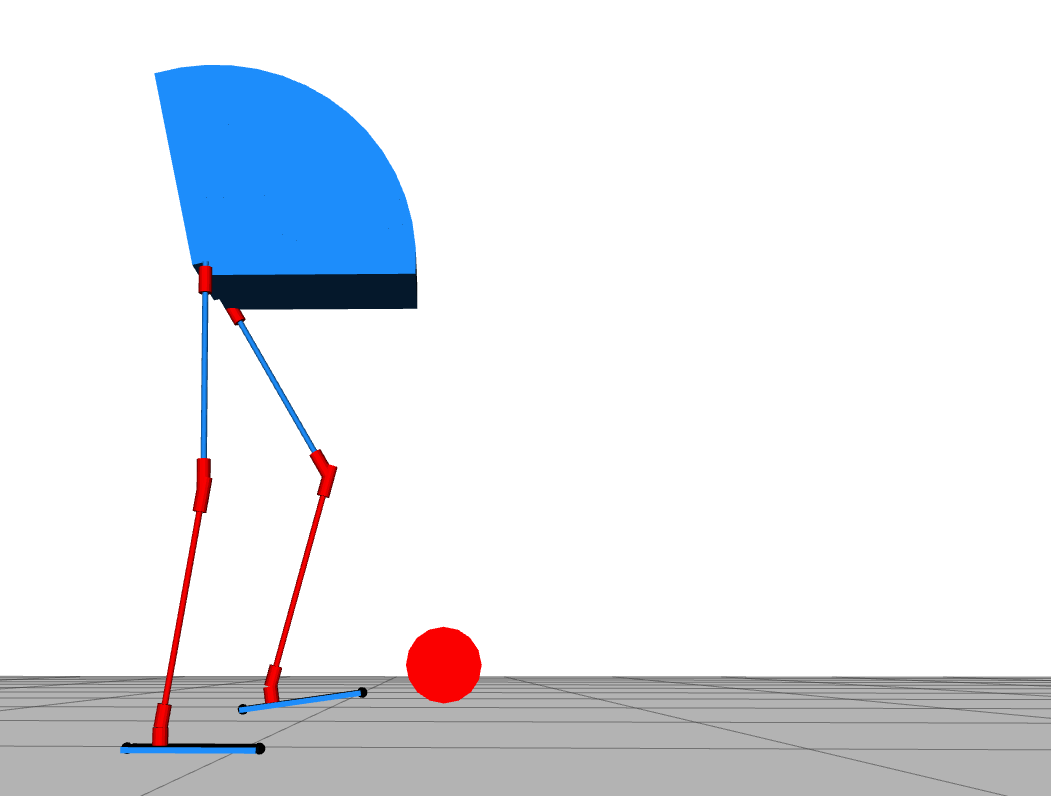}} &
		\subfloat[][$t=0.33$ s]{\includegraphics[trim =0mm 0mm 0mm 0mm,width=0.23\textwidth]{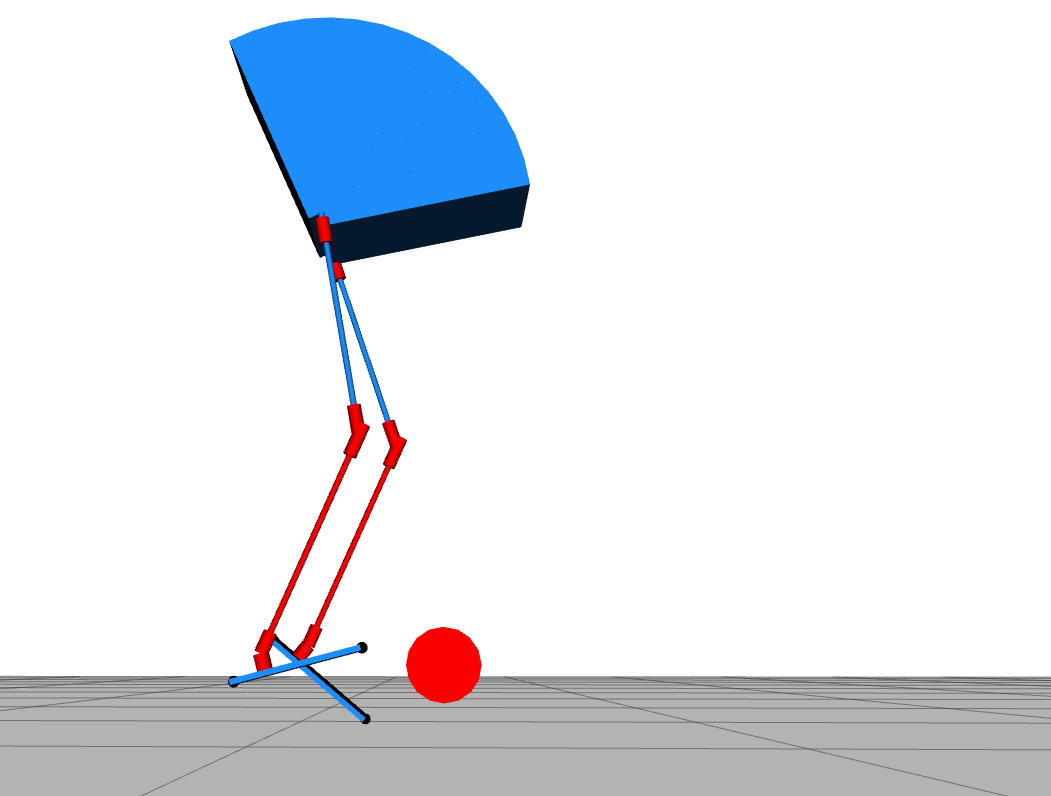}} &
		\subfloat[][$t=0.44$ s]{\includegraphics[trim =0mm 0mm 0mm 0mm,width=0.23\textwidth]{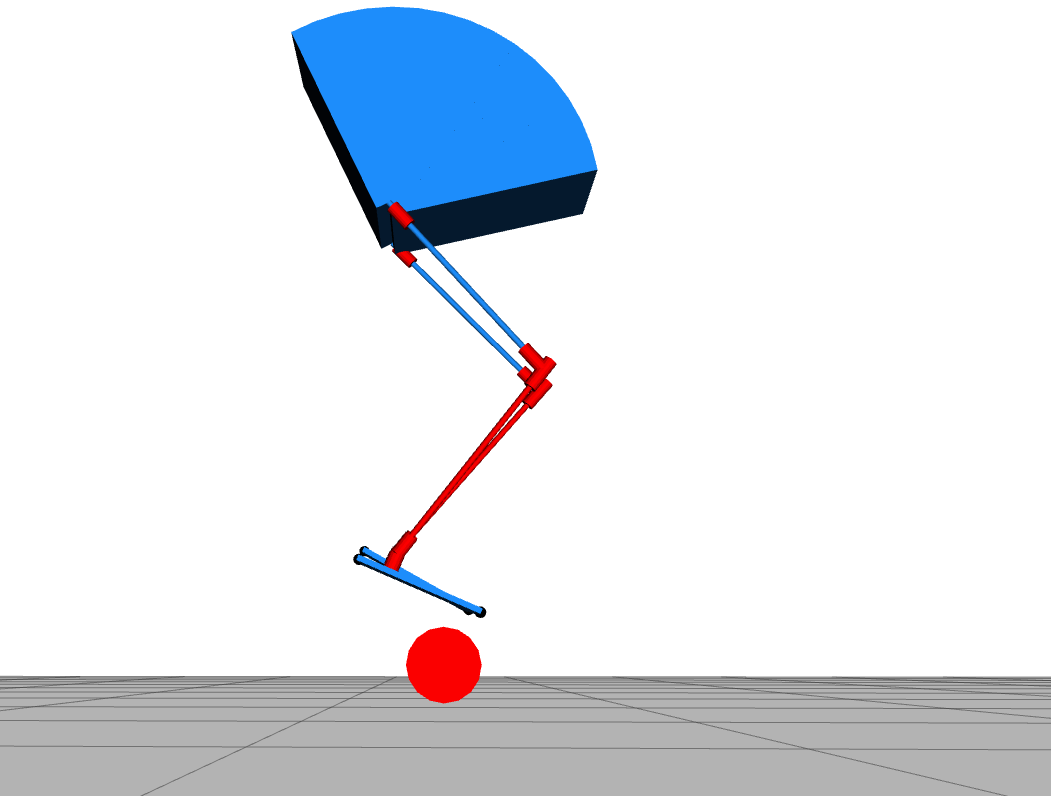}}\\[-0.7em]
		\subfloat[][$t=0.57$ s]{\includegraphics[trim =0mm 0mm 0mm 0mm,width=0.23\textwidth]{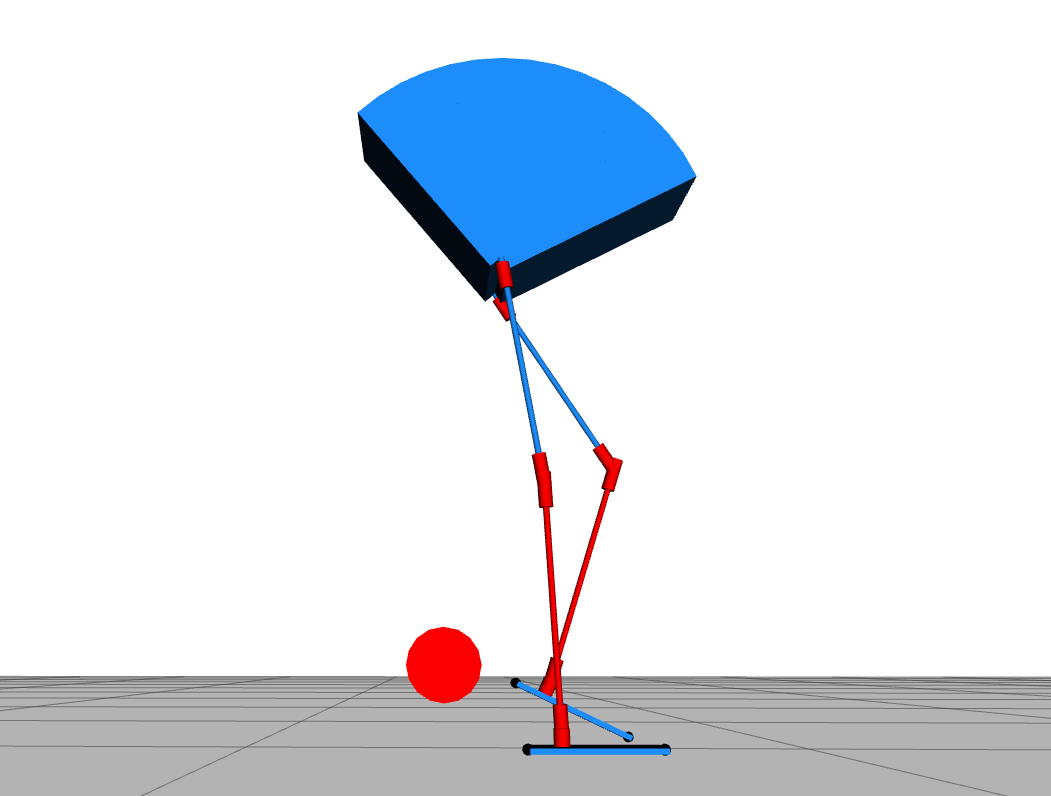}} &
		\subfloat[][$t=0.68$ s]{\includegraphics[trim =0mm 0mm 0mm 0mm,width=0.23\textwidth]{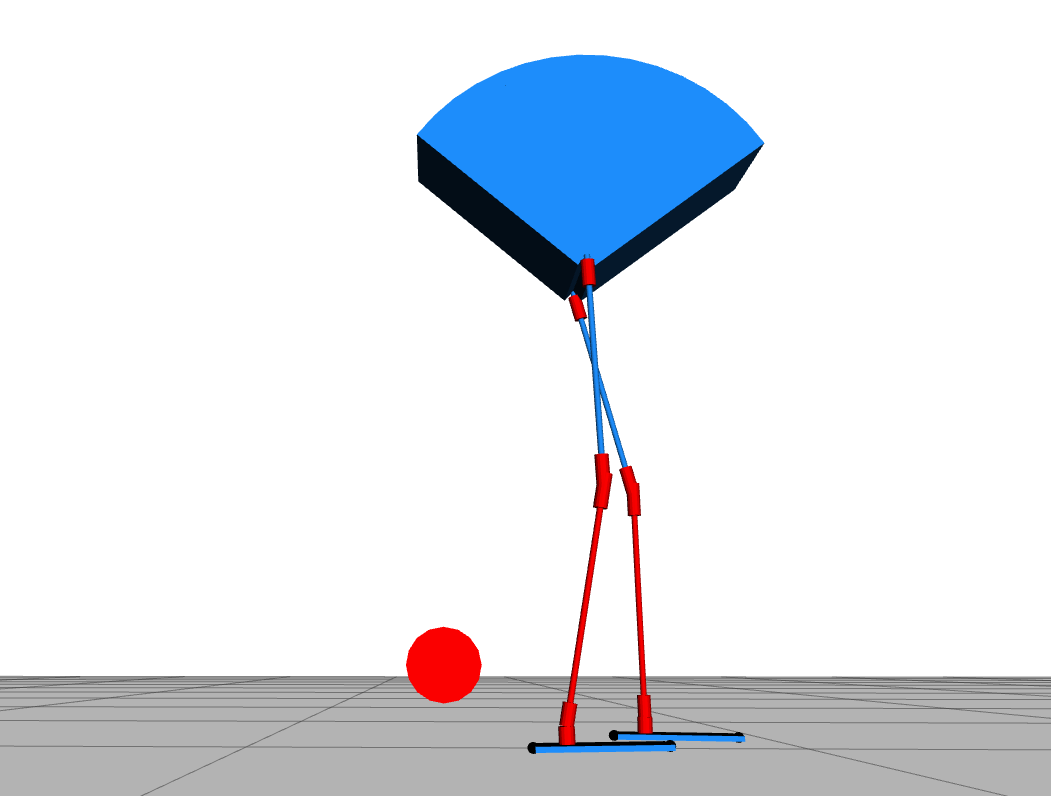}} &
		\subfloat[][$t=0.88$ s]{\includegraphics[trim =0mm 0mm 0mm 0mm,width=0.23\textwidth]{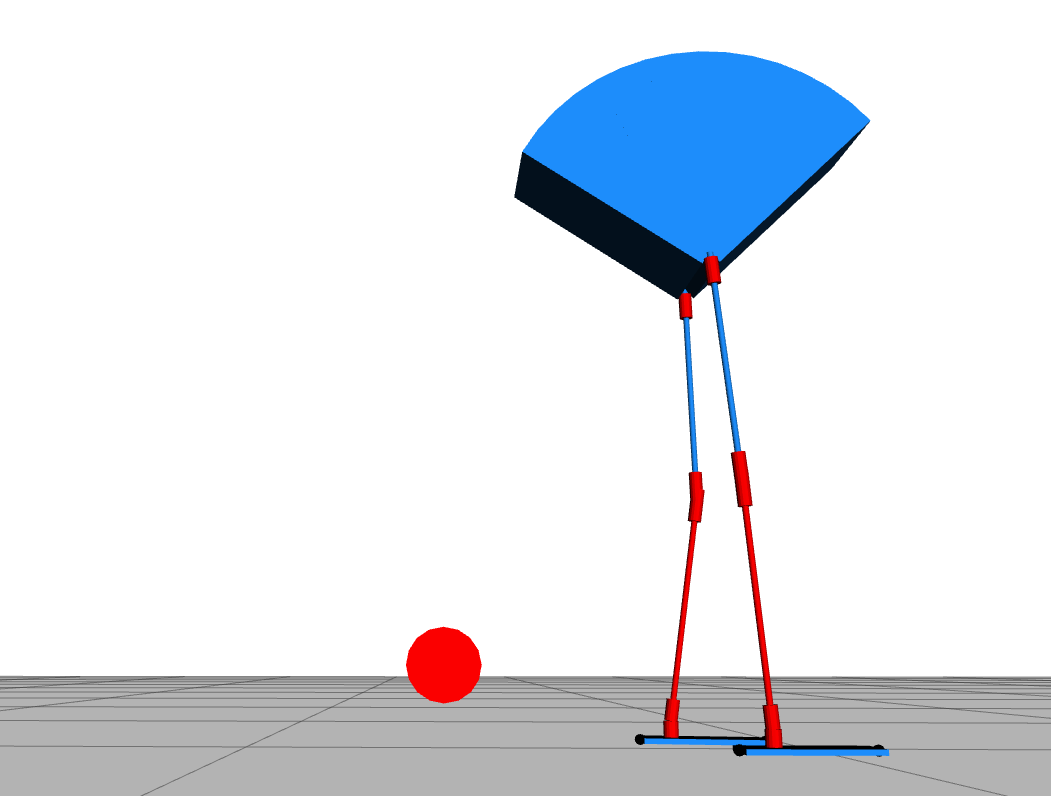}} &
		\subfloat[][$t=1.1$ s]{\includegraphics[trim =0mm 0mm 0mm 0mm,width=0.23\textwidth]{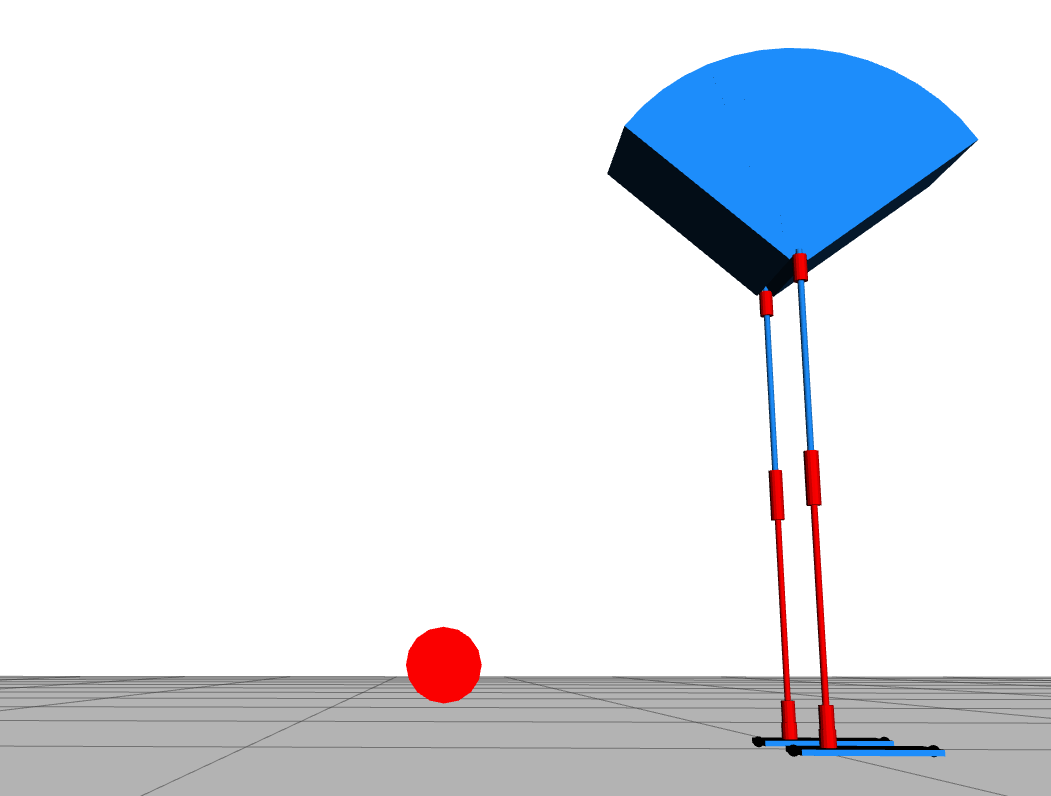}}
	\end{tabular}
	\caption{The Spring Flamingo robot jumps over a obstacle of $0.16$ meters high.}
	\vspace{-1em}
	\label{fig::walker} 
\end{figure}
The Spring Flamingo robot is a 9-DoF flat-footed biped robot with actuated hips and knees and passive springs at ankles \cite{pratt1998intuitive}. In this example, the Spring Flamingo robot is commanded to jump over an obstacle that is $0.16$ m high while walking horizontally from one position to another. The results are in \cref{fig::walker}, in which the initial walking velocity is $0.26$ m/s and the average walking velocity is around $0.9$ m/s.

\subsection{LittleDog}\label{subsection::dog}
The LittleDog robot is 18-DoF quadruped robot used in research of robot walking \cite{shkolnik2011bounding}. In this example, the LittleDog robot is required to walk over terrain with two gaps. The results are in \cref{fig::dog}, in which the average walking velocity is $0.25$ m/s. 

\subsection{Atlas}\label{subsection::atlas}

The Atlas robot is a 30-DoF humanoid robot used in the DARPA Robotics Challenge \cite{nelson2012petman}. In this example, the Atlas robot is required to pick a red ball with its left hand while keeping balanced only with its right foot. Moreover, the contact wrenches applied to the supporting foot should satisfy contact constraints of a flat foot \cite{hereid2018dynamic}. The results are in \cref{fig::atlas} and it takes around $1.3$ s for the Atlas robot to pick the ball.

\begin{figure}[t]
	\vspace{-2.5em}
	\centering
	\begin{tabular}{cccc}
		\subfloat[][$t=0$ s]{\includegraphics[trim =0mm 0mm 0mm 0mm,width=0.23\textwidth]{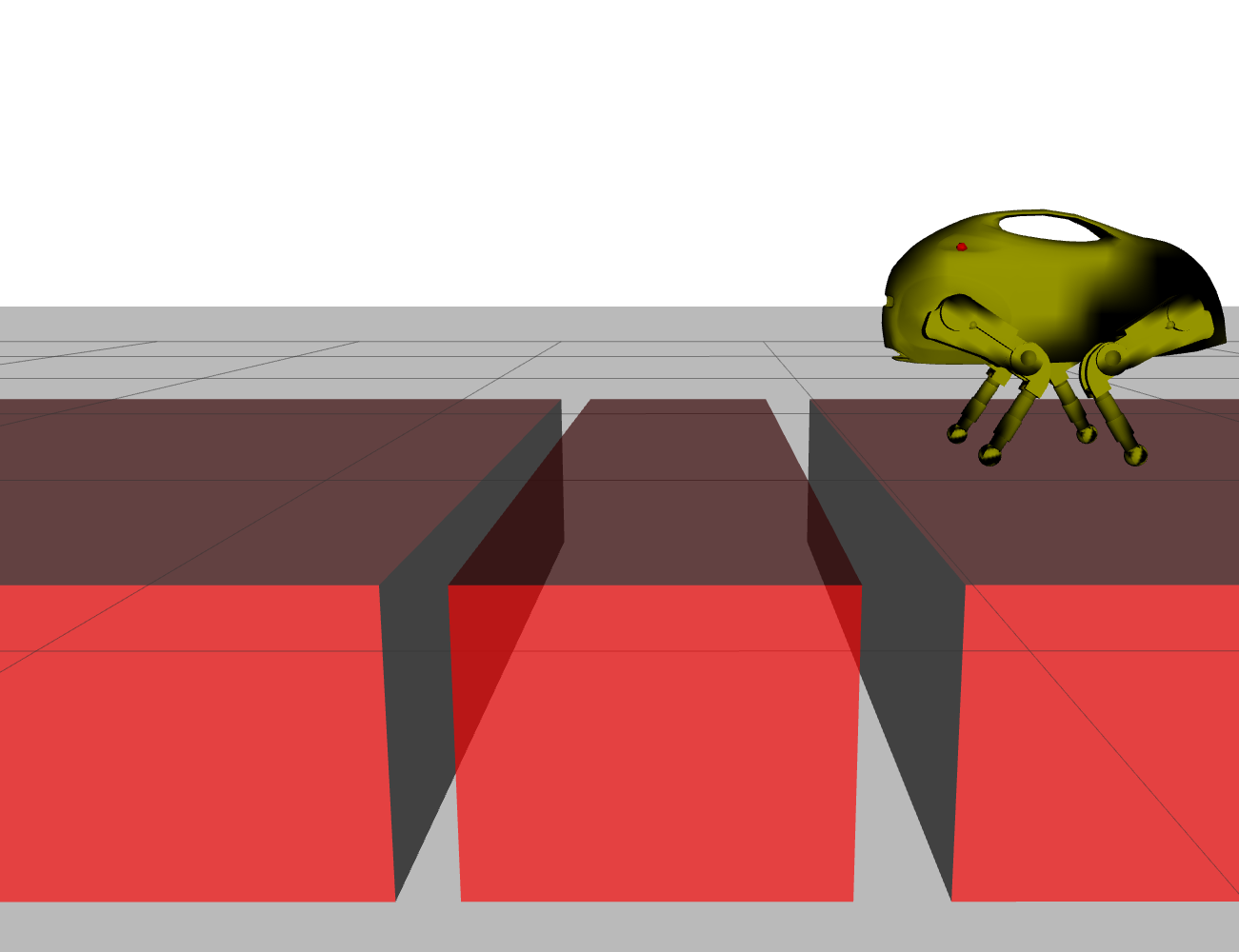}} &
		\subfloat[][$t=0.48$ s]{\includegraphics[trim =0mm 0mm 0mm 0mm,width=0.23\textwidth]{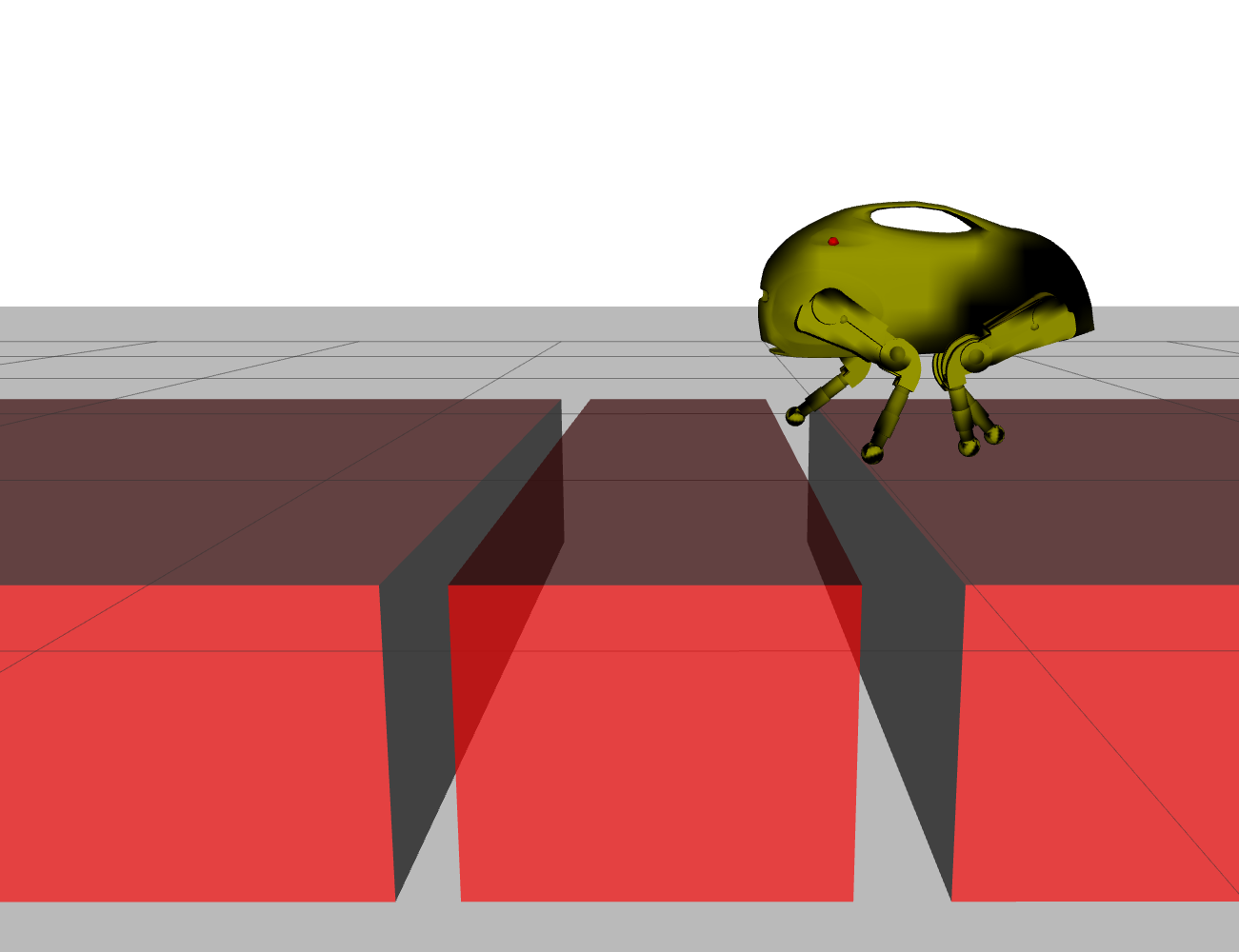}} &
		\subfloat[][$t=0.56$ s]{\includegraphics[trim =0mm 0mm 0mm 0mm,width=0.23\textwidth]{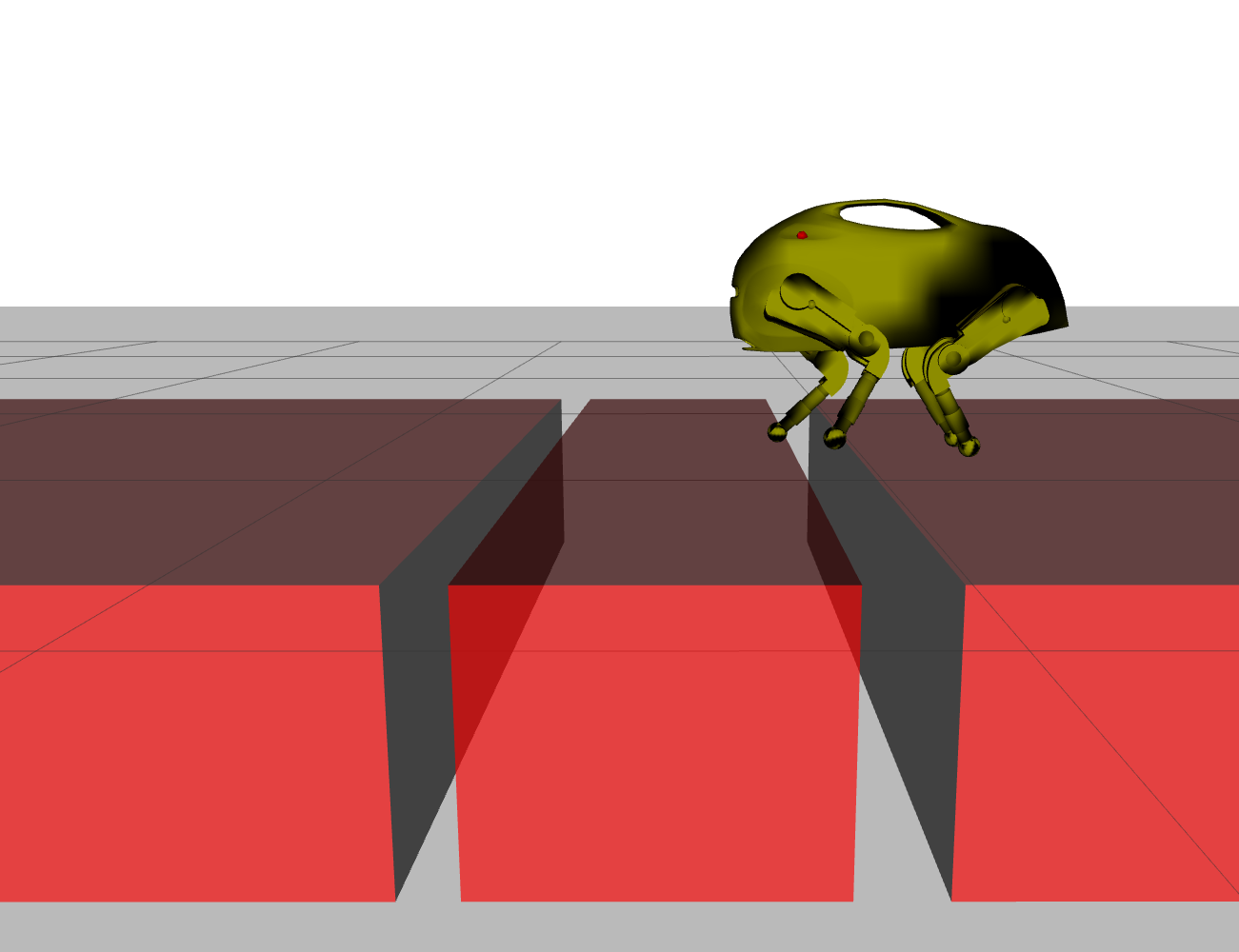}} &
		\subfloat[][$t=1.04$ s]{\includegraphics[trim =0mm 0mm 0mm 0mm,width=0.23\textwidth]{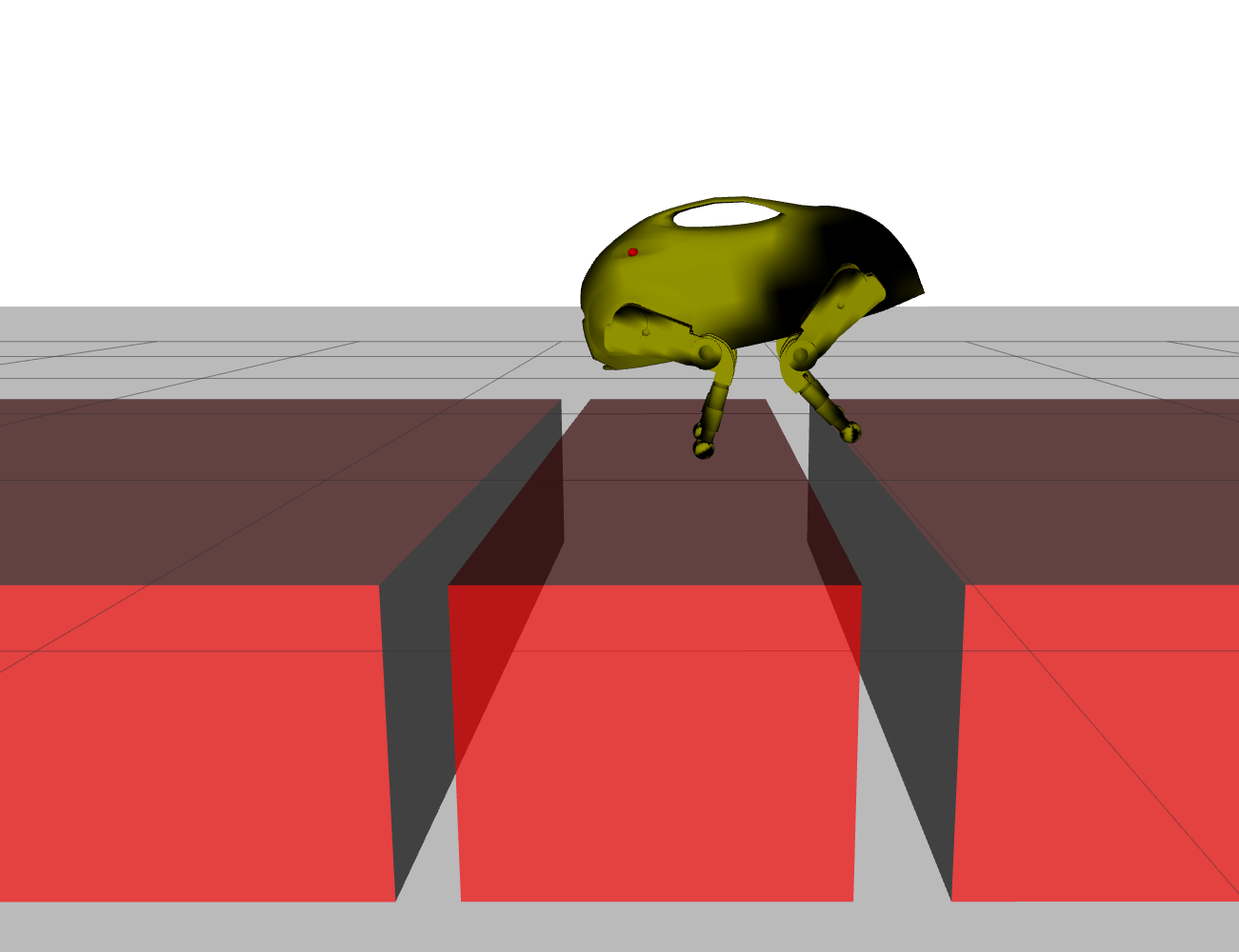}}\\[-1em]
		\subfloat[][$t=1.84$ s]{\includegraphics[trim =0mm 0mm 0mm 0mm,width=0.23\textwidth]{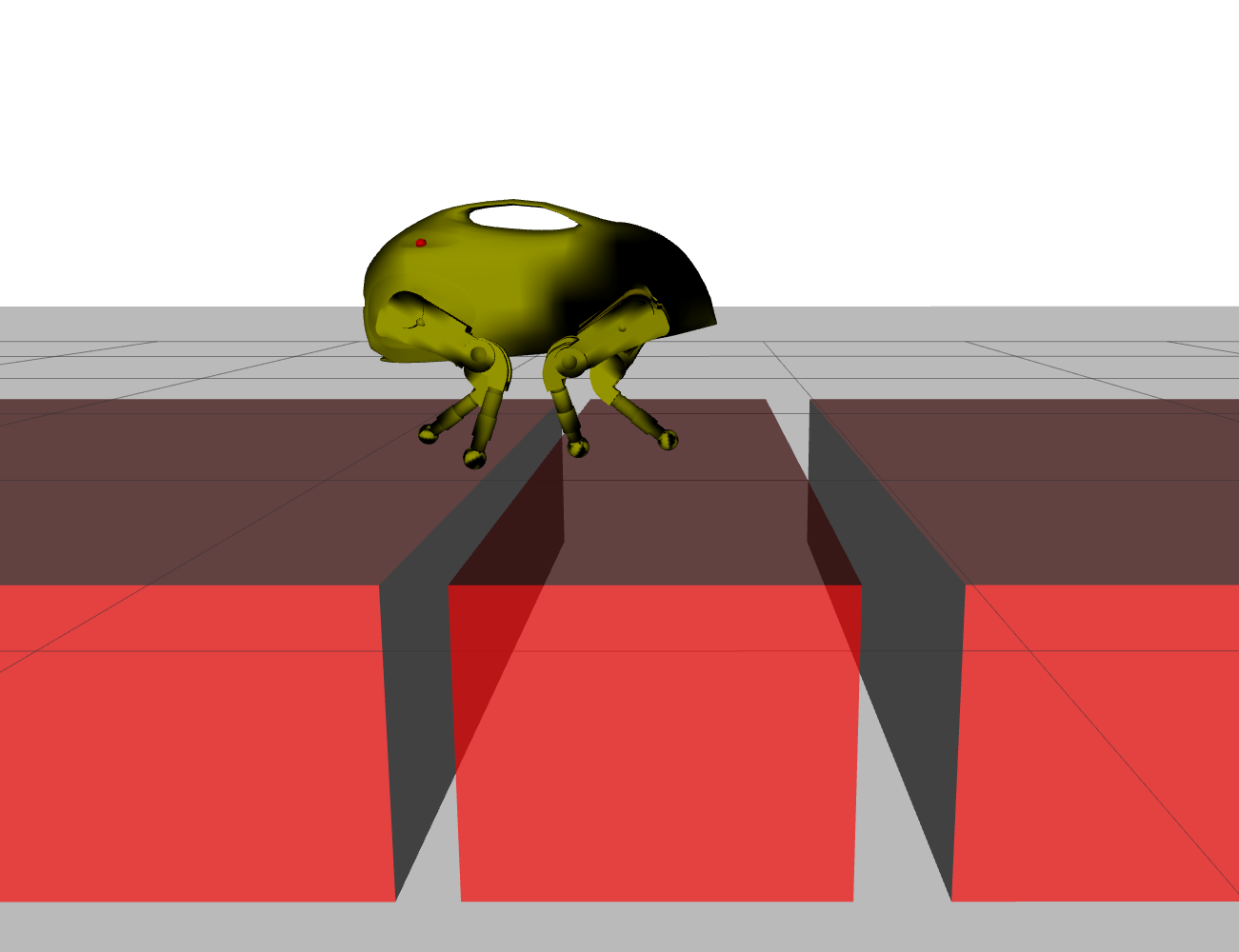}} &
		\subfloat[][$t=2.16$ s]{\includegraphics[trim =0mm 0mm 0mm 0mm,width=0.23\textwidth]{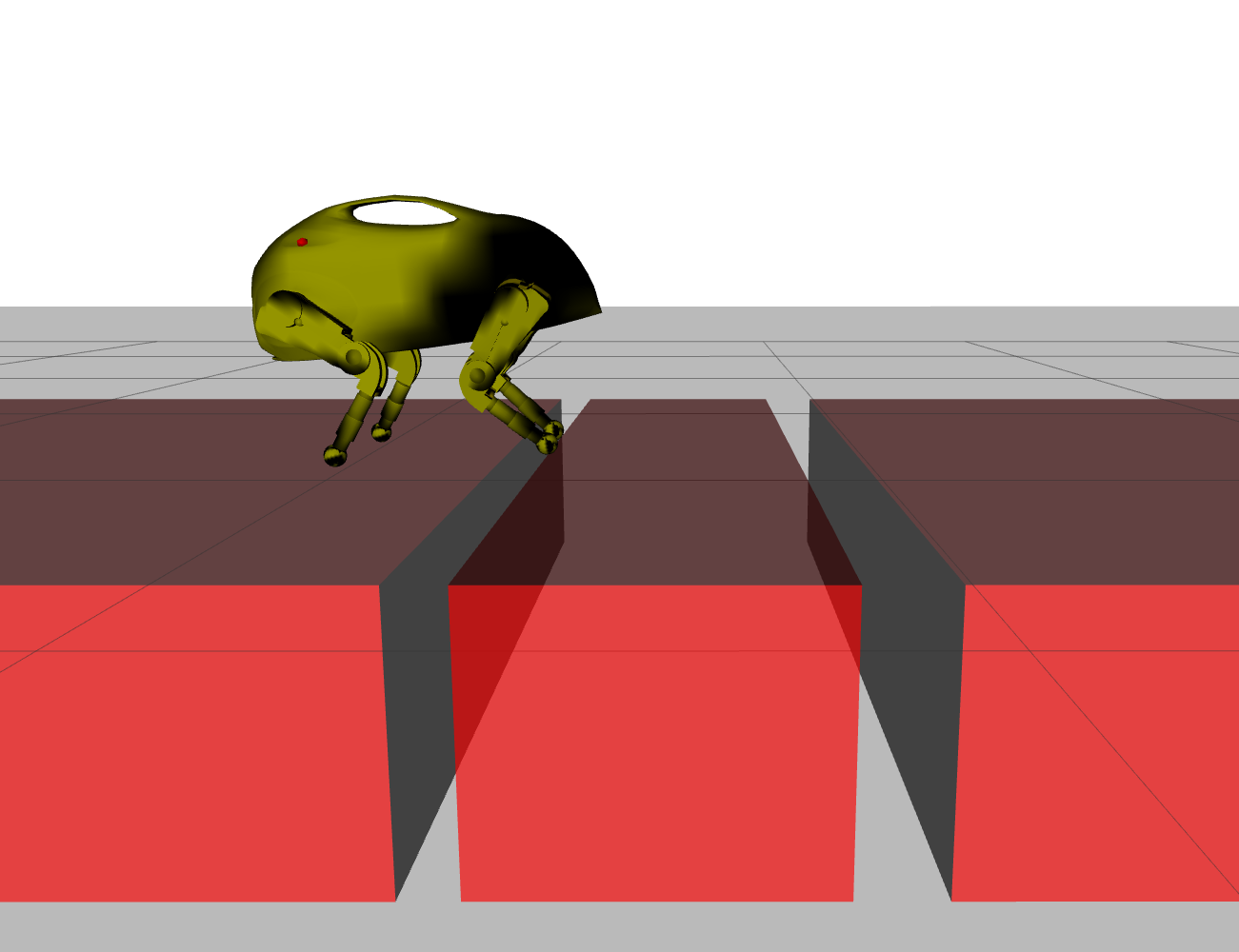}} &
		\subfloat[][$t=2.72$ s]{\includegraphics[trim =0mm 0mm 0mm 0mm,width=0.23\textwidth]{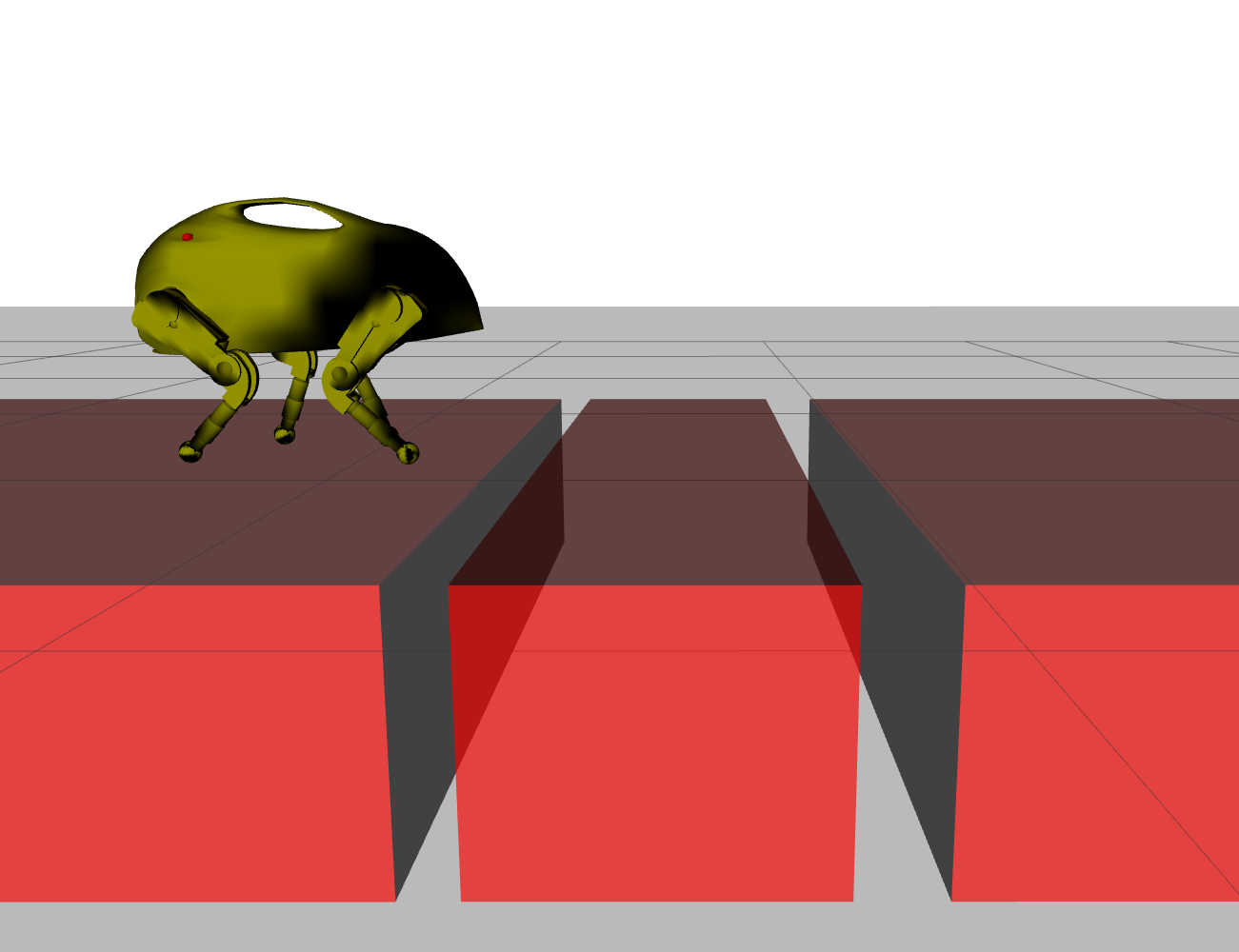}} &
		\subfloat[][$t=3.2$ s]{\includegraphics[trim =0mm 0mm 0mm 0mm,width=0.23\textwidth]{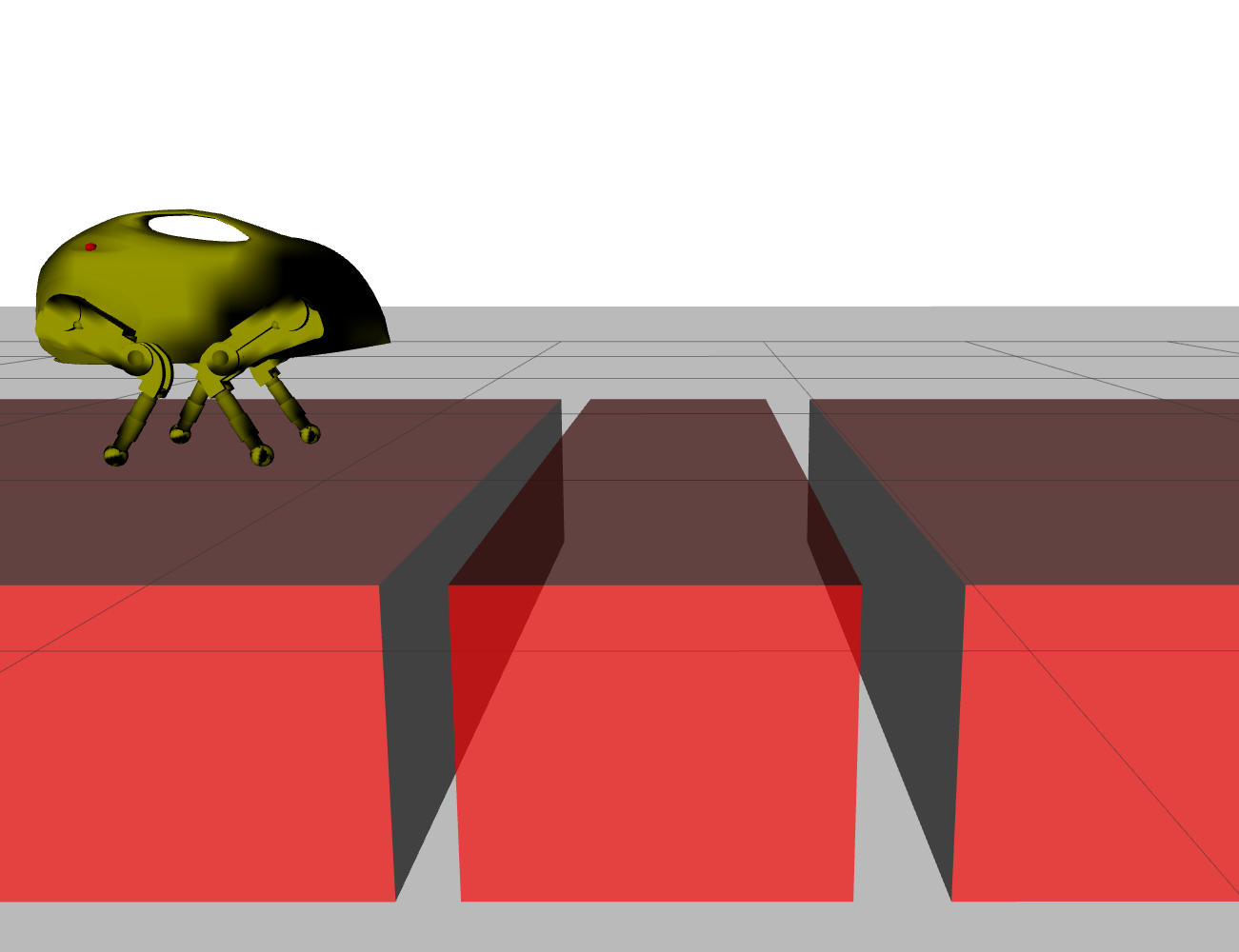}}
	\end{tabular}
	\caption{The LittleDog robot walks over terrain with gaps.}
	\label{fig::dog}
	\vspace{0.5em}
	\begin{tabular}{cccc}
		\hspace{1em}\subfloat[][$t=0$s]{\includegraphics[trim =0mm 0mm 0mm 0mm,width=0.23\textwidth]{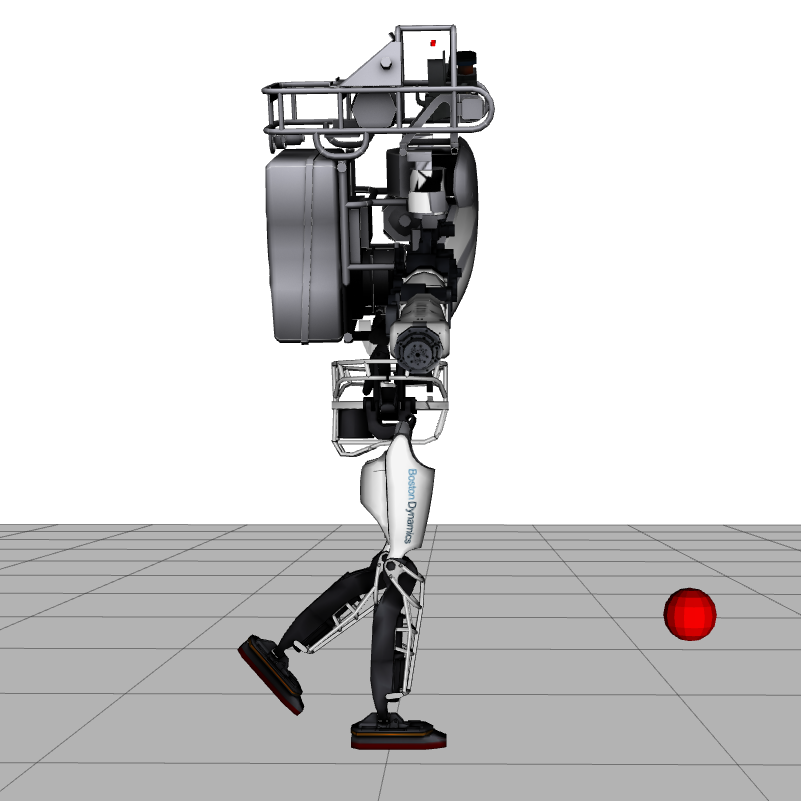}} &
		\subfloat[][$t=0.4$ s]{\includegraphics[trim =0mm 0mm 0mm 0mm,width=0.23\textwidth]{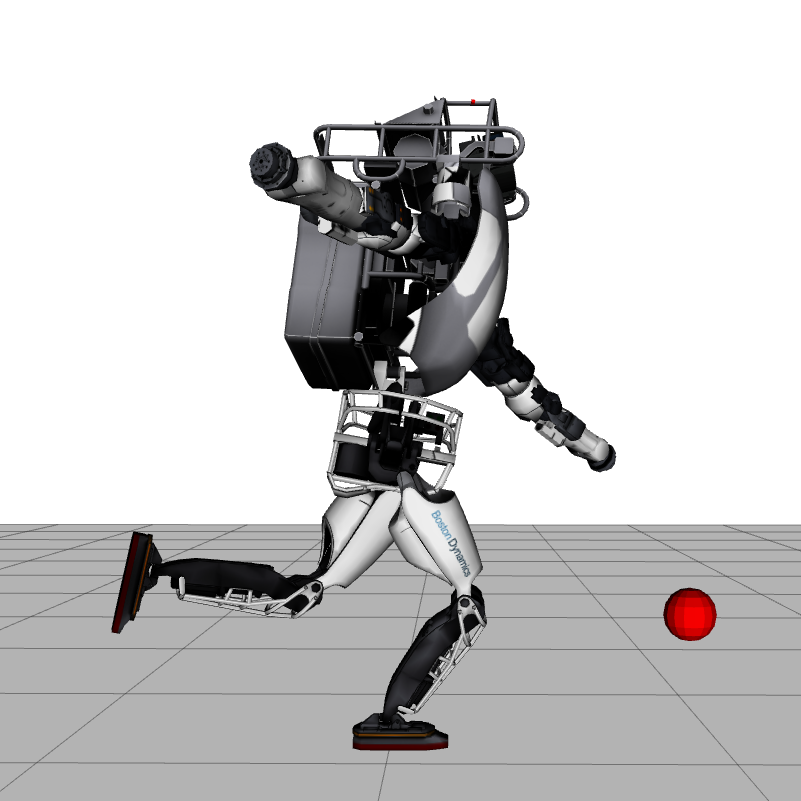}}&
		\subfloat[][$t=0.6$ s]{\includegraphics[trim =0mm 0mm 0mm 0mm,width=0.23\textwidth]{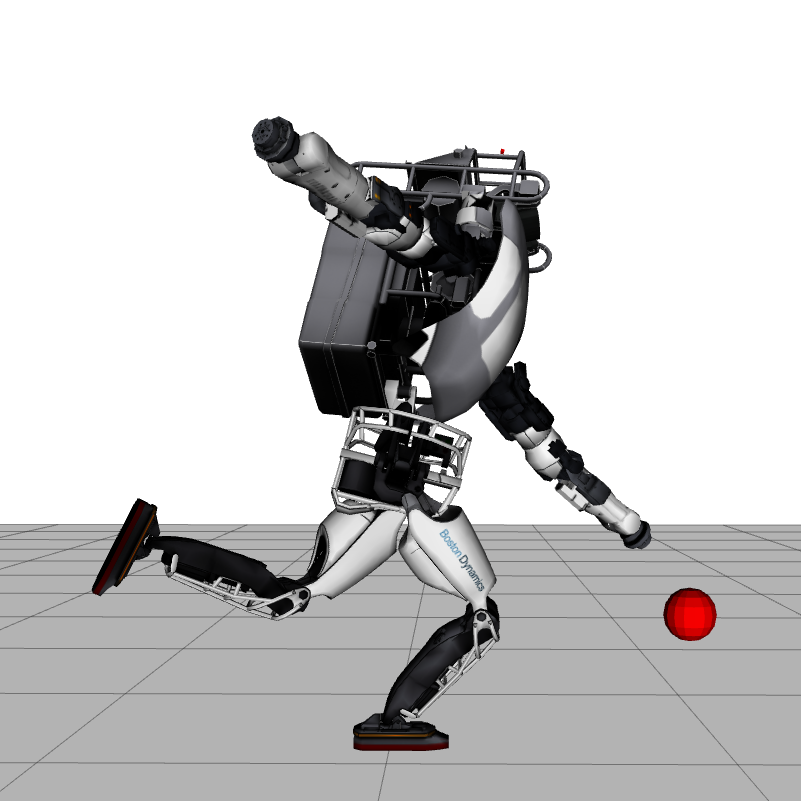}}& 
		\subfloat[][$t=1.3$ s]{\includegraphics[trim =0mm 0mm 0mm 0mm,width=0.23\textwidth]{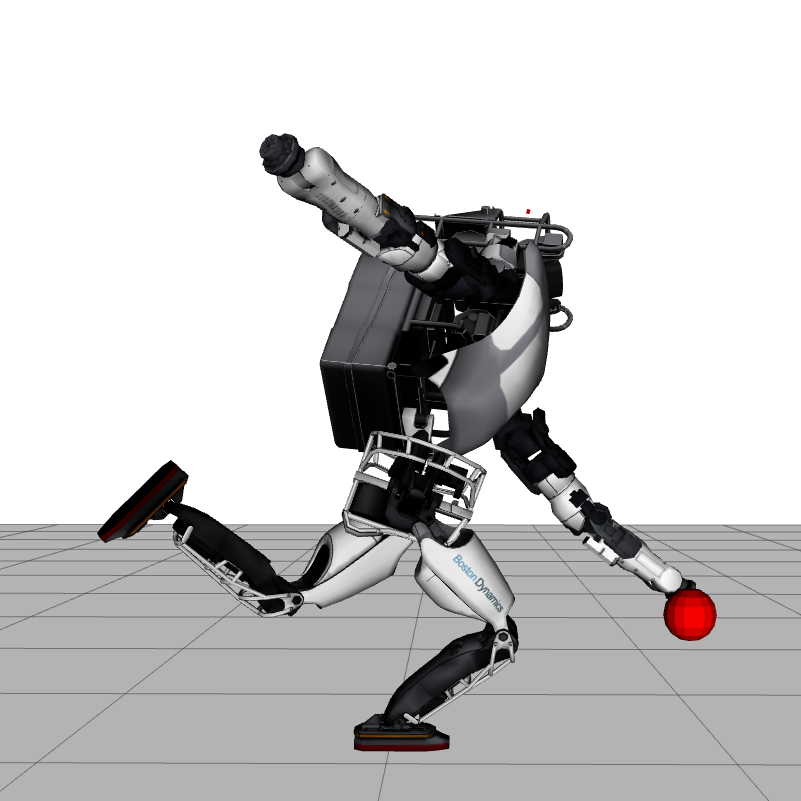}}		
	\end{tabular}
	\caption{The Atlas robot picks a red ball while keeping balanced with a single foot.}
	\label{fig::atlas} 
	\vspace{-1em}
\end{figure}

\section{Conclusion}\label{section::conclusion}
{In this paper, we present $O(n)$ algorithms for the linear-time higher-order variational integrators and $O(n^2)$ algorithms to linearize the DEL equations for use in trajectory optimization. The proposed algorithms are validated through comparison with existing methods and implementation on robotic systems for trajectory optimization. The results illustrate that the same integrator can be used for simulation and trajectory optimization in robotics, preserving mechanical properties while achieving good scalability and accuracy. Furthermore, thought not presented in this paper, these $O(n)$ algorithms can be regularized for parallel computation, which results in $O(\log(n))$ algorithms with enough processors.}

\bibliographystyle{unsrt}
\bibliography{mybib}

\clearpage

\setcounter{page}{1}

\title{Efficient Computation of Higher-Order Variational Integrators in Robotic Simulation and Trajectory Optimization: Appendix}
\author{Taosha Fan \quad  Jarvis Shultz \quad Todd Murphey}

\institute{Department of Mechanical Engineering, Northwestern University,\\ 2145 Sheridan Road, Evanston, IL 60208, USA\\
\email{taosha.fan@u.northwestern.edu, jschultz@northwestern.edu, t-murphey@northwestern.edu}}
\maketitle
\begin{abstract}
	This appendix provides the complete $O(n)$ algorithms to compute the Newton direction for higher-order variational integrators and the proofs of the propositions in the paper ``Efficient Computation of Higher-Order Variational Integrators in Robotic Simulation and Trajectory Optimization'' [1], published in the 13th International Workshop on the Algorithmic Foundations of Robotics (WAFR'18). It is assumed that the reader has read the original paper and knows the problem statements and the notation used. The numbering of the equations, algorithms, propositions, etc., is consistent with the numbering used in the original paper. 
\end{abstract}

\renewcommand{\thesection}{\Alph{section}}
\numberwithin{equation}{section}
\numberwithin{algorithm}{section} 
\numberwithin{myalg}{section} 

\section{Introduction}
 In the paper ``Efficient Computation of Higher-Order Variational Integrators in Rob-otic Simulation and Trajectory Optimization'' [1], we present $O(n)$ algorithms to evaluate the discrete Euler-Lagrange (DEL) equations and compute the Newton direction for solving the DEL equations, and $O(n^2)$ algorithms to linearize the DEL equations. As an appendix to  [1], this document provides the complete $O(n)$ algorithms to compute the Newton direction for higher-order variational integrators and the proofs of the propositions in  [1], which are not covered in the original paper due to space limitations.
 
 In this appendix, we begin with the complete $O(n)$ algorithms to compute the Newton direction in \cref{sectiona::newton}. In \cref{sectiona::preliminary}, we give an overview of preliminaries used in the algorithms and proofs. \cref{prop::dabi,prop::dkdq,prop::dvkdq,prop::eval} in [1, \cref{section::lin_vi,section::quad_lin}] to compute the higher-order variational integrators are proved in \cref{section::proof}.\par
 
 For implementation only, the reader only needs to read \cref{algorithm::dabi,algorithm::dabi_b} in \cref{sectiona::newton} as well as \cref{algorithm::drha1,algorithm::dkdv,algorithm::dvdq} in [1, \cref{section::lin_vi,section::quad_lin}]. \cref{sectiona::preliminary,section::proof} are not required to read as they present the proofs of the propositions in  [1] that do not necessarily aid in implementation.
 
 Even though most of the important content in  [1] is reiterated, we still advise the reader to read the original paper to know the problem statements and the notation used. Moreover, as mentioned in the abstract, the numbering of the equations, algorithms, propositions, etc., is consistent with the numbering used in  [1]. Therefore, the original paper will not be explicitly cited in the rest of this appendix when we make references to anything in it. 

\section{The $O(n)$ Algorithms to Compute the Newton Direction}\label{sectiona::newton}

In this section, we present \cref{algorithm::dabi,algorithm::dabi_b} to compute the Newton direction for higher-order variational integrators. The algorithms are self-contained and we refer the reader to \cref{subsectiona::diff} for differentiation on Lie groups that is used to compute $\tD_1\qka{\lF_i}$ in \cref{eq::gammax2} of \cref{algorithm::dabi_b}. The correctness and the $O(n)$ complexity of \cref{algorithm::dabi,algorithm::dabi_b} are proved in \cref{subsection::pp2}, however, this is not required to read for implementation. We remind the reader that $\delta\qkg{q_i}$ is the Newton direction for $\qkg{q_i}$, and $\qkvp{r_i}$ is the residue of the DEL equations \cref{eq::eval1,eq::eval2}. Moreover, from \cref{prop::dabi}, \cref{algorithm::dabi,algorithm::dabi_b} assume that the inverse of the Jacobian $\JJ^{-1}(\qk{\lq})$ exists, and $\qka{\lF_i}$ and $\qka{Q_i}$ can be respectively formulated as $\qka{\lF_i}=\qka{\lF_i}(\qka{g_i},\qka{\lv_i},\qka{u}) $ and $\qka{Q_i}=\qka{Q_i}(\qka{q_i},\qka{\dot{q}_i},\qka{u})$.

There are a number of quantities, such as $\qkap{D_i}$, $\qkag{\Phi_i}$, $\qka{\zeta_i}$, $\qkg{H_i}$, etc., which are recursively introduced in \cref{algorithm::dabi_b} to compute the Newton direction. Since there is no influence on the implementation of the algorithms as long as these quantities are correctly computed, we leave the explanation of their meaning to \cref{subsection::pp2}. Similarly, the detailed explanation of $\qkn{\leta_i}$ and $\ld\qkp{\lv_i}$ in \cref{algorithm::dabi} is left to \cref{subsectiona::tree,subsection::spatial}, respectively. For purposes of implementation, the reader only needs to know that these quantities are recursively computed through \cref{algorithm::dabi_b,algorithm::dabi}.
\vspace{0.5em}
\begin{myalg}[Recursive Computation of the Newton Direction]
	\label{algorithm::dabi}
	\begin{algorithmic}[1]
		\State initialize $\qka{g_0}=\I$ and $\qka{\lv_0}=0$
		\For{$i=1\rightarrow n$}
		\vspace{0.2em}
		\For{$\alpha=0\rightarrow s$}
		\vspace{0.2em}
		\State $g^{k,\alpha}_{i}=g_{\pa(i)}^{k,\alpha}g^{k,\alpha}_{\pa(i),i}(q_i^{k,\alpha})$
		\vspace{0.2em}
		\State $\lS_i^{k,\alpha} = \Ad_{g_{i}^{k,\alpha}} S_i$,\quad $\lM_i^{k,\alpha} = \Ad_{g_{i}^{k,\alpha}}^{-T} M_i { \Ad^{-1}_{g_{i}^{k,\alpha}}}$
		\State $\dot{q}_i^{k,\alpha}=\frac{1}{\Delta t}\sum\limits_{\beta=0}^s\qab{b}{\qikb}$,\quad $\lv^{k,\alpha}_{i}= \lv^{k,\alpha}_{\pa(i)}+ {{\lS_i^{k,\alpha}}}\cdot \dot{q}_i^{k,\alpha}$
		\State $\qka{\dot{\lS}\overline{\vphantom{S}}_i}=\ad_{\qka{\lv_i}}\qka{\lS_i}$
		\EndFor
		\EndFor
		\For{$i=n\rightarrow 1$}
		\State use \cref{algorithm::dabi_b} to evaluate
		\vspace{0.3em}
		\begin{enumerate}[label=\alph*),leftmargin=4em]
			\item $\qkap{D_i}$, $\qkan{G_i}$, $\qka{l_i}$ and $\qka{\lmu_i}$\vspace{0.2em}
			\item $\qkap{\Pi_i}$, $\qkan{\Psi_i}$, $\qka{\zeta_i}$ and $\qka{\lG_i}$
			\item $\qka{H_i}$ and $\qka{\Phi_i}$
			\item $\qkap{X_i}$, $\qkan{Y_i}$ and $\qka{y_i}$
		\end{enumerate}
		\EndFor
		\vspace{0.2em}
		\State initialize $\qkn{\le_0}=0$ and $\ld\qkp{\lv_0}=0$
		\vspace{0.2em}
		\For{$i=1\rightarrow n$}
		\algrule
		\newpage
		\algrule
		\For{$\gamma=1\rightarrow s$}
		\State 
		$\delta\qkg{q_i} = \sum\limits_{\rho=0}^{s}\qkgp{X_i}\cdot\ld\qkp{\lv_{\pa(i)}} +
		\sum\limits_{\nu= 1}^s \qkgn{Y_i}\cdot\qkn{\le_{\pa(i)}}+\qkg{y_i}$
		\vspace{-0.25em}
		\EndFor
		\vspace{0.10em}
		\For{$\nu=1\rightarrow s$}
		\State $\qkn{\le_i}= \qkn{\le_{\pa(i)}} + \qkn{\lS_i}\cdot\delta\qkn{q_i}$
		\vspace{0.2em}
		\EndFor
		\For{$\rho=0\rightarrow s$}
		\vspace{0.1em}
		\State $\delta\qkp{\dot{q}_i}=\frac{1}{\Delta t}\sum\limits_{\gamma=1}^{s}\qpg{b}\cdot\delta\qkg{q_i}$
		\vspace{0.1em}
		\State $\ld\qkp{\lv_i}= \ld\qkp{\lv_{\pa(i)}} + \qkp{\dot{\lS}\overline{\vphantom{S}}_i}\cdot\delta \qkp{q_i}+\qkp{\lS_i}\cdot\delta\qkp{\dot{q}_i} $
		\EndFor
		\EndFor
		\vspace{-0.5em}
	\end{algorithmic}
\end{myalg}
\vspace{0.25em}
\begin{myalg}[Recursive Computation of the Newton Direction -- Backward Pass]
	\label{algorithm::dabi_b}
	\allowdisplaybreaks
	\begin{algorithmic}[1]
		\State $\forall \alpha=0,\,1,\,\cdots,\,s$, $\forall \rho=0,\,1,\,\cdots,\,s$ and $\forall \nu=0,\,1,\,\cdots,\, s-1$,
		\begin{subequations}\label{eq::mux}
			\begin{align}
			\nonumber&\qkap{D_i}= \qap{\sigma} \qka{\lM_i} + \sum_{j\in\chd(i)}\Big(\qkap{D_j}+\sum_{\gamma=1}^{s}\qkag{H_j}\qkgp{X_j}-\\
			&\hspace{19em}\overline{\sigma}^{\alpha 0}\ad_{\qka{\lmu_j}}^D \qka{\lS_j}\qkap{X_j}\Big),\\
			&\qkan{G_i}=\sum_{j\in\chd(i)} \Big(\qkan{G_j}+\sum_{\gamma=1}^{s}\qkag{H_j}\qkgn{Y_j} - \overline{\sigma}^{\alpha 0}\ad_{\qka{\lmu_j}}^D \qka{\lS_j}\qkan{Y_j}\Big),\\
			&\qka{l_i}=\sum_{j\in\chd(i)}\Big(\qka{l_j}+\sum_{\gamma=1}^s\qkag{H_j}\qkg{y_j}- \overline{\sigma}^{\alpha 0}\ad_{\qka{\lmu_j}}^D \qka{\lS_j}\qka{y_j}\Big),\\[0.3em]
			&\nonumber \qka{\lmu_i}=\qka{\lM_i}\qka{\lv_i}+\sum_{j\in \chd(i)}\qka{\mu_j}
			\end{align}
		\end{subequations}
		in which
		\begin{equation}\label{eq::sigma}
		\qap{\sigma}=\begin{cases}
		1 & \alpha=\rho,\\
		0 & \alpha\neq\rho
		\end{cases}\quad\quad\text{and}\quad\quad
		\overline{\sigma}^{\alpha 0}=\begin{cases}
		1 & \alpha\neq 0,\\
		0 & \alpha = 0
		\end{cases}
		\end{equation}
		\vspace{0.2em}

		\State $\forall \alpha=0,\,1,\,\cdots,\,s-1$, $\forall \rho=0,\,1,\,\cdots,\,s$ and $\forall \nu=0,\,1,\,\cdots,\, s-1$,
		\vspace{-0.2em}
		\begin{subequations}\label{eq::gammax}
			\begin{align}
			&\nonumber\qkap{\Pi_i} = \qap{\sigma} \tD_2 \qka{\lF_i} + \sum_{j\in\chd(i)}\Big(\qkap{\Pi_j}+\sum_{\gamma=1}^s\qkag{\Phi_j}\qkgp{X_j}-\\
			&\hspace{19em}\overline{\sigma}^{\alpha 0}\ad_{\qka{\lG_j}}^D \qka{\lS_j}\qkap{X_j}\Big),
			\end{align}
		\end{subequations}
	\algrule
	\newpage
	\algrule
	\vspace{-0.5em}
		\begin{subequations}[resume]
			\begin{align}
			&\nonumber\qkan{\Psi_i} = \qan{\sigma} \Big(\tD_1\qka{\lF_i}+\ad_{\qka{\lF_i}}^D-\tD_2\qka{\lF_i}\ad_{\qka{\lv_i}}\Big) +\\
			&\label{eq::gammax2}\hspace{4em}\sum_{j\in\chd(i)}\Big(\qkan{{\Psi}_j}+\sum_{\gamma=1}^s\qkag{\Phi_j}\qkgn{Y_j}-\overline{\sigma}^{\alpha 0}\ad_{\qka{\lG_j}}^D\qka{\lS_j}\qkan{Y_j}\Big),\\[0.5em]
			&\qka{\zeta_i}=\sum_{j\in\chd(i)}\Big(\qka{\zeta_j}+\sum_{\gamma=1}^s\qkag{\Phi_j}\qkg{y_j}-
			\overline{\sigma}^{\alpha 0}\ad_{\qka{\lG_j}}^D \qka{\lS_j}\qka{y_j}\Big),\\[0.3em]
			&\nonumber\qka{\lG_i}= \qka{\lF_i}+\sum_{j\in \chd(i)}\qka{\lG_j}
			\end{align}
		\end{subequations}
	\vspace{0.2em}
		\State $\forall\alpha=0,\,1,\,\cdots,\,s$ and $\forall\gamma=1,\,2,\,\cdots,\, s$,
		\begin{equation}
		\vspace{-1em}
		\nonumber
		\qkag{H_i}= \qkag{D_i}\qkg{\dot{\lS}\overline{\vphantom{S}}_i} +\qkag{G_i}\qkg{\lS_i} +\dfrac{1}{\Delta t}\sum\limits_{\rho=0}^s\qpg{b} \qkap{D_i}\qkp{\lS_i}.
		\end{equation}
		\vspace{0.2em}
		\State $\forall\alpha=0,\,1,\,\cdots,\,s-1$ and $\forall\gamma=1,\,2,\,\cdots,\, s$,
		\begin{equation}
		\nonumber
		\qkag{\Phi_i}= \qkag{\Pi_i}\qkg{\dot{\lS}\overline{\vphantom{S}}_i} +\qkag{\Psi_i}\qkg{\lS_i} +\dfrac{1}{\Delta t}\sum\limits_{\rho=0}^s\qpg{b} \qkap{\Pi_i}\qkp{\lS_i}.
		\end{equation}

		\State $\forall \alpha=0,\,1,\,\cdots,\,s-1$, $\forall \rho=0,\,1,\,\cdots,\,s$ and $\forall \nu=0,\,1,\,\cdots,\, s-1$,  
		\begin{align*}
		\qkap{\Theta_i}=\;&\qa{w}\Delta t\cdot\big({\qka{\dot{\lS}\overline{\vphantom{S}}_i}{}}^T \qkap{D_i}+ \qap{\sigma}\qka{\lS_i}^T\ad_{\qka{\lmu_i}}^D\big)+\qka{\lS_i}^T\qkap{\Pi_i},\\
		\qkan{\Xi_i}=\;&\qa{w}\Delta t\cdot{\qka{\dot{\lS}\overline{\vphantom{S}}_i}{}}^T \qkan{G_i} + \qka{\lS_i}^T\qkan{\Psi_i}.
		\end{align*}
		\State $\forall \alpha=0,\,1,\,\cdots,\,s-1$, $\forall \rho=0,\,1,\,\cdots,\,s$ and $\forall \nu=0,\,1,\,\cdots,\, s-1$,
		\begin{align*}
		&\qkap{\overline{\Theta}_i}=\qkap{\Theta_i} + \sum_{\beta=0}^s \qab{a}\qkb{\lS_i}^T \qkbp{D_i},\\
		&\qkan{\overline{\Xi}_i}=\qkan{\Xi_i} +  \sum_{\beta=0}^s \qab{a}\qkb{\lS_i}^T \qkbn{G_i},\\
		&\qka{\overline{\xi}_i} = \qa{w}\Delta t\cdot{\qka{\dot{\lS}\overline{\vphantom{S}}_i}{}}^T  \qka{l_i} + \qka{\lS_i}^T\qka{\zeta_i} +\sum_{\beta=0}^{s}\qab{a}\qkb{\lS_i}^T \qkb{l_i}.
		\end{align*}
		\vspace{-0.5em}
		\State 	$\forall \alpha=0,\,1,\,\cdots,\,s-1$ and $\forall\gamma=1,\,2,\,\cdots,\, s$,
		\begin{equation*}
		\begin{aligned}
		\qkag{\Lambda_{i}}= &\qa{w}\Delta t\cdot\qka{\dot{\lS}\overline{\vphantom{S}}_i}^T \qkag{H_i}+\qka{\lS_i}^T\qkag{\Phi_i} +\sum_{\beta=0}^s \qab{a} \qkb{\lS_i}^T\qkbg{H_i}+\\
		&{\sigma}^{\alpha\gamma}\left(\tD_1\qka{Q_i}+\qa{w}\Delta t\cdot \qka{\lS_i}^T\ad_{\qka{\lmu_i}}^D\qka{\dot{\lS}\overline{\vphantom{S}}_i}\right) +\frac{1}{\Delta t} \qag{b}\cdot\tD_2\qka{Q_i}
		\end{aligned}
		\end{equation*}
		with which $\qk{\Lambda_{i}}=\begin{bmatrix}
		\qkag{\Lambda_{i}}
		\end{bmatrix}\in\R^{s\times s}$
		\algrule
		\newpage
		\algrule
		\State $\forall\gamma=1,\,2,\,\cdots,\, s$ and $\forall\varrho=0,\,1,\,\cdots,\, s-1$, compute $\qkgvp{\lLam_{i}}$ such that $\qk{\Lambda_{i}}^{-1}=\begin{bmatrix}
		\qkgvp{\lLam_{i}}
		\end{bmatrix}\in\R^{s\times s}$  
		\vspace{0.65em}
		\State $\forall \gamma=1,\,2,\,\cdots,\,s$, $\forall \rho=0,\,1,\,\cdots,\,s$ and $\forall \nu=1,\,2,\,\cdots,\,s$ 
		\begin{align*}
		\qkgp{X_i}&=-\sum_{\varrho=0}^{s-1}\qkgvp{\lLam_{i}}\cdot\qkvpp{\overline{\Theta}_i},\\
		\qkgn{Y_i}&=-\sum_{\varrho=0}^{s-1}\qkgvp{\lLam_{i}}\cdot\qkvpn{\overline{\Xi}_i},\\
		\qkg{y_i}&=-\sum_{\varrho=0}^{s-1}\qkgvp{\lLam_{i}}\left(\qkvp{r_i}+\qkvp{\overline{\xi}_i}\right)
		\end{align*}
	\end{algorithmic}
\end{myalg}

\section{Preliminaries}\label{sectiona::preliminary}
In this section, we present additional preliminaries used in \cref{algorithm::dabi,algorithm::dabi_b} and the proofs of \cref{prop::dabi,prop::dkdq,prop::dvkdq,prop::eval}. In \cref{subsectiona::tree}, we extend the contents of \cref{subsection::tree} for the computation of variations and derivatives. In \cref{subsection::spatial,subsectiona::diff}, we respectively introduce the notion of the spatial variation for spatial quantities and the differentiation on Lie groups, which are mainly used in \cref{algorithm::dabi,algorithm::dabi_b} and the proof of \cref{prop::dabi}.

\subsection{The Tree Representation Revisited}\label{subsectiona::tree}

In addition to the computation of rigid body dynamics as those in \cref{subsection::tree}, the tree representation can also be used to compute the variations and derivatives.

As is known, in the tree representation, the configuration $g_i\in SE(3)$ of rigid body $i$ is
\begin{equation}\label{eqa::g}
g_i=g_{\pa(i)}g_{\pa(i),i}(q_i)
\end{equation}
in which $g_{\pa(i),i}(q_i)=g_{\pa(i),i}(0)\exp(\hat{S}_i q_i )$ and $S_i$ is the body Jacobian of joint $i$ with respect to frame $\{i\}$. In addition, the spatial Jacobian of joint $i$ with respect to frame $\{0\}$ is
\begin{equation}\label{eqa::S}
	\lS_i = \Ad_{g_i}S_i
\end{equation}
in which $S_i$ is constant by definition. Using \cref{eqa::g,eqa::S} as well as $\Ad_{g_i}S_i = \big(g_i \hat{S}_i g_i^{-1}\big)^\vee$, we obtain $\leta_i = (\delta g_i g_i^{-1})^\vee$ as
\begin{equation}\label{eqa::eta}
	\leta_i=\leta_{\pa(i)}+\lS_i\cdot\delta q_i,
\end{equation}
or equivalently, 
\begin{equation}\label{eqa::etas}
\leta_i = \lS_i\cdot\delta q_i + \sum_{j\in\anc(i)}^n \lS_j\cdot\delta q_j
\end{equation}
and furthermore,
\begin{subequations}\label{eqa::dgdq}
\begin{equation}\label{eqa::dgdq1}
\left(\frac{\partial g_i}{\partial q_j} g_i^{-1}\right)^\vee = \begin{cases}
\lS_j & j\in\anc(i) \cup\{i\},\\
0 & \text{otherwise},
\end{cases}
\end{equation}
\begin{equation}\label{eqa::dgdq2}
\left(\frac{\partial g_j}{\partial q_i} g_i^{-1}\right)^\vee = \begin{cases}
\lS_i & j\in\des(i) \cup\{i\},\\
0 & \text{otherwise}.
\end{cases}
\end{equation}
\end{subequations}
In addition, from \cref{eqa::S,eqa::eta}, $\delta \Ad_{g_i} = \ad_{\leta_i} \Ad_{g_i}$ and $\ad_{\lS_i} \lS_i =0$, we obtain
\begin{equation}\label{eqa::dS}
	\delta \lS_i = \ad_{\leta_i}\lS_i=-\ad_{\lS_i}\le_i=\ad_{\le_{\pa(i)}}\lS_i=-\ad_{\lS_i}\le_{\pa(i)}.
\end{equation}

Moreover, as a result of \cref{eqa::dS,eqa::etas,eqa::dgdq}, we further obtain
\begin{subequations}\label{eqa::dSdq}
	\begin{equation}\label{eqa::dSdq1}
	\frac{\partial \lS_i}{\partial q_j} = \begin{cases}
	\ad_{\lS_j}\lS_i & j\in\anc(i),\\
	0 & \text{otherwise},
	\end{cases}
	\end{equation}
	\begin{equation}\label{eqa::dSdq2}
	\frac{\partial \lS_j}{\partial q_i} = \begin{cases}
	\ad_{\lS_i}\lS_j & j\in\des(i),\\
	0 & \text{otherwise}.
	\end{cases}
	\end{equation}
\end{subequations}
Since the spatial velocity $\lv_i$ of rigid body $i$ is
\begin{equation}\label{eqa::v}
\begin{aligned}
	\lv_i &=\lS_i\cdot\dot{q}_i+\sum_{j\in\anc(i)}{\lS_j}\cdot\dot{q}_j\\
		  &= \lv_{\pa(i)} + \lS_i\cdot\dot{q}_i,
\end{aligned}
\end{equation}  
we obtain
$$
\begin{aligned}
\delta \lv_i &=\delta \lS_i\cdot \dot{q}_i + \lS_i \cdot\delta \dot{q}_i+\sum_{j\in\anc(i)}\left(\delta \lS_j\cdot \dot{q}_j + \lS_j \cdot\delta \dot{q}_j\right)\\
&=\delta \lv_{\pa(i)} + \delta \lS_i\cdot \dot{q}_i + \lS_i \cdot\delta \dot{q}_i.
\end{aligned}
$$
Substitute \cref{eqa::dS} into the equation above, the result is
\begin{equation}\label{eqa::dv}
\begin{aligned}
\delta \lv_i &=\ad_{\leta_i} \lS_i\cdot \dot{q}_i + \lS_i \cdot\delta \dot{q}_i+\sum_{j\in\anc(i)}\left(\ad_{\leta_j} \lS_j\cdot \dot{q}_j + \lS_j \cdot\delta \dot{q}_j\right) \\
&=\delta \lv_{\pa(i)} + \ad_{\leta_i} \lS_i\cdot \dot{q}_i + \lS_i \cdot\delta \dot{q}_i.
\end{aligned}
\end{equation}
From \cref{eqa::dS,eqa::dSdq,eqa::dv,eqa::v}, we obtain
\begin{subequations}\label{eqa::dvdqdot}
	\begin{equation}\label{eqa::dvdqdot1}
	\frac{\partial \lv_i}{\partial \dot{q}_j}=\begin{cases}
	S_j & j\in \anc(i)\cup \{i\},\\
	0 & \text{otherwise},
	\end{cases}
	\end{equation}
	\begin{equation}\label{eqa::dvdqdot2}
	\frac{\partial \lv_j}{\partial \dot{q}_i}=\begin{cases}
	S_i & j\in \des(i)\cup \{i\},\\
	0 & \text{otherwise},
	\end{cases}
	\end{equation}
\end{subequations}
and
\begin{subequations}\label{eqa::dvdq}
	\begin{equation}\label{eqa::dvdq1}
	\frac{\partial \lv_i}{\partial q_j}=\begin{cases}
	\ad_{\lS_j}(\lv_i-\lv_j)& j\in\anc(i)\cup\{i\},\\
	0 & \text{otherwise},
	\end{cases}
	\end{equation}
	\begin{equation}\label{eqa::dvdq2}
	\frac{\partial \lv_j}{\partial q_i}=\begin{cases}
	\ad_{\lS_i}(\lv_j-\lv_i)& j\in\des(i)\cup\{i\},\\
	0 & \text{otherwise}.
	\end{cases}
	\end{equation}
\end{subequations}
In addition, from \cref{eqa::v,eqa::S}, $\Ad_{\dot{g}_i} = \ad_{\lv_i} \Ad_{g_i}$ and $\ad_{\lS_i} \lS_i =0$, we obtain
\begin{equation}\label{eqa::Sdot}
\dot{\lS}_i = \ad_{\lv_i} \lS_i =-\ad_{\lS_i}\lv_i=\ad_{\lv_{\pa(i)}}\lS_i  = -\ad_{\lS_i}\lv_{\pa(i)}.
\end{equation} 
As for the spatial inertia matrix $\lM_i = \Ad_{g_i}^{-T} M_i\Ad_{g_{i}}^{-1}$, algebraic manipulation shows that
\begin{equation}\label{eqa::dM}
\delta \lM_i = -\ad_{\leta_i}^T\cdot \lM_i -\lM_i\cdot\ad_{\leta_i},
\end{equation}
and from Eqs. \eqref{eqa::eta} to \eqref{eqa::dgdq} and \cref{eqa::dM}, we obtain
\begin{subequations}\label{eqa::dMdq}
	\begin{equation}\label{eqa::dMdq1}
	\frac{\partial \lM_i}{\partial q_j} = \begin{cases}
	-\ad_{\lS_j}^T \lM_i-\lM_i\ad_{\lS_j} & j\in\anc(i)\cup\{i\},\\
	0 & \text{otherwise},
	\end{cases}
	\end{equation}
	\begin{equation}\label{eqa::dMdq2}
	\frac{\partial \lM_j}{\partial q_i} = \begin{cases}
	-\ad_{\lS_i}^T \lM_j-\lM_j\ad_{\lS_i} & j\in\des(i)\cup\{i\},\\
	0 & \text{otherwise}.
	\end{cases}
	\end{equation}
\end{subequations}

In \cref{subsection::prop1,subsection::pp2,subsection::pp3,subsection::pp4}, \cref{eqa::eta} to \eqref{eqa::dMdq} will be used to prove \cref{prop::dabi,prop::dvkdq,prop::dkdq,prop::eval}.

\subsection{The Spatial Variation}\label{subsection::spatial}
In this subsection, we introduce the \textit{spatial variation} $\ldh\overline{(\cdot)}$ that is used in \cref{algorithm::dabi,algorithm::dabi_b} and the proof of \cref{propa::dabi}. Note that the notion of the spatial variation $\ldh\overline{(\cdot)}$ only applies to the spatial quantities $\overline{(\cdot)}$ of $T_eSE(3)$ or $T_e^*SE(3)$ that are described in the spatial frame.

If $\la,a\in T_e SE(3) $ are related as $\la= \Ad_g a$ in which $g\in SE(3)$, we have
 \begin{equation}
\nonumber
	\delta \overline{a} = \Ad_g \delta a + \ad_{\le}\overline{a}
\end{equation}
in which $\le = (\delta g g^{-1})^\vee$. For numerical simplicity, it is sometimes preferable to have the variations of $\la$ and $a$ still related by $\Ad_g$. Therefore, we define the spatial variation $\ld\la$ to be

\begin{equation}\label{eqa::da2}
\ld\la=\delta \la - \ad_{\le}\la
\end{equation}
such that $\ld\la = \Ad_g \delta a$ as long as $\la = \Ad_ g a$. In a similar way, if $\lb^*,b^*\in T_e^*SE(3)$ are related as $\lb^*=\Ad_g^{-T}b^*$, we obtain
$$\delta \lb^* = \Ad_g^{-T} \delta b^* - \ad_{\le}^T\lb^*.$$
Similar to \cref{eqa::da2}, the spatial variation $\ldh\lb^*$ is defined to be

\begin{equation}\label{eqa::db2}
\ldh\lb^* = \delta \lb^* + \ad_{\leta}^T\lb^*
\end{equation} 
such that $\ldh\lb^* = \Ad_g^{-T}\delta b^*$ as long as $\lb^*=\Ad_g^{-T}b^*$. In addition, note that $\delta\big({b^*}^T a\big) = \delta {b^*}^T a+ {b^*}^T\delta a= \ldh {\lb^*}^T\la + {\lb^*}^T\ld\la$ and $\delta({\lb^*}^T\la)=\delta({b^*}^Ta)$, we have
\begin{equation}\label{eqa::lddot}
\delta ({\lb^*}^T \la)=\ldh {\lb^*}^T\la + {\lb^*}^T\ld\la.
\end{equation}

In general, the spatial variations $\ldh\overline{(\cdot)}$ are the infinitesimal changes of spatial quantities in either the Lie algebra $T_e SE(3)$ or the dual Lie algebra $T_e^*SE(3)$ after canceling out the influences of the frame change.

In \cref{section::lin_vi}, we have a number of spatial quantities that are defined in $T_eSE(3)$ and $T_e^*SE(3)$, whose spatial variations $\ldh\overline{(\cdot)}$ can be computed in the tree representation.

Following \cref{eqa::da2,eqa::dS,eqa::S}, for $\qka{\lS_i} = \Ad_{\qka{g_i}}S_i$, the spatial variation $\ldh\qka{\lS_i}$ is
\begin{equation}\label{eqa::ldS}
\ldh \qka{\lS_i} = 0
\end{equation}
though $\delta \qka{\lS_i} = \ad_{\qka{\le_i}}\qka{\lS_i}$ is usually    not zero. In addition, according to \cref{eqa::dv,eqa::da2}, we have 
$$
\ld\qka{\lv_i} = \delta\qka{\lv_{\pa(i)}} + \ad_{\qka{\le_i}}\qka{\lS_i}\cdot\qka{\dot{q}_i} +\qka{\lS_i}\cdot\delta \qka{\dot{q}_i}-\ad_{\qka{\le_i}}\qka{\lv_i}
$$
Substitute \cref{eqa::eta,eqa::v} into the equation above to expand $\ad_{\qka{\le_i}}\qka{\lv_i}$ and apply \cref{eqa::Sdot,eqa::dS}, it can be shown that
\begin{equation}\label{eqa::ldv}
\ld\qka{\lv_i} = \ld\qka{\lv_{\pa(i)}} + \qka{\dot{\lS}\overline{\vphantom{S}}_{i}}\cdot\delta \qka{q_i} + \qka{\lS_i}\cdot \delta \qka{\dot{q}}.
\end{equation}

In terms of $\qka{\lmu_i}$, $\qka{\lG_{i}}$ and $\qka{\lO_i}$ in \cref{eq::eval}, which are spatial quantities in $T_e^*SE(3)$, we can still implement the tree representation to compute the spatial variation. According to \cref{defa::abmom}, we have
$$\delta \qka{\lmu_i} = \delta(\qka{\lM_i}\qka{\lv_i}) + \sum_{j\in\chd(i)}\delta \qka{\lmu_j}. $$
From \cref{eqa::db2}, the spatial variation $\ld\qka{\lmu_i}$ is
$$\ld\qka{\lmu_i}=\delta(\qka{\lM_i}\qka{\lv_i})+\sum_{j\in\chd(i)}\delta \qka{\lmu_j} +\ad_{\qka{\le_i}}^T\qka{\lmu_i}. $$
Using $\qka{\lmu_i}=\qka{\lM_i}\qka{\lv_i}+\sum_{j\in \chd(i)}\qka{\lmu_j}$ and $\qka{\le_i}= \qka{\le_j}- \qka{\lS_j}\cdot\delta \qka{q_j}$, we have
\begin{multline}\label{eqa::dmu}
\ld\qka{\lmu_i}=\delta(\qka{\lM_i}\qka{\lv_i})+\ad_{\qka{\le_i}}^T(\qka{\lM_i}\qka{\lv_i})+\\
\sum_{j\in\chd(i)}\left(\delta\qka{\lmu_j}+\ad_{\qka{\le_j}}^T\qka{\lmu_j}-\ad_{\qka{\lS_j}}^T\qka{\lmu_j}\cdot\delta \qka{q_j}\right)
\end{multline}
As a result of \cref{eqa::dM,eqa::da2}, $\delta(\qka{\lM_i}\qka{\lv_i})+\ad_{\qka{\le_i}}^T(\qka{\lM_i}\qka{\lv_i})$ is
\begin{equation}\label{eqa::dmv}
\begin{aligned}
\delta(\qka{\lM_i}\qka{\lv_i})+\ad_{\qka{\le_i}}^T(\qka{\lM_i}\qka{\lv_i})&=\qka{\lM_i}(\delta \qka{\lv_i}-\ad_{\qka{\le_i}}\qka{\lv_i})\\
&=\qka{\lM_i}\ld\qka{\lv_i}.
\end{aligned}
\end{equation}

From \cref{eqa::db2,eqa::dmv} and $\ad_{\qka{\lS_j}}^T\qka{\lmu_j}=\ad_{\qka{\lmu_j}}^D\qka{\lS_j}$, \cref{eqa::dmu} is simplified to
\begin{equation}\label{eqa::ldmu}
\begin{aligned}
\ld\qka{\lmu_i}&= \qka{\lM_i}\ld\qka{\lv_i}+\sum_{j\in\chd(i)}\left(\ld\qka{\lmu_j}-\ad_{\qka{\lS_j}}^T\qka{\lmu_j}\cdot\delta \qka{q_j}\right)\\
&=\qka{\lM_i}\ld\qka{\lv_i}+\sum_{j\in\chd(i)}\big(\ld\qka{\lmu_j}-\ad_{\qka{\lmu_j}}^D\qka{\lS_j}\cdot\delta \qka{q_j}\big).
\end{aligned}
\end{equation}
In a similar way, for the spatial variation $\ldh\qka{\lG_i}$, we obtain
\begin{equation}\label{eqa::ldG}
\begin{aligned}
\ldh\qka{\lG_i}&= \ldh\qka{\lF_i}+\sum_{j\in\chd(i)}\left(\ldh\qka{\lG_j}-\ad_{\qka{\lS_j}}^T\qka{\lG_j}\cdot\delta \qka{q_j}\right)\\
&= \ldh\qka{\lF_i}+\sum_{j\in\chd(i)}\big(\ldh\qka{\lG_j}-\ad_{\qka{\lG_j}}^D\qka{\lS_j}\cdot\delta \qka{q_j}\big).
\end{aligned}
\end{equation}
As for $\qka{\lO_i}= w^{\alpha}\Delta t\cdot \ad_{\qka{\lv_{i}}}^T\cdot\qka{\lmu_i}+\qka{\lG_{i}}$, from \cref{eqa::da2,eqa::db2}, algebraic manipulation shows that
\begin{equation}\label{eqa::ldo}
\begin{aligned}
	\ldh\qka{\lO_i}&=\delta \qka{\lO_i}+ \ad_{\qka{\le_i}}^T\qka{\lO_i}\\
	&=w^\alpha\Delta t\cdot\big(\ad_{\qka{\lv_i}}^T\cdot\ld\qka{\lmu_i}+ \ad_{\ld\qka{\lv_i}}^T \qka{\lmu_i}\big)+\ldh\qka{\lG_i}\\
	&=w^\alpha\Delta t\cdot\big(\ad_{\qka{\lv_i}}^T\cdot\ld\qka{\lmu_i}+ \ad_{\qka{\lmu_i}}^D \ld\qka{\lv_i}\big)+\ldh\qka{\lG_i}.
\end{aligned}
\end{equation}

In \cref{subsection::pp2}, \cref{eqa::ldS,eqa::ldv,eqa::ldmu,eqa::ldG,eqa::ldo} will be used to prove \cref{propa::dabi}. 

\subsection{Differentiation on Lie Groups}\label{subsectiona::diff}
For an analytical function $f:\R^n\rightarrow \R$, the directional derivative at $x\in\R^n$ in the direction $\delta x$ is defined to be 
\begin{equation*}
	\tD f(x)\cdot\delta x = \left.\frac{\mathrm{d}}{\mathrm{d}t} f(x+t\cdot\delta x)\right|_{t=0}
\end{equation*}
in which $\tD f(x)=\begin{bmatrix}
\frac{\partial f}{\partial x_1} & \frac{\partial f}{\partial x_2} &\cdots & \frac{\partial f}{\partial x_n} 
\end{bmatrix}^T\in\R^n$.

In a similar way, we might define the directional derivative on Lie groups using the Lie algebra and the exponential map as follows.

\numberwithin{definition}{section}
\begin{definition}\label{def::df}
	If $G$ is a $n$-dimensional smooth Lie group and $f:G\longrightarrow \R$ is a smooth function on $G$, the directional derivative at $g\in G$ in the direction $\le= \delta g g^{-1}\in T_e G$ is defined to be
	$$\tD f(g)\cdot \le = \left.\frac{\mathrm{d}}{\mathrm{d}t} f\left(\exp\left(t\cdot\le\right)g\right)\right|_{t=0}. $$
	Moreover, if $\overline{e}_1$, $\overline{e}_2$, $\cdots$, $\overline{e}_n$ is a basis for the Lie algebra $T_e G$, then $\tD f(g)$ can be explicitly written as
	\begin{equation*}
	\tD f(g) = \left.\frac{\mathrm{d}}{\mathrm{d}t}\begin{bmatrix}
	 f\left(\exp\left(t\cdot\overline{e}_1\right)g\right) && f\left(\exp\left(t\cdot\overline{e}_2\right)g\right) && \cdots && f\left(\exp\left(t\cdot\overline{e}_n\right)g\right)
	\end{bmatrix}^T\right|_{t=0}.
	\end{equation*}
\end{definition}

In regard to Lie group theory, $\R^n$ is also a smooth Lie group for which the binary operation is addition, the Lie algebra is itself and the exponential map is the identity map. Furthermore, the definition of directional derivatives on Lie groups in \cref{def::df} is consistent with the definition of directional derivatives in $\R^n$. Therefore, it is without loss of any generality to interpret all the quantities in this paper as elements of Lie groups and all the derivatives in this paper as derivatives on Lie groups that are defined by \cref{def::df}. 

In this paper, following the notation in multivariate calculus, if $f:G_1\times G_2\times \cdots \times G_d\rightarrow \R$ is a smooth function in which $G_1$, $G_2$, $\cdots$, $G_d$ are Lie groups, we use $\tD_i f$ to denote the derivative with respect to $G_i$. In particular, for $\qka{\lF_i}=\qka{\lF_i}(\qka{g_i},\qka{\lv_i},\qka{u_i})$ that is used for the computation of the Newton direction in \cref{algorithm::dabi_b}, note that $\tD_1\qka{\lF_i}$ is the derivative with respect to $\qka{g_i}$ and $\tD_2\qka{\lF_i}$ is the derivative with respect to $\qka{\lv_i}$.

\counterwithout{definition}{section}
\section{Proof of Propositions}\label{section::proof}
\setcounter{definition}{0}
\setcounter{prop}{0}

In this section, we review and prove \cref{prop::dabi,prop::dkdq,prop::dvkdq,prop::eval} in  [1] though these proofs are not necessary for implementation.

\subsection{Proof of \cref{prop::eval}}\label{subsection::prop1}
In \cref{subsection::eval}, we define the discrete articulated body momentum and discrete articulated body impulse are respectively as follows.

\begin{definition}\label{defa::abmom}
	The discrete articulated body momentum $\lmu^{k,\alpha}_i\in \R^6$ for articulated body $i$ is defined to be
	\begin{equation}\label{eqa::abmom}
	\qka{\lmu_i} = \qka{\lM_i}\qka{\lv_i}+\sum_{j\in\chd(i)} \qka{\lmu_j}\;\;\;\quad
	\forall \alpha=0,\,1,\cdots,\,s
	\end{equation}
	in which $\qka{\lM_i}$ and $\qka{\lv_i}$ are respectively the spatial inertia matrix and spatial velocity of rigid body $i$.
\end{definition}
\begin{definition}\label{defa::abF}
	Suppose $\lF_i(t)\in\R^6$ is the sum of all the wrenches directly acting on rigid body $i$, which does not include those applied or transmitted through the joints that are connected to rigid body $i$. The discrete articulated body impulse $\qka{\lG_i}\in\R^6$ for articulated body $i$ is defined to be
	\begin{equation}\label{eqa::abF}
		\qka{\lG_i}=\lF_{i}^{k,\alpha} + \sum_{j\in\chd(i)}\lG_{j}^{k,\alpha}
	\end{equation}
	in which $\qka{\lF_i}=\qa{\omega} \lF_i(\qka{t})\Delta t\in\R^6$ is the discrete impulse acting on rigid body $i$. Note that $\lF_i(t)$, $\qka{\lF_i}$ and $\qka{\lG_i}$ are expressed in frame $\{0\}$. 
\end{definition}

The DEL equations \cref{eq::DEL_general} can be recursively evaluated with $\qka{\lmu_i}$ and $\qka{\lF_i}$ as \cref{propa::eval} indicates.

\begin{prop}\label{propa::eval}
	If $Q_i(t)\in\R$ is the sum of all joint forces applied to joint $i$ and $p^k=\begin{bmatrix}
	p_1^k & p_2^k &\cdots & p_n^k
	\end{bmatrix}^T\in \R^n$ is the discrete momentum, the DEL equations \cref{eq::DEL_general} can be evaluated as
	\begin{subequations}\label{eqa::eval}
		\begin{align}
		&\label{eqa::eval1}r_i^{k,0}=p_i^{k} + {\lS_i^{k,0}}^T\cdot\lO_i^{k,0}+\sum_{\beta= 0}^s a^{0\beta}{\lS_i^{k,\beta}}^T\cdot\lmu_i^{k,\beta}+Q_i^{k,0},\\
		&\label{eqa::eval2}r_i^{k,\alpha}={\lS_i^{k,\alpha}}^T\cdot\lO_i^{k,\alpha}+\sum_{\beta= 0}^s a^{\alpha\beta}{\lS_i^{k,\beta}}^T \cdot\lmu_i^{k,\beta}+\qka{Q_i}
		\quad\forall\alpha=1,\cdots,s-1,\\
		&p_i^{k+1}={\qks{\lS_i}}^T\cdot\qks{\lO_i}+\sum_{\beta= 0}^s a^{s\beta}{\lS_i^{k,\beta}}^T \cdot\lmu_i^{k,\beta}+\qks{Q_i}
		\end{align}
	\end{subequations}
	in which $\qka{r_i}$ is the residue of the DEL equations \cref{eq::DEL_general1,eq::DEL_general2}, $a^{\alpha\beta}=w^\beta b^{\beta\alpha}$, $\qka{\lO_i}= w^{\alpha}\Delta t\cdot {\ad}_{\qka{\lv_{i}}}^T\cdot\qka{\lmu_i}+\qka{\lG_{i}}$, and $\qka{Q_i}=\qa{\omega} Q_i(\qka{t})\Delta t$ is the discrete joint force applied to joint $i$. 
\end{prop}
\begin{proof}
		\allowdisplaybreaks
The Lagrangian of a mechanical system is defined to be
\begin{equation}\label{eqa::L}
\cL(q,\dot{q})=K(q,\dot{q})-V(q)
\end{equation}
in which $K(q,\dot{q})$ is the kinetic energy and $V(q)$ is the potential energy. It is by the definition of $\lF_i(t)$ and $Q_i(t)$ that 
$$\int_0^T \cF(t)\cdot\delta q dt-\delta \int_0^T V(q)dt = \int_0^T \sum_{i=1}^n\lF_i(t)\cdot\leta_idt+\int_0^T \sum_{i=1}^nQ_i(t)\cdot \delta q_i dt$$
in which $\leta_i = (\delta g_i g_i^{-1})^\vee$. Therefore, the Lagrange-d'Alembert principle \cref{eq::lda} is equivalent to
\begin{equation}\label{eqa::ld}
\delta\mathfrak{S}=\delta\int_0^T K(q,\dot{q})dt+\int_0^T \sum_{i=1}^n\lF_i(t)\cdot\leta_i dt + \int_0^T\sum_{i=1}^n Q_i(t) \cdot \delta q_idt=0 .
\end{equation}
As a result of \cref{eqa::ld,eq::dlda}, we have
\begin{multline}\label{eqa::var}
\sum_{k=0}^{N-1}\sum_{\alpha=0}^s w^\alpha\sum_{i=1}^{n} \Big[\left\langle \tfrac{\partial K}{\partial q_i}(\qka{q},\qka{\dot{q}}),\delta\qika\right\rangle + \left\langle\tfrac{\partial K}{\partial \dot{q}_i}(\qka{q},\qka{\dot{q}}), \delta \qdotika \right\rangle+\\ \innprod{\lF_i(t^{k,\alpha})}{\le_i^{k,\alpha}} +  \innprod{Q_i(t^{k,\alpha})}{\delta \qika}\Big]\Delta t=0.
\end{multline}
Note that the kinetic energy $K(\qka{q},\qka{\dot{q}})$ is
\begin{equation}\label{eqa::K}
	K(\qka{q},\qka{\dot{q}})=\frac{1}{2}\sum_{j=1}^{n} \qka{\lv_j}^T\qka{\lM_j} \qka{\lv_j}
\end{equation}
in which $\qka{\lM_i}\in\R^{6\times 6}$ is the spatial inertia matrix and $\qka{\lv_i}\in\R^6$ is the spatial velocity. Using \cref{eqa::dvdqdot2,eqa::K,eqa::abmom}, we obtain
\begin{equation}\label{eqa::dKdqdot}
\begin{aligned}
\frac{\partial K}{\partial \dot{q}_i}(\qka{q},\qka{\dot{q}})=&\sum_{j=1}^{n}\frac{\partial \qka{\lv_j}}{\partial \dot{q}_i}^T\qka{\lM_j}\qka{\lv_j}\\
=&\qka{\lS_i}^T\qka{\lM_i}\qka{\lv_i}+\sum_{j\in\des(i)}\qka{\lS_i}^T\qka{\lM_j}\qka{\lv_j}\\
=&\qka{\lS_i}^T \qka{\lmu_i}.
\end{aligned}
\end{equation}
In a similar way, as a result of \cref{eqa::dMdq2,eqa::dvdq2,eqa::Sdot,eqa::K,eqa::abmom}, a tedious but straightforward algebraic manipulation results in
\begin{equation}\label{eqa::dKdq}
\begin{aligned}
\frac{\partial K}{\partial q_i}(\qka{q},\qka{\dot{q}})=&\sum_{j\in\des(i)\cup\{i\}}\left[\ad_{\qka{\lS_i}}(\qka{\lv_j}-\qka{\lv_i})-\ad_{\qka{\lS_i}}\qka{\lv_j}\right]^T\qka{\lM_j}\qka{v_j}\\
=&\qka{S_i}^T\ad_{\lv_i}^T\cdot\qka{\lmu_i}\\
=&{\dot{\lS}\overline{\vphantom{S}}_i^{k,\alpha}}^T\qka{\lmu_i}.
\end{aligned}
\end{equation}
In addition, using \cref{eqa::etas,eqa::abF} and $\qka{\lF_i}=\qa{w}\lF_i(\qka{t})\Delta t$, we obtain
\begin{equation}\label{eqa::abFdq}
\begin{aligned}
	\sum_{i=1}^{n}\innprod{\qa{w}\lF_i(\qka{t})\Delta t}{\qka{\leta_i}}=&\sum_{i=1}^{n}\innprod{\qa{w}\lF_i(\qka{t})\Delta t}{\qka{\lS_i}\cdot\delta \qka{q_i}+\!\!\sum_{j\in\anc(i)}\qka{\lS_j}\cdot\qka{q_j}}\\
	=&\sum_{i=1}^{n}\innprod{\qka{\lF_i}+\sum_{j\in\des(i)}\qka{\lF_j}}{\qka{\lS_i}\cdot\delta\qka{q_i}}\\
	=&\sum_{i=1}^{n}\innprod{\qka{\lG_i}}{\qka{\lS_i}\cdot\delta\qka{q_i}}\\
	=&\sum_{i=1}^{n}\innprod{\qka{\lS_i}^T\qka{\lG_i}}{\delta\qka{q_j}}.
\end{aligned}
\end{equation}
From \cref{eq::qdot}, we obtain
\begin{equation}\label{eqa::dqdot}
\delta\qka{\dot{q}_i}=\frac{1}{\Delta t}\sum_{\beta=0}^s \qab{b}\cdot\delta\qkb{q_i}.
\end{equation}
Substituting \cref{eqa::dKdq,eqa::dKdqdot,eqa::abFdq} into \cref{eqa::var} and simplifying the resulting equation with \cref{eqa::dqdot} as well as the chain rule, we obtain
\begin{equation}
\nonumber
	\sum_{k=0}^{N-1}\sum_{\alpha=0}^{s}\sum_{i=1}^{n}\innprod{\qka{\lS_i}^T\cdot\qka{\lO_i}+\sum_{\beta=0}^{s}\qab{a}\qkb{\lS_i}^T\cdot\qkb{\lmu_i}+\qka{Q_i}}{\delta\qka{q_i}}=0
\end{equation}
in which $a^{\alpha\beta}=w^\beta b^{\beta\alpha}$, $\qka{\lO_i}= w^{\alpha}\Delta t\cdot {\ad}_{\qka{\lv_{i}}}^T\cdot\qka{\lmu_i}+\qka{\lG_{i}}$ and $\qka{Q_i}=\qa{\omega} Q_i(\qka{t})\Delta t$. The equation above is equivalent to requiring
\begin{align*}
&p_i^{k} + {\lS_i^{k,0}}^T\cdot\lO_i^{k,0}+\sum_{\beta= 0}^s a^{0\beta}{\lS_i^{k,\beta}}^T\cdot\lmu_i^{k,\beta}+Q_i^{k,0}=0,\\
&{\lS_i^{k,\alpha}}^T\cdot\lO_i^{k,\alpha}+\sum_{\beta= 0}^s a^{\alpha\beta}{\lS_i^{k,\beta}}^T \cdot\lmu_i^{k,\beta}+\qka{Q_i}=0
\quad\forall\alpha=1,\cdots,s-1,\\
&p_i^{k+1}={\qks{\lS_i}}^T\cdot\qks{\lO_i}+\sum_{\beta= 0}^s a^{s\beta}{\lS_i^{k,\beta}}^T \cdot\lmu_i^{k,\beta}+\qks{Q_i}.
\end{align*}
This completes the proof.
\end{proof}
\subsection{Proof of \cref{prop::dabi}}\label{subsection::pp2}
\setcounter{assumption}{0}
In \cref{subsection::newton}, we make the assumption on the discrete impulse $\qka{\lF_i}$ and discrete joint force $\qka{Q_i}$ as follows.

\begin{assumption}\label{assumptiona::abi}
	Let $u(t)$ be control inputs of the mechanical system, we assume that the discrete impulse $\qka{\lF_i}$ and discrete joint force $\qka{Q_i}$ can be respectively formulated as
	$\qka{\lF_i}=\qka{\lF_i}(\qka{g_i},\qka{\lv_i},\qka{u}) $
	and
	$\qka{Q_i}=\qka{Q_i}(\qka{q_i},\qka{\dot{q}_i},\qka{u})$
	in which $\qka{u}=u(\qka{t})$.
\end{assumption}
\setcounter{prop}{0}
\numberwithin{prop}{section}

From the notion of the spatial variation in \cref{subsection::spatial}, we have the following proposition for the Newton direction computation, which is later used in the proof of \cref{prop::dabi}.

\begin{prop}
	If $\delta \qka{q_i}$ is the Newton direction for $\qka{q_i}$, $\qka{r_i}$ is the residue of the DEL equations \cref{eq::eval1,eq::eval2}, and \cref{assumption::abi} holds, the computation of the Newton direction $\delta\qka{q_i}$ is equivalent to requiring
	\begin{subequations}\label{eq::dvar}
		\begin{multline}\label{eq::dvar1}
		\ld\qka{\lmu_i} = \qka{\lM_i}\ld\qka{\lv_i} +\sum_{j\in\chd(i)}\big(\ld\qka{\lmu_j}-\ad_{\qka{\lmu_j}}^D\qka{\lS_j}\cdot\delta \qka{q_j}\big)\\
		\forall \alpha=0,\,1,\,\cdots,\,s,
		\end{multline}
		\begin{multline}\label{eq::dvar2}
		\ldh\qka{\lG_i}= \big(\tD_1\qka{\lF_i}+\ad_{\qka{\lF_i}}^D-\ad_{\qka{\lv_i}}\big)\cdot\qka{\le_i} +\tD_2\qka{\lF_i}\cdot\ld\qka{\lv_i} +\\
		\sum_{j\in\chd(i)}\big(\ldh\qka{\lG_j}-\ad_{\qka{\lG_j}}^D \qka{\lS_j}\cdot\delta\qka{q_j}\big)\quad
		\forall \alpha=0,\,1,\,\cdots,\,s-1,
		\end{multline}
		\begin{multline}\label{eq::dvar3}
		\ldh\qka{\lO_i} = \omega^\alpha\Delta t\cdot\big(\ad_{\qka{\lv_i}}^T\cdot\ld\qka{\lmu_i}+\ad_{\qka{\lmu_i}}^D\ld\qka{\lv_i}\big)+
		\ldh\qka{\lG_i}\\
		\forall \alpha=0,\,1,\,\cdots,\,s-1,
		\end{multline}
		\begin{multline}\label{eq::dvar4}
		\qka{\lS_i}^T\ldh\qka{\lO_i} +\sum_{\beta=0}^{s}\qab{a}\qkb{\lS_i}^T\ld\qkb{\lmu_i}+\tD_1\qka{Q_i}\cdot\delta\qka{q_i}+\\
		\tD_2\qka{Q_i}\cdot\delta\qka{\dot{q}_i} = -\qka{r_i}\quad\forall \alpha=0,\,1,\,\cdots,\,s-1.
		\end{multline}
	\end{subequations}
	in which $\ld\qka{\lv_i}$, $\ld\qka{\lmu_i}$, $\ldh\qka{\lG_i}$ and $\ldh\qka{\lO_i}$ are the spatial variations of $\qka{\lv_i}$, $\qka{\lmu_i}$, $\qka{\lG_i}$ and $\qka{\lO_i}$, respectively. Note that $\delta\qkz{q_i}=0$ and $\qkz{\le_i}=0$ though $\ld\qkz{\lv_i}\neq 0$.
\end{prop}
\begin{proof}
	\allowdisplaybreaks
	\cref{eq::dvar1,eq::dvar3} are respectively the same as \cref{eqa::ldmu,eqa::ldo}, thus we only need to prove \cref{eq::dvar2,eq::dvar4}.\par
	From \cref{assumptiona::abi}, we have $\qka{\lF_i}=\qka{\lF_i}(\qka{g_i},\qka{\lv_i},\qka{u}) $, and since $\delta\qka{u_i}=0$, we obtain $\delta \qka{\lF_i}$ as
	$$\delta \qka{\lF_i}=\tD_1\qka{\lF_i}\cdot\qka{\le_i} + \tD_2\qka{\lF_i}\cdot\delta\qka{\lv_i}.  $$
	According to \cref{eqa::db2}, the spatial variation $\ldh\qka{\lF_i}$ is 
	$$\ldh\qka{\lF_i} = \tD_1\qka{\lF_i}\cdot\qka{\le_i} + \tD_2\qka{\lF_i}\cdot\delta\qka{\lv_i} + \ad_{\qka{\le_i}}^T\qka{\lF_i}. $$
	Since $\delta\qka{\lv_i}=\ld\qka{\lv_i}+\ad_{\qka{\le_i}}\qka{\lv_i}$, $\ad_{\qka{\lv_i}}\qka{\le_i}=-\ad_{\qka{\le_i}}\qka{\lv_i}$ as well as $\ad_{\qka{\le_i}}^T\qka{\lF_i}=\ad_{\qka{\lF_i}}^D\qka{\le_i}$, the equation above is equivalent to
	$$\ldh\qka{\lF_i}= \big(\tD_1\qka{\lF_i}+\ad_{\qka{\lF_i}}^D-\tD_2\qka{\lF_i}\ad_{\qka{\lv_i}}\big)\cdot\qka{\le_i} + \tD_2\qka{\lF_i}\cdot\ld\qka{\lv_i}. $$
	Substitute the equation above into \cref{eqa::ldG}, the result of which is \cref{eq::dvar2}.
	
	As for the proof of \cref{eq::dvar4}, from \cref{eq::eval2,eq::eval1}, the Newton direction $\delta\qka{q_i}$ requires that
	\begin{multline}\label{eqa::domega}
		\delta\big(\qka{S_i}^T\lO_i\big)+\sum_{\beta=0}^s\qab{a}\delta\big(\qkb{S_i}^T\qkb{\lmu_i}\big)+\tD_1\qka{Q_i}\cdot\delta\qka{q_i}+\\
\tD_2\qka{Q_i}\cdot\delta\qka{\dot{q}_i} = -\qka{r_i}\quad\forall \alpha=0,\,1,\,\cdots,\,s-1.
	\end{multline}
	As a result of \cref{eqa::lddot,eqa::ldS}, we have $\delta\big(\qka{\lS_i}^T\qka{\lmu_i}\big) =\qka{\lS_i}^T\ld\qka{\lmu_i}$ and
	$\delta\big(\qka{\lS_i}^T\qka{\lO_i}\big)=\qka{\lS_i}^T\ldh\qka{\lO_i}$, with which and \cref{eqa::domega}, we obtain \cref{eq::dvar4}. This completes the proof.
\end{proof}

In \cref{subsection::newton}, \cref{prop::dabi} to compute the Newton direction is stated as follows, for which note that the higher-order variational integrator has $s+1$ control points and the mechanical system has $n$ degrees of freedom.

\setcounter{prop}{1}
\counterwithout{prop}{section}
\begin{prop}\label{propa::dabi}
	For higher-order variational integrators of unconstrained mechanical systems, if \cref{assumption::abi} holds and $\qk{\JJ}^{-1}(\qk{\lq})$ exists, the Newton direction $\delta \qk{\lq}= -{\qk{\JJ}}^{-1}(\qk{\lq})\cdot\qk{r}$ can be computed with \cref{algorithm::dabi} in $O(s^3n)$ time.
\end{prop}
\begin{proof}
	The proof consists of proving the correctness and the $O(n)$ complexity of the algorithms.
	
For each $j\in\chd(i)$, we suppose that there exists $\qkap{D_j}$, $\qkan{G_j}$, $\qka{l_j}$ and $\qkap{\Pi_j}$, $\qkan{\Psi_j}$, $\qka{\zeta_j}$ such that
\begin{multline}\label{eq::dmu}
\ld\qka{\lmu_j} = \sum_{\rho=0}^{s}\qkap{D_j}\cdot\ld\qkp{\lv_j} + \sum_{\nu=1}^{s}\qkan{G_j}\cdot\qkn{\le_j} + \qka{l_j}\\
\forall \alpha=0,\,1,\,\cdots,\,s,
\end{multline}
\begin{multline}\label{eq::dgamma}
\ldh\qka{\lG_j} = \sum_{\rho=0}^{s}\qkap{\Pi_j}\cdot\ld\qkp{\lv_j} + \sum_{\nu=1}^{s}\qkan{\Psi_j}\cdot\qkn{\le_j} + \qka{\zeta_j}\\
\forall \alpha=0,\,1,\,\cdots,\,s-1.
\end{multline}
According to \cref{eqa::eta,eqa::ldv,eqa::dqdot}, $\ld\qkp{\lv_j}$ and $\qkn{\le_j}$ can be respectively computed as
\begin{equation}\label{eq::dgk}
\qkn{\le_j}=\qkn{\le_i} + \qkn{\lS_j}\cdot\delta\qjkn
\end{equation}
and
\begin{equation}\label{eq::dvk}
\ld\qkp{\lv_j}=\ld \qkp{\lv_i} + \qkp{\dot{\lS}{}_j}\cdot\delta\qkp{q_j} + \frac{1}{\Delta t}\qkp{\lS_j} \sum_{\gamma=1}^s \qpg{b} \cdot\delta\qkg{q_j}
\end{equation}
for which note that $\delta\qkz{q_j}=0$.
Substitute \cref{eq::dgk,eq::dvk} into \cref{eq::dmu}, algebraic manipulation shows that
\begin{equation}\label{eq::abmb}
\ld\qka{\lmu_j} = \sum_{\rho=0}^{s}\qkap{D_j}\cdot\ld\qkp{\lv_i} + \sum_{\nu=1}^{s}\qkan{G_j}\cdot\qkn{\le_i} + 
\qka{l_j} + \sum_{\gamma=1}^s \qkag{H_j}\delta\qkg{q_j},
\end{equation}
in which
$$\qkag{H_j}= \qkag{D_j}\dot{\lS}{}_j^{k,\gamma} +\qkag{G_j}\qkg{\lS_j} + \dfrac{1}{\Delta t}\sum\limits_{\rho=0}^s\qpg{b} \qkap{D_j}\qkp{\lS_j}.$$ In a similar way, using \cref{eq::dgamma,eq::dgk,eq::dvk}, we also have
\begin{multline}\label{eq::abFb}
\ldh\qka{\lG_j} = \sum_{\rho=0}^{s}\qkap{\Pi_j}\cdot\ld\qkp{\lv_i} + \sum_{\nu=1}^{s}\qkan{\Psi_j}\cdot\qkn{\le_i} + \qka{\zeta} + \sum_{\gamma=1}^s \qkag{\Phi_j}\delta\qkg{q_j}
\end{multline}
in which
\begin{equation}
\nonumber
\qkag{\Phi_j}= \qkag{\Pi_j}\qkg{\dot{\lS}{}_j} +\qkag{\Psi_j}\qkg{\lS_j} + \dfrac{1}{\Delta t}\sum\limits_{\rho=0}^s\qpg{b} \qkap{\Pi_j}\qkp{\lS_j}.
\end{equation}
From \cref{eqa::Sdot,eqa::ldo,eq::abmb,eq::abFb,eq::dvk} and $$\qka{\lS_j}^T\ad_{\qka{\lS_j}}^T\qka{\lmu_j}=\qka{\lS_j}^T\ad_{\qka{\lmu_j}}^D\qka{\lS_j}=0,$$
we obtain
\begin{equation}\label{eq::omega}
\qka{\lS_j}^T\ldh \qka{\lO_j} = \sum_{\rho=0}^s \qkap{\Theta_j}\cdot\ld \qkp{\lv_i} +\sum_{\nu=1}^{s}\qkan{\Xi} \cdot\qkn{\le_i}+ \qka{\xi_j}
\end{equation}
in which
\begin{align*}
\qkap{\Theta_j}=\;&\qa{w}\Delta t\cdot\big({\qka{\dot{\lS}\overline{\vphantom{S}}_j}}^T \qkap{D_j}+ \qap{\sigma}\qka{\lS_j}^T\ad_{\qka{\lmu_j}}^D\big)+\qka{\lS_j}^T\qkap{\Pi_j},\\
\qkan{\Xi_j}=\;&\qa{w}\Delta t\cdot{\qka{\dot{\lS}\overline{\vphantom{S}}_j}{}}^T \qkan{G_j} + \qka{\lS_j}^T\qkan{{\Psi}_j},\\
\qka{\xi_j}\;\,= \;& \qa{w}\Delta t\cdot{\qka{\dot{\lS}\overline{\vphantom{S}}_j}{}}^T  \qka{l_j} + \qka{\lS_j}^T\qka{\zeta_j} +\sum_{\gamma=1}^s\Big[ \qa{w}\Delta t\cdot \big({\qka{\dot{\lS}\overline{\vphantom{S}}_j{}}}^T\qkag{H_j}+\\
& \qag{\sigma}\qka{\lS_j}^T\ad_{\qka{\lmu_j}}^D\qka{\dot{\lS}\overline{\vphantom{S}}_j}\big)+\qka{\lS_j}^T\qkag{\Phi_j}\Big]\delta\qkg{q_j},
\end{align*}
and note that $\qap{\sigma}$ is given in \cref{eq::sigma} of \cref{algorithm::dabi_b}.
Substituting \cref{eq::abmb,eq::omega,eqa::dqdot} into \cref{eq::dvar4}, we obtain
\begin{multline}\label{eqa::error}
\sum_{\rho=0}^s \qkap{\overline{\Theta}_j}\cdot\ld \qkp{\lv_i}+\sum_{\nu=1}^{s}\qkan{\overline{\Xi}_j} \cdot\qkn{\le_i}+\qka{\overline{\xi}_j}+\sum_{\gamma=1}^s\qkag{\Lambda_{j}}\cdot\delta\qkg{q_j} = -\qka{r_j}\\
\forall \alpha=0,\,1,\,\cdots,\,s-1.
\end{multline}
in which
\begin{align}
&\nonumber\qkap{\overline{\Theta}_j}=\qkap{\Theta_j} + \sum_{\beta=0}^s \qab{a}\qkb{\lS_j}^T\qkbp{D_j},\\
&\nonumber\qkan{\overline{\Xi}_j}=\qkan{\Xi_j} +  \sum_{\beta=0}^s \qab{a}\qkb{\lS_j}^T \qkbn{G_j},\\
&\nonumber\qka{\overline{\xi}_j} \;\;= \qa{w}\Delta t\cdot{\qka{\dot{\lS}\overline{\vphantom{S}}_j}{}}^T  \qka{l_j} + \qka{\lS_j}^T\qka{\zeta_j} +\sum_{\beta=0}^{s}\qab{a}\qkb{\lS_j}^T \qkb{l_j},\\
&\begin{aligned}
\nonumber\qkag{\Lambda_{j}}= &\qa{w}\Delta t\cdot\qka{\dot{\lS}\overline{\vphantom{S}}_j}^T \qkag{H_j}+\qka{\lS_j}^T\qkag{\Phi_j} +\sum_{\beta=0}^s \qab{a} \qkb{\lS_j}^T\qkbg{H_j}+\\
& {\sigma}^{\alpha\gamma}\left(\tD_1\qka{Q_j}+\qa{w}\Delta t\cdot \qka{\lS_j}^T\ad_{\qka{\lmu_j}}^D\qka{\dot{\lS}\overline{\vphantom{S}}_j}\right) +\frac{1}{\Delta t} \qag{b}\cdot\tD_2\qka{Q_j}.
\end{aligned}
\end{align}
For notational convenience, we define $\qka{\Delta_j}$ to be
\begin{multline}\label{eq::delta}
\qka{\Delta_j}= \sum_{\rho=0}^s \qkap{\overline{\Theta}_j}\cdot\ld \qkp{\lv_i} + \sum_{\nu=1}^{s}\qkan{\overline{\Xi}_j} \cdot\qkn{\le_i} + \qka{\overline{\xi}_j}\\ \forall \alpha=0,\,1,\,\cdots,\,s-1.
\end{multline}
such that \cref{eqa::error} is rewritten as
\begin{equation}\label{eqa::error2}
\sum_{\gamma=1}^s\qkag{\Lambda_{j}}\cdot\delta\qkg{q_j} = -\qka{r_j}-\qka{\Delta_j}\quad\quad\forall \alpha=0,\,1,\,\cdots,\,s-1.
\end{equation}
In addition, if we further define $\qk{\Lambda_{j}}$, $\qk{r_j}$, $\qk{\Delta_j}$ and $\delta\qk{\lq_j}$ respectively as
\begin{align*}
&\qk{\Lambda_{j}}=\begin{bmatrix}
\qkag{\Lambda_{j}}
\end{bmatrix}\in\R^{s\times s},\\
&\qk{r_j}=\begin{bmatrix}
r_j^{k,0} & r_j^{k,1} &\cdots & r_j^{k,s-1}
\end{bmatrix}^T\in\R^s,\\
&\qk{\Delta_j}=\begin{bmatrix}
\Delta_j^{k,0} & \Delta_j^{k,1} &\cdots & \Delta_j^{k,s-1}
\end{bmatrix}^T\in\R^s,\\
&\delta\qk{\lq_j}=\begin{bmatrix}
\delta q_j^{k,1} & \delta q_j^{k,2} &\cdots & \delta q_j^{k,s}
\end{bmatrix}^T\in\R^s,
\end{align*}
in which $0\leq \alpha\leq s-1$ and $1\leq \gamma\leq s$, then \cref{eqa::error2} is equivalent to requiring
\begin{equation}\label{eqa::error1}
\qk{\Lambda_{j}}\cdot\delta\qk{\lq_j}=-\qk{r_j}-\qk{\Delta_j}.
\end{equation}
in which $\qk{\Lambda_{j}}$ is invertible since ${\JJ^k}^{-1}(\lq^k)$ exists. From \cref{eqa::error1}, we obtain
$$\delta\qk{\lq_j}= -\qk{\Lambda_{j}}^{-1}(\qk{r_j}+\qk{\Delta_j}).$$
If $\qk{\Lambda_{j}}^{-1}$ is explicitly written as $\qk{\Lambda_{j}}^{-1}=\begin{bmatrix}
\qkgvp{\lLam_j}
\end{bmatrix}\in\R^{s\times s}$ in which $1\leq \gamma\leq s$ and $0\leq\varrho\leq s-1$, expanding the equation above, we obtain
\begin{equation}\label{eq::qsol}
\delta \qkg{q_j}= -\sum_{\varrho=0}^{s-1}\qkgvp{\lLam_j}\left(\qkvp{r_j}+\qkvp{\Delta_j}\right)\quad\; \forall \gamma=1,\,2,\,\cdots,\,s.
\end{equation}
Substitute \cref{eq::delta} into \cref{eq::qsol}, the result is
\begin{equation}\label{eq::sol_q}
\delta\qkg{q_j}= \sum_{\rho=0}^{s}\qkgp{X_j}\cdot\ld\qkp{\lv_i} +\sum_{\nu=1}^{s}\qkgn{Y_j}\cdot\qkn{\le_i}+\qkg{y_j}
\end{equation}
in which
\begin{align*}
\qkgp{X_j}&=-\sum_{\varrho=0}^{s-1}\qkgvp{\lLam_{j}}\cdot\qkvpp{\overline{\Theta}_j},\\
\qkgn{Y_j}&=-\sum_{\varrho=0}^{s-1}\qkgvp{\lLam_{j}}\cdot\qkvpn{\overline{\Xi}_j},\\
\qkg{y_j}&=-\sum_{\varrho=0}^{s-1}\qkgvp{\lLam_{j}}\left(\qkvp{r_j}+\qkvp{\overline{\xi}_j}\right).
\end{align*}
Making use of \cref{eq::abmb,eq::sol_q} and canceling out $\delta \qkg{q_j}$, we obtain
\begin{equation}\label{eqa::dmuj}
\begin{aligned}
\ld\qka{\lmu_j}-\ad_{\qka{\lmu_j}}^D\qka{\lS_j}\cdot\delta\qka{q_j}=&\sum_{\rho=0}^s\qkp{\lD_j}\cdot\ld\qkp{\lv_i}+\sum_{\nu=1}^{s}\qkan{\overline{G}_j}\cdot\qkn{\le_i} + \qka{\overline{l}_j}
\end{aligned}
\end{equation}
in which $\alpha=0,\,1,\,\cdots,\, s$, and
\begin{subequations}\label{eqa::dmuj2}
\begin{align}
&\qkp{\lD_j}=\qkp{D_j}+\sum_{\gamma=1}^{s}\qkag{H_j}\qkgp{X_j}-\overline{\sigma}^{\alpha 0}\ad_{\qka{\lmu_j}}^D \qka{\lS_j}\qkap{X_j},\\
&\qkn{\overline{G}_j}=\qkan{G_j}+\sum_{\gamma=1}^{s}\qkag{H_j}\qkgn{Y_j} - \overline{\sigma}^{\alpha 0}\ad_{\qka{\lmu_j}}^D \qka{\lS_j}\qkan{Y_j},\\
&\qka{\overline{l}_j}=\qka{l_j}+\sum_{\gamma=1}^s\qkag{H_j}\qkg{y_j}- \overline{\sigma}^{\alpha 0}\ad_{\qka{\lmu_j}}^D \qka{\lS_j}\qka{y_j},
\end{align}
\end{subequations}
and note that $\overline{\sigma}^{\alpha 0}$ is given in \cref{eq::sigma} of \cref{algorithm::dabi_b}. In a similar way, using \cref{eq::abFb,eq::sol_q}, we obtain
\begin{equation}\label{eqa::dfj}
\ldh\qka{\lG_j}-\ad_{\qka{\lG_j}}^D\qka{\lS_j}\cdot\delta \qka{q_j}=\sum_{\rho=0}^{s}\qkap{\overline{\Pi}_j}\cdot\ld\qkp{\lv_j} + \sum_{\nu=1}^{s}\qkan{\overline{\Psi}_j}\cdot\qkn{\le_j} + \qka{\overline{\zeta}_j}
\end{equation}
in which $\alpha=1,\,2,\,\cdots,\,s$, and
\begin{subequations}\label{eqa::dfj2}
\begin{align}
&\qkap{\overline{\Pi}_j}=\qkap{\Pi_j}+\sum_{\gamma=1}^s\qkag{\Phi_j}\qkgp{X_j}-\overline{\sigma}^{\alpha 0}\ad_{\qka{\lG_j}}^D \qka{\lS_j}\qkap{X_j},\\
&\qkan{\overline{\Psi}_j}=\qkan{{\Psi}_j}+\sum_{\gamma=1}^s\qkag{\Phi_j}\qkgn{Y_j}-\overline{\sigma}^{\alpha 0}\ad_{\qka{\lG_j}}^D\qka{\lS_j}\qkan{Y_j},\\
&\qka{\overline{\zeta}_j}=\qka{\zeta_j}+\sum_{\gamma=1}^s\qkag{\Phi_j}\qkg{y_j}-
\overline{\sigma}^{\alpha 0}\ad_{\qka{\lG_j}}^D \qka{\lS_j}\qka{y_j}.
\end{align}
\end{subequations}
Finally, for each $j\in\chd(i)$, substituting \cref{eqa::dmuj,eqa::dfj} respectively into \cref{eq::dvar1,eq::dvar2} and applying \cref{eqa::dmuj2,eqa::dfj2} to expand $\qkp{\lD_j}$, $\qkn{\overline{G}_j}$, $\qka{\overline{l}_j}$ and $\qkap{\overline{\Pi}_j}$, $\qkan{\overline{\Psi}_j}$, $\qka{\overline{\zeta}_j}$, we respectively obtain
$\qkp{D_i}$, $\qkn{G_i}$, $\qka{l_i}$ and $\qkap{\Pi_i}$, $\qkan{\Psi_i}$, $\qka{\zeta_i}$ as \cref{eq::mux,eq::gammax} of \cref{algorithm::dabi_b} such that
\begin{multline}\label{eq::dmu3}
\ld\qka{\lmu_i} = \sum_{\rho=0}^{s}\qkap{D_i}\cdot\ld\qkp{\lv_i} + \sum_{\nu=1}^{s}\qkan{G_i}\cdot\qkn{\le_i} + \qka{l_i}\\
\forall \alpha=0,\,1,\,\cdots,\,s,
\end{multline}
\vspace{-2em}
\begin{multline}\label{eq::dgamma3}
\ldh\qka{\lG_i} = \sum_{\rho=0}^{s}\qkap{\Pi_i}\cdot\ld\qkp{\lv_i} + \sum_{\nu=1}^{s}\qkan{\Psi_i}\cdot\qkn{\le_i} + \qka{\zeta_i}\\
\forall \alpha=0,\,1,\,\cdots,\,s-1.
\end{multline}
In particular, note that even if rigid body $i$ is the leaf node of the tree representation whose $\chd(i)=\O$, there still exists $\qkp{D_i}$, $\qkn{G_i}$, $\qka{l_i}$ and $\qkap{\Pi_i}$, $\qkan{\Psi_i}$, $\qka{\zeta_i}$ from \cref{eq::mux,eq::gammax} of \cref{algorithm::dabi_b}. Moreover, as long as $\qkp{D_i}$, $\qkn{G_i}$, $\qka{l_i}$ and $\qkap{\Pi_i}$, $\qkan{\Psi_i}$, $\qka{\zeta_i}$ are given for each rigid body $i$, we can further obtain $\qkap{X_i}$, $\qkan{Y_i}$, $\qka{y_i}$ following lines 3 to 9 of \cref{algorithm::dabi_b}.

In summary, for each rigid body $i$, we have shown that $\qkap{X_i}$, $\qkan{Y_i}$, $\qka{y_i}$ as well as $\qkp{D_i}$, $\qkn{G_i}$, $\qka{l_i}$ and $\qkap{\Pi_i}$, $\qkan{\Psi_i}$, $\qka{\zeta_i}$ are computable through the backward pass by \cref{algorithm::dabi_b}, and $\delta\qka{q_i}$ as well as $\qka{\le_i}$ and $\ld\qka{\lv_i}$ are computable through the forward pass by lines 4 to 15 of \cref{algorithm::dabi}, which proves the correctness of the algorithms.

In regard to the complexity, \cref{algorithm::dabi_b} has $O(s^2)+O(s^3)$ complexity since there are $O(s^2)$ quantities and the computation of $\qka{\Lambda_{i}}^{-1}$ takes $O(s^3)$ time, and thus the backward pass by lines 1 to 3 of \cref{algorithm::dabi} totally takes $O(s^3n+s^2n)$ time. Moreover, in lines 4 to 15 of \cref{algorithm::dabi}, the forward pass takes $O(s^2n)$ time. As a result, the overall complexity of \cref{algorithm::dabi} is $O(s^3n)$, which proves the complexity of the algorithms.
\end{proof}
\subsection{Proof of \cref{prop::dkdq}}\label{subsection::pp3}
\begin{prop}\label{propa::dkdq}
	For the kinetic energy $K(q,\dot{q})$ of a mechanical system, $\frac{\partial^2 K }{\partial \dot{q}^2}$, $\frac{\partial^2 K }{\partial \dot{q}\partial q}$, $\frac{\partial^2 K }{\partial q\partial \dot{q}}$, $\frac{\partial^2 K}{\partial q^2}$ can be recursively computed with \cref{algorithm::dkdv} in $O(n^2)$ time.
\end{prop}
\begin{proof}
	\allowdisplaybreaks
According to \cref{eqa::dKdq,eqa::dKdqdot,eqa::abmom}, we have
\begin{equation}\label{eqa::dKdqdot2}
\frac{\partial K}{\partial \dot{q}_i}= \lS_i^T\Big(\lM_i\lv_i+\sum_{i'\in\des(i)}\lM_{i'}\lv_{i'}\Big)
\end{equation}
and
\begin{equation}\label{eqa::dKdq2}
\frac{\partial K}{\partial q_i}={\dot{\lS}}\overline{\vphantom{S}}_i^T\Big(\lM_i\lv_i+\sum_{i'\in\des(i)}\lM_{i'}\lv_{i'}\Big).
\end{equation}
Since $\lM_i\lv_i$, $\lS_i$ and $\dot{\lS}_i$ only depend on $q_j$ and $\dot{q}_j$ for $j\in \anc(i)\cup\{i\}$, it is straightforward to show from \cref{eqa::dKdq2,eqa::dKdqdot2} that the derivatives $\frac{\partial^2 K}{\partial \dot{q}_i\partial\dot{q}_j}$, $\frac{\partial^2 K}{\partial \dot{q}_i\partial q_j}$,  $\frac{\partial^2 K}{\partial q_i\partial \dot{q}_j}$ and $\frac{\partial^2 K}{\partial q_i \partial q_j}$ can be respectively computed as
\begin{equation}\label{eqa::dK2dqdotdqdot}
\frac{\partial^2 K}{\partial \dot{q}_i\partial \dot{q}_j}=
\begin{cases}
\frac{\partial }{\partial \dot{q}_j}\left(\frac{\partial K}{\partial \dot{q}_i}\right) & j\in \anc(i)\cup\{i\},\vspace{0.25em}\\
\frac{\partial^2 K}{\partial \dot{q}_j\partial \dot{q}_i} & j\in \des(i),\\
0 &\text{otherwise},
\end{cases}
\end{equation}
\begin{equation}\label{eqa::dK2dqdotdq}
\frac{\partial^2 K}{\partial \dot{q}_i\partial q_j}=
\begin{cases}
\frac{\partial }{\partial q_j}\left(\frac{\partial K}{\partial \dot{q}_i}\right) & j\in \anc(i)\cup\{i\},\vspace{0.25em}\\
\frac{\partial^2 K}{\partial q_j\partial \dot{q}_i} & j\in \des(i),\\
0 &\text{otherwise},
\end{cases}
\end{equation}
\begin{equation}\label{eqa::dK2dqdqdot}
\frac{\partial^2 K}{\partial q_i\partial \dot{q}_j}=
\begin{cases}
\frac{\partial }{\partial \dot{q}_j}\left(\frac{\partial K}{\partial q_i}\right) & j\in \anc(i)\cup\{i\},\vspace{0.25em}\\
\frac{\partial^2 K}{\partial \dot{q}_j\partial q_i} & j\in \des(i),\\
0 &\text{otherwise},
\end{cases}
\end{equation}
\begin{equation}\label{eqa::dK2dqdq}
\frac{\partial^2 K}{\partial q_i\partial q_j}=
\begin{cases}
\frac{\partial }{\partial q_j}\left(\frac{\partial K}{\partial q_i}\right) & j\in \anc(i)\cup\{i\},\vspace{0.25em}\\
\frac{\partial^2 K}{\partial q_j\partial q_i} & j\in \des(i),\\
0 &\text{otherwise}.
\end{cases}
\end{equation}
Therefore, we only need to consider the derivatives for $j\in\anc(i)\cup\{i\}$, whereas the derivatives for $j\notin\anc(i)\cup\{i\}$ are computed from \cref{eqa::dK2dqdq,eqa::dK2dqdqdot,eqa::dK2dqdotdq,eqa::dK2dqdotdqdot}. In addition, if $j\in\anc(i)\cup\{i\}$, using \cref{eqa::dMdq1,eqa::dvdq,eqa::Sdot,eqa::dvdqdot1}, we obtain
\begin{align}
\label{eqa::dMvdqdot}&\frac{\partial\lM_i\lv_i}{\partial \dot{q}_j}=\lM_i{\lS}_j,\\[0.5em]
\label{eqa::dMvdq}&\nonumber\frac{\partial\lM_i\lv_i}{\partial q_j}=-\ad_{\lS_j}^T\lM_i\lv_i-\lM_i\ad_{\lS_j}\lv_i+\lM_i\ad_{\lS_j}(\lv_i-\lv_j)\\
&\quad\quad\quad=\lM_i\dot{\lS}_j-\ad_{\lS_j}^T\lM_i\lv_i\\[0.5em]
\label{eqa::dSdotdqdot}&\frac{\partial\dot{S}_i}{\partial \dot{q}_j}=\ad_{\lS_j}\lS_i,\\[0.5em]
\label{eqa::dSdotdq}&\frac{\partial\dot{S}_i}{\partial q_j}=\ad_{\lv_i}\ad_{\lS_j}\lS_i+\ad_{\ad_{\lS_j}\left(\lv_i-\lv_j\right)}\lS_i.
\end{align}
For notational clarity, we define $\lmu_i$, $\M_i$, $\M_i^A$ and $\M_i^B$ as
\begin{align}
&\label{eqa::mu_i}\lmu_i = \lM_i\lv_i+\sum_{j\in\des(i)}\lM_j\lv_j=\lM_i\lv_i + \sum_{j\in\chd(i)}\lmu_j,\\
&\label{eqa::M_i}\M_i = \lM_i+\sum_{j\in\des(i)}\lM_j=\lM_i + \sum_{j\in\chd(i)}\M_j,\\
&\label{eqa::Ma}\M_i^A = \M_i \lS_i,\\
&\label{eqa::Mb}\M_i^B = \M_i\dot{\lS}_i-\ad_{\lmu_i}^D\lS_i
\end{align}
which will be used in the derivation of $\frac{\partial^2 K}{\partial \dot{q}_i\partial \dot{q}_j}$, $\frac{\partial^2 K}{\partial \dot{q}_i\partial q_j}$, $\frac{\partial^2 K}{\partial q_i\partial \dot{q}_j}$ and $\frac{\partial^2 K}{\partial q_i\partial q_j}$.\\

\vspace{0.5em}
\noindent1) $\frac{\partial^2 K}{\partial \dot{q}_i\partial \dot{q}_j}$\par
\vspace{0.5em}
If $j\in\anc(i)\cup\{i\}$, from \cref{eqa::dKdqdot2,eqa::dMvdqdot,eqa::M_i,eqa::Ma}, it is simple to show that
\begin{equation}\label{dK2dqdot2}
\begin{aligned}
	\frac{\partial^2 K}{\partial \dot{q}_i\partial \dot{q}_j}&=\frac{\partial }{\partial \dot{q}_j}\left(\frac{\partial K}{\partial \dot{q}_i}\right)\\
		&=\lS_i^T\Big(\lM_i\lS_j+\sum_{i'\in \des(i)}\lM_{i'}\lS_j\Big)\\
		&=\lS_i^T\Big(\lM_i+\sum_{i'\in \des(i)}\lM_{i'}\Big)\lS_j\\
		&=\lS_j^T\M_i\lS_i\\
		&=\lS_j^T\M_i^A.
\end{aligned}
\end{equation}\par
\noindent2) $\frac{\partial^2 K}{\partial \dot{q}_i\partial {q}_j}$\par
\vspace{0.5em}
If $j\in\anc(i)\cup\{i\}$, using \cref{eqa::dKdqdot2,eqa::dSdq1,eqa::dMvdq,eqa::M_i,eqa::Ma}, we obtain
\begin{equation}\label{dK2dqdotdq}
\begin{aligned}
	\frac{\partial^2K}{\partial \dot{q}_i\partial q_j}&=\frac{\partial }{\partial q_j}\left(\frac{\partial K}{\partial \dot{q}_i}\right)\\
	&=\sum_{i'\in\des(i)\cup\{i\}}\Big(\lS_i^T\lM_{i'}\dot{\lS}_j-\lS_i^T\ad_{\lS_j}^T\lM_{i'}\lv_{i'}+\lS_i^T\ad_{\lS_j}^T\lM_{i'}\lv_{i'}\Big)\\
	&=\lS_i^T\Big(\lM_i+\sum_{i'\in \des(i)}\lM_{i'}\Big)\dot{\lS}_j\\
	&=\dot{\lS}\overline{\vphantom{S}}_j^T\M_i\lS_i\\
	&=\dot{\lS}\overline{\vphantom{S}}_j^T\M_i^A.
\end{aligned}
\end{equation}\par
\noindent3) $\frac{\partial^2 K}{\partial \dot{q}_i\partial {q}_j}$\par
\vspace{0.5em}
If $j\in\anc(i)\cup\{i\}$, using \cref{eqa::dKdq2,eqa::dMvdqdot,eqa::M_i,eqa::dSdotdqdot,eqa::mu_i}, we obtain
\begin{equation}
\nonumber
\begin{aligned}
\frac{\partial^2K}{\partial q_i\partial \dot{q}_j}&=\frac{\partial }{\partial \dot{q}_j}\left(\frac{\partial K}{\partial q_i}\right)\\
&=\sum_{i'\in\des(i)\cup\{i\}}\Big(\dot{\lS}\overline{\vphantom{S}}_i^T\lM_{i'}\lS_j+\lS_i^T\ad_{\lS_j}^T\lM_{i'}\lv_{i'}\Big)\\
&=\lS_j^T\Big(\lM_i+\sum_{i'\in \des(i)}\lM_{i'}\Big)\dot{\lS}_i+\Big(\lM_i\lv_i+\sum_{i'\in \des(i)}\lM_{i'}\lv_{i'}\Big)^T\ad_{\lS_j}\lS_i\\
&=\lS_j^T\M_i\dot{\lS}_i + \lmu_i^T\ad_{\lS_j}\lS_i.\\
\end{aligned}
\end{equation}
Then simplify the equation above with $\lmu_i^T\ad_{{\lS}_j}\lS_i=-{\lS}\overline{\vphantom{S}}_j^T\ad_{\lmu_i}^D\lS_i$ and \cref{eqa::Mb}, the result is
\begin{equation}\label{dK2dqdqdot}
	\frac{\partial^2K}{\partial q_i\partial \dot{q}_j}={\lS_j^T}\left(\M_i\dot{\lS}_i-\ad_{\lmu_i}^D\lS_i\right)=\lS_j^T\M_i^B.
\end{equation}

\noindent4) $\frac{\partial^2 K}{\partial q_i\partial q_j}$\par
\vspace{0.5em}
\indent If $j\in\anc(i)\cup\{i\}$, using \cref{eqa::dMvdq,eqa::dMvdqdot,eqa::dSdotdq,eqa::dKdq2,eqa::M_i,eqa::mu_i,eqa::Sdot} and $\ad_{\ad_{\lv_i}\lS_j}=\ad_{\lv_i}\ad_{\lS_j}-\ad_{\lS_j}\ad_{\lv_i}$,
we obtain
\begin{equation}
\nonumber
\begin{aligned}
\frac{\partial^2 K}{\partial q_i\partial q_j}&=\frac{\partial }{\partial q_j}\left(\frac{\partial K}{\partial q_i}\right)\\
&=\sum_{i'\in\des(i)\cup\{i\}}\bigg[\left(\lM_{i'}\lv_{i'}\right)^T\Big(\ad_{\lv_i}\ad_{\lS_j}\lS_i-\ad_{\lS_j}\ad_{\lv_i}\lS_i+\\
&\quad\quad\quad\quad\hspace{15em}\ad_{\ad_{\lS_j}\left(\lv_i-\lv_j\right)}\lS_i\Big)+\dot{\lS}\overline{\vphantom{S}}_j^T\lM_{i'}\dot{\lS}_i\bigg]\\
&=\dot{\lS}\overline{\vphantom{S}}_j^T\Big(\lM_i+\sum_{i'\in \des(i)}\lM_{i'}\Big)\dot{\lS_i}+\Big(\lM_i\lv_i+\sum_{i'\in \des(i)}\lM_{i'}\lv_{i'}\Big)^T\ad_{\dot{\lS_j}}\lS_i\\
&=\dot{\lS}\overline{\vphantom{S}}_j^T\M_i\dot{\lS}_i+\lmu_i^T\ad_{\dot{\lS}_j}\lS_i.
\end{aligned}
\end{equation}
Similar to $\frac{\partial^2 K}{\partial \dot{q}_i\partial {q}_j}$, using $\lmu_i^T\ad_{\dot{\lS}_j}\lS_i=-\dot{\lS}\overline{\vphantom{S}}_j^T\ad_{\lmu_i}^D\lS_i$ and \cref{eqa::Mb}, we obtain
\begin{equation}\label{dK2dq2}
\frac{\partial^2 K}{\partial q_i\partial q_j} =\dot{\lS}\overline{\vphantom{S}}_j^T\left(\M_i\dot{\lS}_i-\ad_{\lmu_i}^D\lS_i\right) = \dot{\lS}\overline{\vphantom{S}}_j^T\M_i^B.
\end{equation}\par
Thus far, we have proved that $\frac{\partial^2 K}{\partial \dot{q}_i\partial \dot{q}_j}$, $\frac{\partial^2 K}{\partial \dot{q}_i\partial q_j}$, $\frac{\partial^2 K}{\partial q_i\partial \dot{q}_j}$ and $\frac{\partial^2 K}{\partial q_i\partial q_j}$ can be computed using \cref{eqa::dK2dqdotdq,eqa::dK2dqdotdqdot,eqa::dK2dqdq,eqa::dK2dqdqdot,dK2dq2,dK2dqdot2,dK2dqdotdq,dK2dqdqdot} with which we further have $\frac{\partial^2 K}{\partial \dot{q}^2}$, $\frac{\partial^2 K}{\partial \dot{q}\partial q}$, $\frac{\partial^2 K}{\partial q\partial \dot{q}}$ and $\frac{\partial^2 K}{\partial q^2}$ computed.

As for the complexity of \cref{algorithm::dkdv}, it takes $O(n)$ time to pass the tree representation forward to compute $g_i$, $M_i$, $\lS_i$, $\lv_i$, $\dot{\lS}_i$ and another $O(n)$ time to pass the tree representation backward to compute $\lmu_i$, $\M_i$, $\M_i^A$ and $\M_i^B$. In the backward pass, $\frac{\partial^2 K}{\partial \dot{q}_i\partial \dot{q}_j}$, $\frac{\partial^2 K}{\partial \dot{q}_i\partial q_j}$, $\frac{\partial^2 K}{\partial q_i\partial \dot{q}_j}$ and $\frac{\partial^2 K}{\partial q_i\partial q_j}$ are computed for each $i$ using \cref{eqa::dK2dqdotdq,eqa::dK2dqdotdqdot,eqa::dK2dqdq,eqa::dK2dqdqdot,dK2dq2,dK2dqdot2,dK2dqdotdq,dK2dqdqdot} which totally takes at most $O(n^2)$ time. Therefore, the complexity of \cref{algorithm::dkdv} is $O(n^2)$. This completes the proof.
\end{proof}
\subsection{Proof of \cref{prop::dvkdq}}\label{subsection::pp4}
\begin{prop}\label{propa::dvkdq}
	If $\gv\in \R^3$ is gravity, then for the gravitational potential energy $V_{\vec{g}}(q)$,  $\frac{\partial^2 V_{\vec{g}} }{\partial q^2}$ can be recursively computed with \cref{algorithm::dvdq} in $O(n^2)$ time.
\end{prop}
\begin{proof}
	\allowdisplaybreaks
	It is known that the gravitational potential energy $V_\gv(q)$ is
	\begin{equation}\label{eqa::V}
		V_{\gv}(q)=-\sum_{i=1}^{n}m_i\cdot  \gv^Tp_i.
	\end{equation}
	in which $m_i\in\R$ is the mass of rigid body $i$ and $p_i\in\R^3$ is the mass center of rigid body $i$ as well as the origin of frame $\{i\}$. In addition, from \cref{eqa::dgdq1,eqa::dgdq2}, we have 
	\begin{subequations}\label{eqa::dpdq}
		\begin{equation}\label{eqa::dpdq1}
		\frac{\partial p_i}{\partial q_j}=\begin{cases}
		\hat{\overline{s}}_j p_i + \overline{n}_j & j\in\anc(i)\cup\{i\},\\
		0 & \text{otherwise},
		\end{cases}
		\end{equation}
		and
		\begin{equation}\label{eqa::dpdq2}
		\frac{\partial p_j}{\partial q_i}=\begin{cases}
		\hat{\overline{s}}_i p_j + \overline{n}_i & j\in\des(i)\cup\{i\},\\
		0 & \text{otherwise},
		\end{cases}
		\end{equation}
	\end{subequations}
in which $\overline{s}_i,\overline{n}_i\in\R^3$ and $\lS_i=\begin{bmatrix}
\overline{s}_i^T & \overline{n}_i^T
\end{bmatrix}^T\in\R^6$ is the spatial Jacobian of joint $i$. From \cref{eqa::V,eqa::dpdq2}, algebraic manipulation gives
\begin{equation}\label{eqa::dVdq}
\frac{\partial V_\gv}{\partial q_i}=-\lS_i^T\bigg(m_i\begin{bmatrix}
\hat{p}_i\gv\\
\gv
\end{bmatrix}+\sum_{i'\in\des(i)}m_{i'}\begin{bmatrix}
\hat{p}_{i'}\gv\\
\gv
\end{bmatrix}\bigg).
\end{equation}
Moreover, observe that $\lS_i$ and $p_i$ only depends on $q_j$ for $j\in\anc(i)\cup\{i\}$, we obtain from \cref{eqa::dVdq} that
\begin{equation}\label{eqa::dV2dq2}
\frac{\partial^2 V_\gv}{\partial q_i\partial q_j}=\begin{cases}
\frac{\partial}{\partial q_j}\left(\frac{\partial V_\gv}{\partial q_i}\right) & j\in\anc(i)\cup\{i\},\\
\frac{\partial^2 V_\gv}{\partial q_j\partial q_i}& j\in\des(i),\\
0 & \text{otherwise},
\end{cases}
\end{equation}
which means that only $\frac{\partial^2 V_\gv}{\partial q_i\partial q_j}$ for $j\in\anc(i)\cup\{i\}$ needs to be explicitly computed. If $j\in\anc(i)\cup\{i\}$, using \cref{eqa::dSdq1,eqa::dpdq1,eqa::dVdq} as well as the equality $\hat{a}b=-\hat{b}a$ for any $a,b\in\R^3$, we obtain
\begin{equation}
\nonumber
\begin{aligned}
\frac{\partial^2 V_\gv}{\partial q_i\partial q_j}&=\frac{\partial}{\partial q_j}\left(\frac{\partial V_\gv}{\partial q_i}\right)\\
&=\sum_{i'\in\des(i)\cup\{i\}}m_{i'}\left[\overline{s}_i^T\left(\hat{\gv}\hat{\overline{s}}_jp_{i'}+\hat{\overline{s}}_j\hat{p}_{i'}\gv\right)-\overline{n}_i^T\hat{\gv}\overline{s}_j\right].
\end{aligned}
\end{equation}
In addition, since $\hat{p}_{i'}\hat{\gv}\overline{s}_j=-\hat{\gv}\hat{\overline{s}}_jp_{i'}-\hat{\overline{s}}_j\hat{p_{i'}}\gv$ and $\hat{a}^T=-\hat{a}$ for any $a\in\R^3$, the equation above is equivalent to 
\begin{equation}\label{eqa::dV}
\frac{\partial^2 V_\gv}{\partial q_i\partial q_j}= \overline{s}_j^T\hat{\gv}\bigg[\big(m_i+\sum_{i'\in \des(i)}m_{i'}\big)\overline{n}_{i}-\big(m_{i}\hat{p}_{i}+\sum_{i'\in\des(i)}m_{i'}\hat{p}_{i'}\big)\overline{s}_i\bigg]
\end{equation}
If we define 
$$\overline{\sigma}_{m_i}=m_i+\sum_{j\in \des(i)}m_{j} = m_i +\sum_{j\in \chd(i)}\overline{\sigma}_{m_{j}},$$
$$\overline{\sigma}_{p_i}=m_ip_i+\sum_{j\in \des(i)}m_{j}p_j = m_ip_i +\sum_{j\in \chd(i)}\overline{\sigma}_{p_{j}},$$
$$\overline{\sigma}_i^A=\hat{\gv}\left(\lsig_{m_i}\cdot\overline{n}_i-\hat{\lsig}_{p_i}\cdot \overline{s}_i \right),$$
then \cref{eqa::dV} is further simplified to
\begin{equation}\label{eqa::dV2}
\frac{\partial^2 V_\gv}{\partial q_i\partial q_j}= \overline{s}_j^T\hat{\gv}\left(\lsig_{m_i}\overline{n}_{i}-\hat{\lsig}_{p_i}\overline{s}_i\right)=\overline{s}_j^T\lsig_i^A.
\end{equation}
As a result, $\frac{\partial^2 V_\gv}{\partial q^2}$ can be computed from \cref{eqa::dV2,eqa::dV2dq2}.

The $O(n^2)$ complexity of \cref{algorithm::dvdq} is as follows: the forward pass to compute $g_i$ and $\lS_i$ and the backward pass to compute $\lsig_{m_i}$, $\lsig_{p_i}$ and $\lsig_i^A$ take $O(n)$ time, respectively; and the computation of $\frac{\partial^2 V_\gv}{\partial q_i\partial q_j}=\frac{\partial^2 V_\gv}{\partial q_j\partial q_i}=\overline{s}_j^T\lsig_i^A$ totally takes $O(n^2)$ time. Therefore, it can be concluded that \cref{algorithm::dvdq} has $O(n^2)$ complexity. This completes the proof.
\end{proof}

\end{document}